\documentclass{article}
\pdfoutput=1  

\usepackage[preprint]{neurips_2026}

\usepackage[utf8]{inputenc}
\usepackage[T1]{fontenc}
\usepackage{hyperref}
\usepackage{url}
\usepackage{booktabs}
\usepackage{amsfonts}
\usepackage{amsmath}
\usepackage{amssymb}
\usepackage{amsthm}
\usepackage{mathtools}
\usepackage{bm}
\usepackage{nicefrac}
\usepackage{microtype}
\usepackage{xcolor}
\usepackage{graphicx}
\usepackage{overpic}
\usepackage{subcaption}
\usepackage{float}
\usepackage{multirow}
\usepackage{array}
\usepackage{wrapfig}
\usepackage{algorithm}
\usepackage{algpseudocode}
\usepackage{caption}       
\usepackage{enumitem}
\usepackage{xspace}
\usepackage{tikz}
\usetikzlibrary{shapes,arrows,arrows.meta,positioning,calc,fit,backgrounds,decorations.pathreplacing}
\usepackage{pgfplots}
\pgfplotsset{compat=1.18}
\usepgfplotslibrary{fillbetween}
\usepackage{tcolorbox}
\tcbuselibrary{theorems,skins,breakable}
\usepackage{cleveref}
\usepackage{listings}

\lstdefinelanguage{json}{
    basicstyle=\small\ttfamily,
    numbers=none,
    numberstyle=\tiny,
    stepnumber=1,
    numbersep=8pt,
    showstringspaces=false,
    breaklines=true,
    frame=single,
    backgroundcolor=\color{gray!5},
    rulecolor=\color{gray!30},
    string=[s]{"}{"},
    stringstyle=\color{blue!70!black},
    comment=[l]{//},
    commentstyle=\color{gray},
    morestring=[b]',
    literate=
        *{0}{{{\color{purple}0}}}{1}
        {1}{{{\color{purple}1}}}{1}
        {2}{{{\color{purple}2}}}{1}
        {3}{{{\color{purple}3}}}{1}
        {4}{{{\color{purple}4}}}{1}
        {5}{{{\color{purple}5}}}{1}
        {6}{{{\color{purple}6}}}{1}
        {7}{{{\color{purple}7}}}{1}
        {8}{{{\color{purple}8}}}{1}
        {9}{{{\color{purple}9}}}{1}
        {:}{{{\color{black}{:}}}}{1}
        {,}{{{\color{black}{,}}}}{1}
        {\{}{{{\color{black}{\{}}}}{1}
        {\}}{{{\color{black}{\}}}}}{1}
        {[}{{{\color{black}{[}}}}{1}
        {]}{{{\color{black}{]}}}}{1},
}

\pgfmathsetseed{20251222}

\DeclareMathOperator{\E}{\mathbb{E}}

\newcommand{\R}{\mathbb{R}}

\newcommand{\cD}{\mathcal{D}}
\newcommand{\cJ}{\mathcal{J}}

\newcommand{\cM}{\mathcal{M}}
\newcommand{\cI}{\mathcal{I}}
\newcommand{\cV}{\mathcal{V}}
\newcommand{\cS}{\mathcal{S}}
\newcommand{\cU}{\mathcal{U}}

\newcommand{\protoname}{PULSE\xspace}

\definecolor{lightblue}{RGB}{230,240,250}
\definecolor{lightgreen}{RGB}{230,250,230}
\definecolor{lightyellow}{RGB}{255,250,230}

\definecolor{covenantred}{HTML}{FF3A3A}
\definecolor{covenantgreen}{HTML}{2ECC71}
\definecolor{covenantblack1000}{HTML}{101010}
\definecolor{covenantblack900}{HTML}{171717}
\definecolor{covenantblack800}{HTML}{2F2F2F}
\definecolor{covenantblack500}{HTML}{828282}
\definecolor{covenantwhite0}{HTML}{F4F4F4}
\definecolor{covenantwhite100}{HTML}{EFEFEF}
\definecolor{covenantwhite100alt}{HTML}{DDDDDD}

\newtheorem{theorem}{Theorem}[section]

\newtheorem{proposition}[theorem]{Proposition}
\newtheorem{corollary}[theorem]{Corollary}
\theoremstyle{definition}
\newtheorem{definition}[theorem]{Definition}

\theoremstyle{remark}

\title{Understanding and Exploiting Weight Update Sparsity for Communication-Efficient Distributed RL}

\author{%
  Erfan Miahi\thanks{Correspondence to: \texttt{erfan@covenant.ai}} \\
  Covenant AI\\
  \And
  Eugene Belilovsky \\
  Mila, Concordia University \\
}

\begin{document}

\maketitle

\begin{abstract}
Bandwidth-constrained distributed reinforcement learning (RL) post-training of large language models is bottlenecked by two channels: weight synchronization from trainers to inference workers, and gradient or pseudo-gradient synchronization across trainers. We find that approximately $\mathbf{99\%}$ of per-step weight updates are invisible after the BF16 cast used by standard training and inference forward passes. We explain this sparsity by showing that, at typical RL post-training learning rates, Adam updates often fall below the local BF16 rounding threshold. We turn this observation into an algorithmic principle called \emph{compute-visible sparsification}: transmit only updates that would change the next forward pass. \protoname (Precision-gated Updates for Low-precision Sparse Exchange) turns this principle into two communication algorithms: \emph{PULSESync} sends lossless sparse BF16 weight patches from trainers to inference workers, and \emph{PULSELoCo} sparsifies DiLoCo-style FP32 pseudo-gradient synchronization with error feedback. Over bandwidth-constrained commodity networks, PULSESync cuts weight-synchronization communication by $> \mathbf{100\times}$ while reconstructing trainer weights bit-identically. PULSELoCo matches DiLoCo across four models while reducing trainer-to-trainer communication by $> \mathbf{17\times}$ versus DiLoCo and $> \mathbf{100\times}$ versus DDP in the largest evaluated setting.
\end{abstract}


\section{Introduction}
\label{sec:introduction}

Reinforcement learning (RL) is now a standard post-training stage for large language models~\citep{lee2024rlaif,shao2024deepseekmath,deepseekr1,srivastava2025rlsurvey,yang2025qwen3,olmo2025olmo3,google2025gemini}. Modern RL pipelines decouple the trainer from inference, because specialized rollout engines generate trajectories up to $12\times$ faster than general-purpose training frameworks~\citep{hu2024openrlhf,sheng2025hybridflow,shen2024nemoaligner}. Combined with data-parallel training, this decoupling creates two channels: trainer-to-inference \emph{weight synchronization}, which refreshes the rollout policy, and trainer-to-trainer \emph{gradient or pseudo-gradient synchronization}. Both channels are already costly in deployed systems: in a geo-distributed RL run, synchronizing a 32B BF16 policy ($62$\,GB) to inference workers averaged $14$ minutes per round at commodity bandwidths~\citep{primeintellectteam2025intellect2}. Trainer-to-trainer synchronization is also expensive: a 32B model carries $128$\,GB of FP32 gradients, so dense gradient synchronization across four trainers moves hundreds of GB per round. \Cref{fig:hero} quantifies the resulting compute-utilization loss across bandwidth regimes.

Prior work has observed that RL fine-tuning changes only $5$--$30\%$ of parameters under coarse checkpoint comparisons~\citep{mukherjee2024sparse}. We instead measure the bitwise change between consecutive optimizer steps: this is the per-step information needed to keep workers in sync. At this granularity, we show that approximately $\mathbf{99\%}$ of parameters remain unchanged after the BF16 cast in standard RL pipelines, consistently across model families and scales (\Cref{sec:sparsity_analysis}). The mechanism is the interaction between BF16 precision and RL learning rates. Gradients are nearly fully dense (${\sim}99\%$ non-zero), and Adam applies its FP32 update to the FP32 master weights. Compute, however, runs in BF16: each forward pass uses a BF16 cast of the master weights. A per-step update affects computation only if it changes the BF16 value used in the next forward pass. At typical RL learning rates (${\sim}3 \times 10^{-6}$), an Adam upper bound on the per-step update falls below this threshold at the majority of weights (\Cref{sec:sparsity_analysis}). We call this criterion \emph{compute visibility}. BF16 is our main evaluated setting, but the principle is dtype-general: lower-precision formats such as FP8 or MXFP4 have coarser rounding cells and should make even more updates compute-invisible. We use this principle to design \protoname (Precision-gated Updates for Low-precision Sparse Exchange): PULSESync, a lossless method for trainer-to-inference weight synchronization, and PULSELoCo, a method for trainer-to-trainer pseudo-gradient synchronization (\Cref{sec:method}).

Trainer-to-inference weight synchronization admits a particularly clean, lossless use of this rule. Inference workers operate on BF16 weights, so any parameter whose BF16 representation is unchanged after an optimizer step is invisible to the worker, transmitted or not. PULSESync transmits only the changed BF16 values at each synchronization step and leaves the trainer's FP32 master untouched. Reconstruction at the worker is bit-identical to full-checkpoint synchronization, so PULSESync is lossless by construction. This targets a synchronization problem largely absent from prior communication-efficient training work, which focuses on trainer-to-trainer synchronization in pre-training~\citep{alistarh2017qsgd,lin2018deep,vogels2019powersgd,peng2024demo} rather than keeping inference workers synchronized during RL post-training. It is a drop-in replacement for dense weight synchronization in any RL pipeline. In a live deployment over the public internet, PULSESync cuts the payload by over $100\times$ at no cost to training behavior (\Cref{app:grail}).

\begin{figure}[tb]
    \centering
    \includegraphics[width=\linewidth]{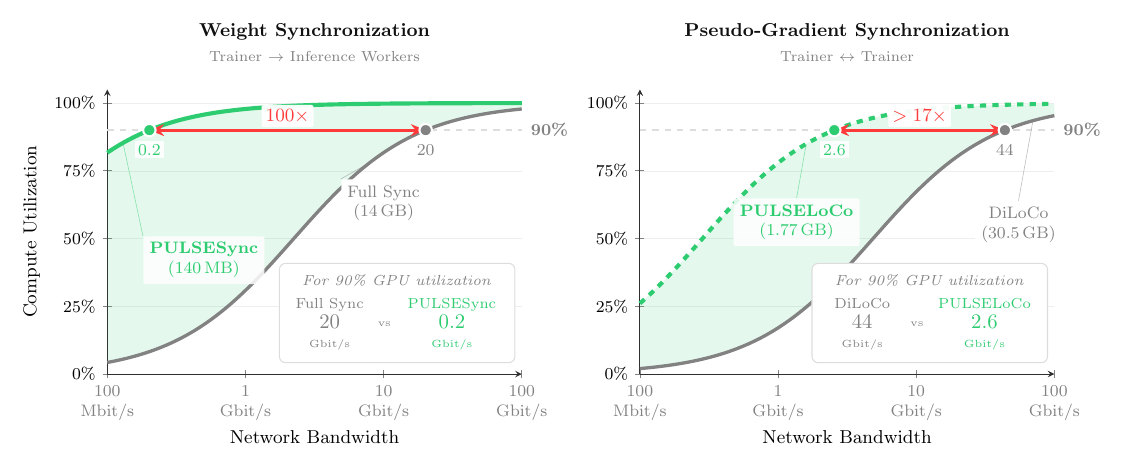}
    \caption{\textbf{Compute utilization vs.\ network bandwidth} on a 7B reference model using a $50$\,s compute interval between communications; bandwidth thresholds scale inversely with this interval. \emph{Left:} weight synchronization from trainer to inference workers. Full checkpoint sync transmits the BF16 weights ($14$\,GB per inference worker); PULSESync transmits encoded sparse BF16 patches ($140$\,MB), a $100\times$ lossless reduction. \emph{Right:} pseudo-gradient synchronization across trainers. DiLoCo sends a full FP32 pseudo-gradient ($30.5$\,GB per worker payload); PULSELoCo transmits an encoded sparse FP32 pseudo-gradient payload ($1.77$\,GB), a $>17\times$ lower payload than DiLoCo and ${\sim}138\times$ lower communication than DDP over the eight-step local-update window (see \Cref{app:bandwidth_accounting}). PULSESync and PULSELoCo reach $90\%$ GPU utilization at ${\sim}0.2$ and ${\sim}2.6$\,Gbit/s, while the corresponding full-payload transfers need ${\sim}20$ and ${\sim}44$\,Gbit/s. We validate weight synchronization in a live deployment over the public internet (\Cref{app:grail}) and pseudo-gradient synchronization against DDP and DiLoCo in \Cref{sec:experiments}.}
    \label{fig:hero}
    \vspace{-1.6em}
\end{figure}

Trainer-to-trainer synchronization needs a different design. Gradients are dense, so sparsifying them directly does not help. Compute visibility applies to updates, not raw gradients, so PULSELoCo synchronizes DiLoCo-style pseudo-gradients: each worker runs $H$ local steps, forms the update from its shared starting point, and synchronizes that update across trainers. PULSELoCo applies the same BF16 criterion as PULSESync to each worker's FP32 pseudo-gradient. Selected entries are synchronized as FP32 values; the rest remain in an FP32 error-feedback buffer and are reconsidered on the next round. This places PULSELoCo in the local-update family of communication-efficient training methods~\citep{stich2018localsgd,douillard2023diloco}. Unlike gradient compressors, which reduce per-round payloads for dense gradient synchronization~\citep{alistarh2017qsgd,lin2018deep,vogels2019powersgd,peng2024demo}, PULSELoCo combines less frequent synchronization with sparse pseudo-gradients selected by the forward-precision criterion. Compared with DDP, the $H$ local steps reduce synchronization frequency by another factor of $H$. We are not aware of prior compressors designed for multi-trainer RL post-training.

We make four contributions:
\begin{enumerate}[leftmargin=*, itemsep=2pt, topsep=4pt]
    \item We show that weight updates in typical RL post-training pipelines are naturally sparse and explain this with a mechanism based on RL learning rates, Adam update bounds, and low-precision forward-pass casting.
    \item We introduce \textbf{PULSESync}, a lossless weight-synchronization method with bit-identical reconstruction for inference workers in RL post-training. This yields over $100\times$ bandwidth reduction in a live deployment over the public internet when coupled with standard reinforcement learning pipelines (\Cref{app:grail}).
    \item We introduce \textbf{PULSELoCo}, which applies compute-visible sparsification to DiLoCo-style pseudo-gradient synchronization with FP32 error feedback. In MATH experiments, PULSELoCo matches DiLoCo; in the largest evaluated setting, it reduces trainer-to-trainer communication by $>17\times$ versus DiLoCo and $>100\times$ versus DDP (\Cref{sec:experiments,app:bandwidth_accounting}).
    \item To our knowledge, we provide the first empirical demonstration that DiLoCo-style RL post-training works for LLMs. Prior work has focused on pre-training settings.
\end{enumerate}\vspace{-10pt}

\section{Background and Problem Formulation} \vspace{-10pt}
\label{sec:background}


\label{subsec:rl_formulation}

We focus on reinforcement learning for reasoning tasks with \emph{verifiable rewards} (RLVR)~\citep{deepseekr1}, where rewards are computed by an automatic verifier (e.g., final-answer matching, unit tests) rather than learned human preferences.\footnote{While our experiments focus on GRPO for reasoning tasks, prior work observed similar sparsity patterns for PPO-based RLHF~\citep{mukherjee2024sparse}, suggesting this is a general phenomenon in RL post-training.}
We use Group Relative Policy Optimization (GRPO)~\citep{shao2024deepseekmath}, the dominant algorithm for training reasoning models~\citep{deepseekr1,yu2025dapo}. GRPO estimates advantages from group-relative rewards without requiring a learned value function: for each prompt, it samples a group of $G$ responses and computes advantages relative to the group mean and standard deviation. The policy is then updated using a clipped surrogate objective similar to PPO~\citep{schulman2017proximal}. Following recent work~\citep{yu2025dapo,liu2025drgrpo}, we omit the KL divergence penalty. We provide the full mathematical formulation in \Cref{app:grpo_details}.


In distributed RL training, different inference workers may operate with different versions of the model weights. We formalize this using \emph{off-policy delay}. Let $\theta_t$ denote the current model parameters at optimization step $t$. A rollout generated using parameters $\theta_{t-\tau}$ is said to have an \textbf{off-policy delay} of $\tau$ steps. In practice, this delay arises from asynchronous updates, communication latency, and batching. It affects both training dynamics and, as we will show, the sparsity structure of weight updates. \vspace{-10pt}

\section{Characterizing Weight Update Sparsity}\vspace{-10pt}
\label{sec:sparsity_analysis}
\label{sec:empirical}

For sparse weight updates to enable communication-efficient synchronization, three conditions must hold: sparsity must be \emph{consistently high} throughout training, \emph{mechanistically understood} so practitioners can preserve it, and \emph{stable} under the delayed synchronization patterns used in practical RL pipelines. This section establishes all three, directly informing the design of \protoname (\Cref{sec:method}). We use the following experimental setup:


\noindent\textbf{Setup.}
We measure sparsity under GRPO with DAPO-inspired hyperparameters~\citep{yu2025dapo} on MATH~\citep{hendrycks2021math} across Qwen2.5-Instruct (0.5B, 1.5B, 7B)~\citep{qwen2.5}, Llama-3.2-3B-Instruct~\citep{llama3}, and Gemma-3-4B-it~\citep{gemma3}. Runs use learning rate $3 \times 10^{-6}$, asymmetric clipping, 400 steps, 4 random seeds, and a composite reward based on correctness and formatting. Full experimental details are in \Cref{app:experimental_details}; training converges within this window for all models (\Cref{app:training_curves}).\\
\noindent\textbf{Sparsity metric.}
We measure \emph{weight update sparsity} after casting parameters to BF16, the dtype used by the next forward pass. Let $\bar{\theta}_t = \operatorname{cast}_{\mathrm{BF16}}(\theta_t)$ denote this BF16 view after optimization step $t$. Per-step sparsity is $|\{i : \bar{\theta}_{t+1}^{(i)} = \bar{\theta}_{t}^{(i)}\}| / d$, where equality is bitwise and $d$ is the total parameter count. Higher sparsity indicates fewer parameter changes that affect computation and greater potential for compression. We provide formal definitions, including the generalization to $k$-step sparsity (comparing $\bar{\theta}_t$ to $\bar{\theta}_{t+k}$), in \Cref{app:sparsity_definitions}. \vspace{-10pt}

\subsection{How Sparse Are Updates Throughout Training?}\vspace{-8pt}
\label{subsec:main_results}

\Cref{fig:sparsity_main} summarizes our findings. Mean per-step sparsity is approximately 99\% across all model scales and families, confirming and extending prior observations~\citep{mukherjee2024sparse} to the GRPO setting. This consistency across architectures (Qwen, Llama, Gemma) and scales (0.5B--7B) suggests that the phenomenon is tied to Adam optimization at RL learning rates rather than to a single model family. Sparsity is also stable throughout training: standard deviation across 400 steps is only 0.2--0.4\%, with even worst-case steps remaining above 98\%. For multi-step comparisons, sparsity remains above 98\% within the $k \leq 8$ range recommended for asynchronous RL~\citep{scalerl}, giving a stable target for communication-efficient method design (\Cref{sec:method}).

\begin{figure}[t]
    \centering
    \includegraphics[width=\linewidth]{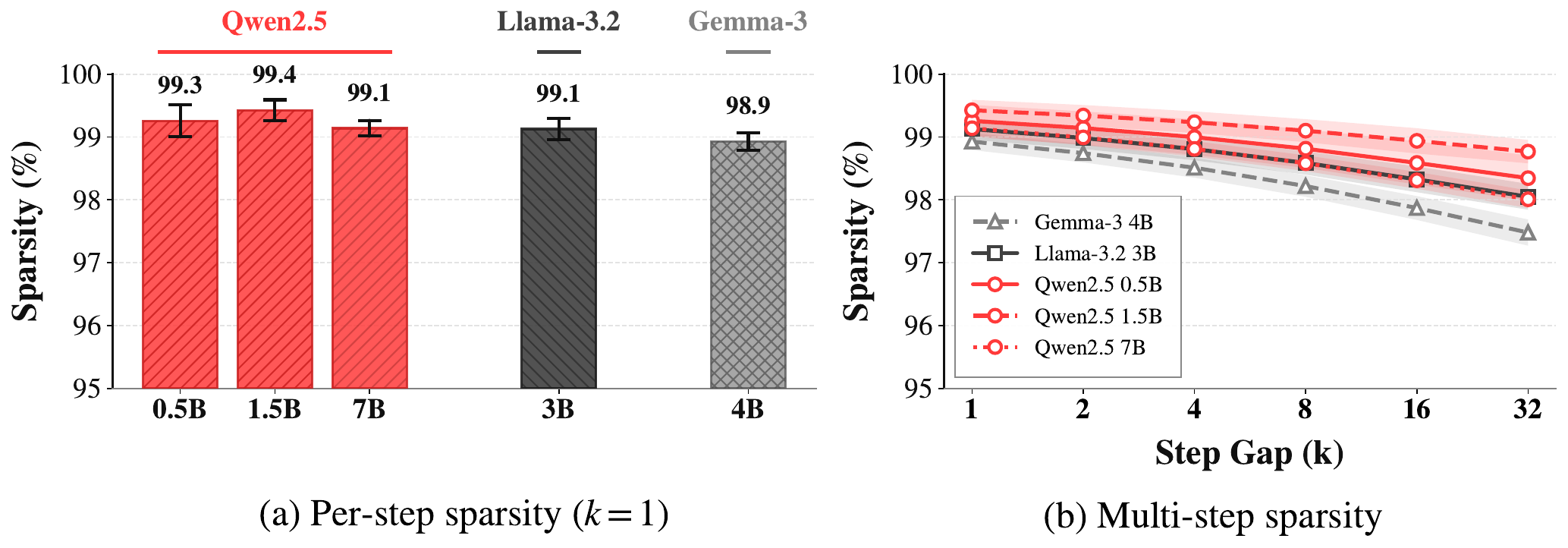}
    \caption{\textbf{Weight update sparsity across model scales and families.} Sparsity is measured after casting parameters to BF16. (a)~Mean per-step sparsity (\%) averaged over 400 training steps. Error bars indicate $\pm$1 standard deviation across steps. (b)~Sparsity when comparing parameter vectors from steps $t$ and $t+k$. Within the recommended $k \leq 8$ range for asynchronous RL~\citep{scalerl}, sparsity remains above 98\% for all models.}\vspace{-15pt}
    \label{fig:sparsity_main}
\end{figure}


With consistently high sparsity established, the remaining questions are why dense gradients produce sparse weight updates and whether the effect survives practical rollout delays. We first explain the BF16 absorption mechanism and its learning-rate dependence, then test robustness to policy staleness. \vspace{-20pt}

\subsection{Why Are Gradients Dense but Updates Sparse?}\vspace{-8pt}
\label{subsec:precision}

A natural hypothesis is that sparsity arises from sparse gradients. However, we find the opposite: \textbf{gradients are nearly fully dense}, with approximately 99\% of parameters receiving non-zero gradients at each step; see \Cref{app:gradient_sparsity}. Sparsity emerges downstream via \emph{update absorption}: BF16's limited resolution means updates smaller than roughly $|w|/256$ cannot be represented and are rounded away. Because learning rate directly scales update magnitude, it determines which weights can be modified. At typical RL post-training learning rates, e.g. $\eta \approx 3 \times 10^{-6}$, most updates fall below this threshold. We provide formal analysis in \Cref{app:bf16_details} that demonstrates this for Adam-style optimizers, though other optimizers may behave differently (\Cref{app:optimizer_dependence}).

\begin{figure}[t]
\centering
\captionsetup[subfigure]{skip=3pt}
\begin{subfigure}[t]{0.47\linewidth}
    \vspace{0pt}
    \centering
    \vspace{0.8em}
    \resizebox{\linewidth}{!}{%
    \begin{tikzpicture}[>=Stealth, font=\scriptsize, x=1cm, y=1.22cm]
        \draw[very thick, covenantblack1000] (0,0) -- (6,0);
        \foreach \x/\lab in {0/$b$,3/boundary,6/next BF16 value} {
            \draw[very thick, covenantblack1000] (\x,0.15) -- (\x,-0.15);
            \node[below=5pt, covenantblack1000] at (\x,-0.15) {\lab};
        }

        \fill[covenantred, opacity=0.13] (0,-0.8) rectangle (3,-1.35);
        \fill[covenantgreen, opacity=0.13] (3,-0.8) rectangle (6,-1.35);
        \draw[covenantred, dashed, thick] (3,-0.65) -- (3,-1.45);
        \node[covenantred, font=\scriptsize\bfseries] at (1.5,-1.08) {absorbed};
        \node[covenantgreen, font=\scriptsize\bfseries] at (4.5,-1.08) {visible};

        \fill[covenantblack1000] (0,0) circle (2.4pt);
        \fill[covenantred] (0.45,0) circle (2.4pt);
        \fill[covenantgreen] (3.25,0) circle (2.4pt);

        \node[covenantred, align=center, font=\scriptsize\bfseries] at (0.8,1.2) {one step\\$w+\Delta$};
        \draw[->, very thick, covenantred] (0.8,0.75) -- (0.45,0.12);
        \node[covenantgreen, align=center, font=\scriptsize\bfseries] at (4.6,1.2) {after later steps\\$w+\sum_{j=1}^{n}\Delta_j$};
        \draw[->, very thick, covenantgreen] (4.4,0.75) -- (3.25,0.12);

        \draw[decorate, decoration={brace, mirror, amplitude=4pt}, thick, covenantblack500] (0,-1.65) -- (3,-1.65);
        \node[covenantblack500, align=center] at (1.5,-2.18) {distance to boundary\\$\approx |w|/256$};
    \end{tikzpicture}}
    \vspace{0.1em}
    \caption{Local BF16 rounding interval.}
\end{subfigure}
\hspace{0.02\linewidth}
\begin{subfigure}[t]{0.49\linewidth}
    \vspace{0pt}
    \centering
    \includegraphics[width=\linewidth]{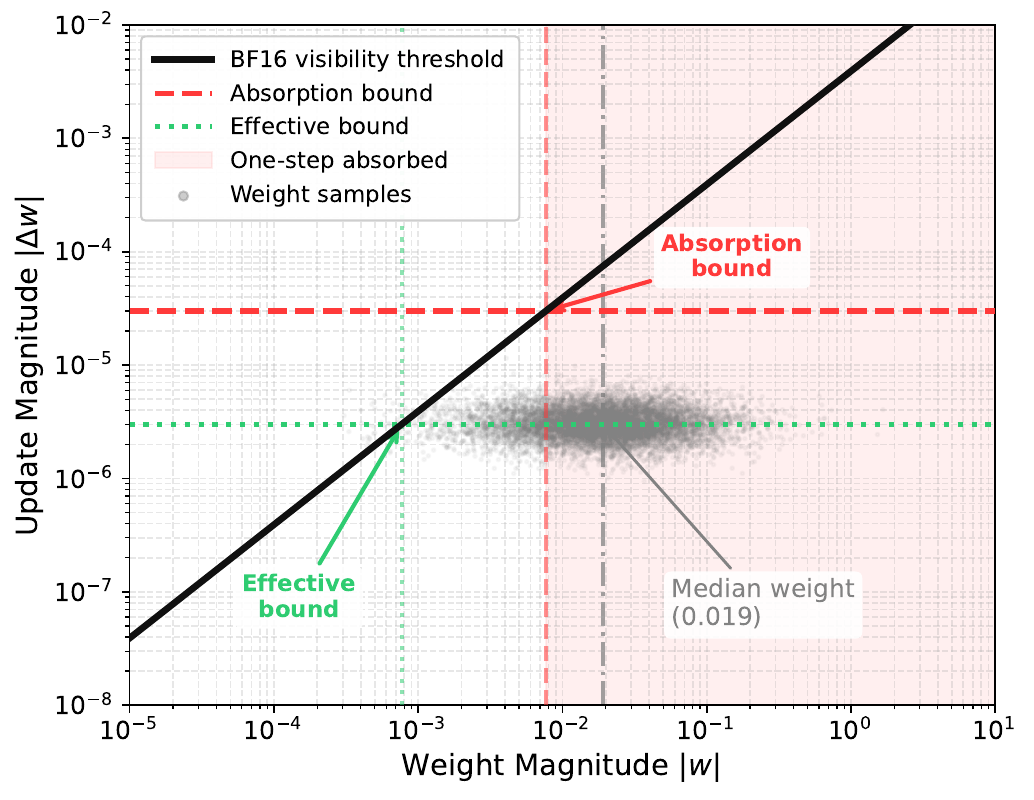}
    \caption{Global one-step threshold.}
\end{subfigure}
\caption{\textbf{BF16 absorption from one parameter to LLM-scale weights.} (a)~A one-step update can be absorbed by the next BF16 cast when it remains inside the current rounding cell. With FP32 master weights, later updates can accumulate and eventually cross the boundary. (b)~The black diagonal is the BF16 visibility threshold, approximately $|\Delta w| = |w|/256$; larger weights require larger updates. The horizontal lines show the effective bound ($\eta$) and absorption bound ($10\eta$) at learning rate $3 \times 10^{-6}$. Gray dots are representative LLM weights; most lie to the right of the absorption-bound crossing, where even the absorption bound lies below the BF16 visibility threshold.\vspace{-19pt}}
\label{fig:bf16_threshold}
\end{figure}

\Cref{fig:bf16_threshold} shows how the one-parameter rounding effect scales to full LLMs. Panel~(a) shows the local mechanism: a single small update can be invisible after the BF16 cast, while FP32 master weights continue to accumulate later updates. Panel~(b) applies the same criterion to real weight magnitudes. The diagonal is the BF16 visibility threshold: a one-step update must exceed roughly $|w|/256$ to change the BF16 value of a weight with magnitude $|w|$. The horizontal lines show two reference update sizes for Adam: the \emph{effective bound} near $\eta$, which matches the sign-like regime typical of stable Adam training~\citep{balles2018dissecting}, and the conservative \emph{absorption bound} at $10\eta$, the worst-case one-step bound for PyTorch-default Adam betas, $\beta_1=0.9$ and $\beta_2=0.999$. Most LLM weights, shown as gray dots, lie to the right of the absorption-bound crossing. For these weights, even the conservative bound is below the BF16 visibility threshold, so a one-step update is absorbed by the BF16 cast. This magnitude argument alone predicts 95--98\% one-step absorption; detailed statistics are in \Cref{app:weight_magnitude}. FP32 accumulation can eventually cross a BF16 boundary, but usually only over many steps, preserving high per-step sparsity (\Cref{app:mixed_precision}).

Learning rate is the primary factor controlling sparsity: raising $\eta$ shifts the update bounds upward, allowing more weights to change. However, $\eta$ is constrained by training stability rather than chosen for compression. RL post-training is sensitive to large policy updates and therefore typically uses smaller learning rates than supervised fine-tuning; in our GRPO sweeps, learning rates above ${\sim}5 \times 10^{-6}$ destabilize training. Thus the practical RL learning-rate range coincides with the range that produces high sparsity. The full learning-rate sweep is reported in \Cref{app:sparsity_factors}.

This analysis reconciles prior explanations of sparsity in RL fine-tuning~\citep{mukherjee2024sparse,shenfeld2025rlrazor,zhu2025pathnotaken}. BF16 precision supplies the rounding cells, while RL's stability-constrained learning rates keep most updates inside them; consistent with this view, \citet{shenfeld2025rlrazor} find that pure FP32 training eliminates sparsity. Standard mixed-precision training still preserves the effect because the next forward pass uses BF16 weights even when FP32 master weights store accumulated residuals (\Cref{app:mixed_precision}). We refer to this forward-pass criterion as \emph{compute visibility}: an update is visible if and only if it changes the value seen by the next forward pass. Because the rule is fixed by the forward precision, it introduces no top-$k$ threshold or compression hyperparameter and applies naturally to lower-precision formats. \vspace{-10pt}


\subsection{How Does Policy Staleness Affect Sparsity?}\vspace{-8pt}
\label{subsec:staleness}

\begin{wrapfigure}[15]{r}{0.47\textwidth}
\vspace{-2.6em}
    \centering
    \includegraphics[width=\linewidth]{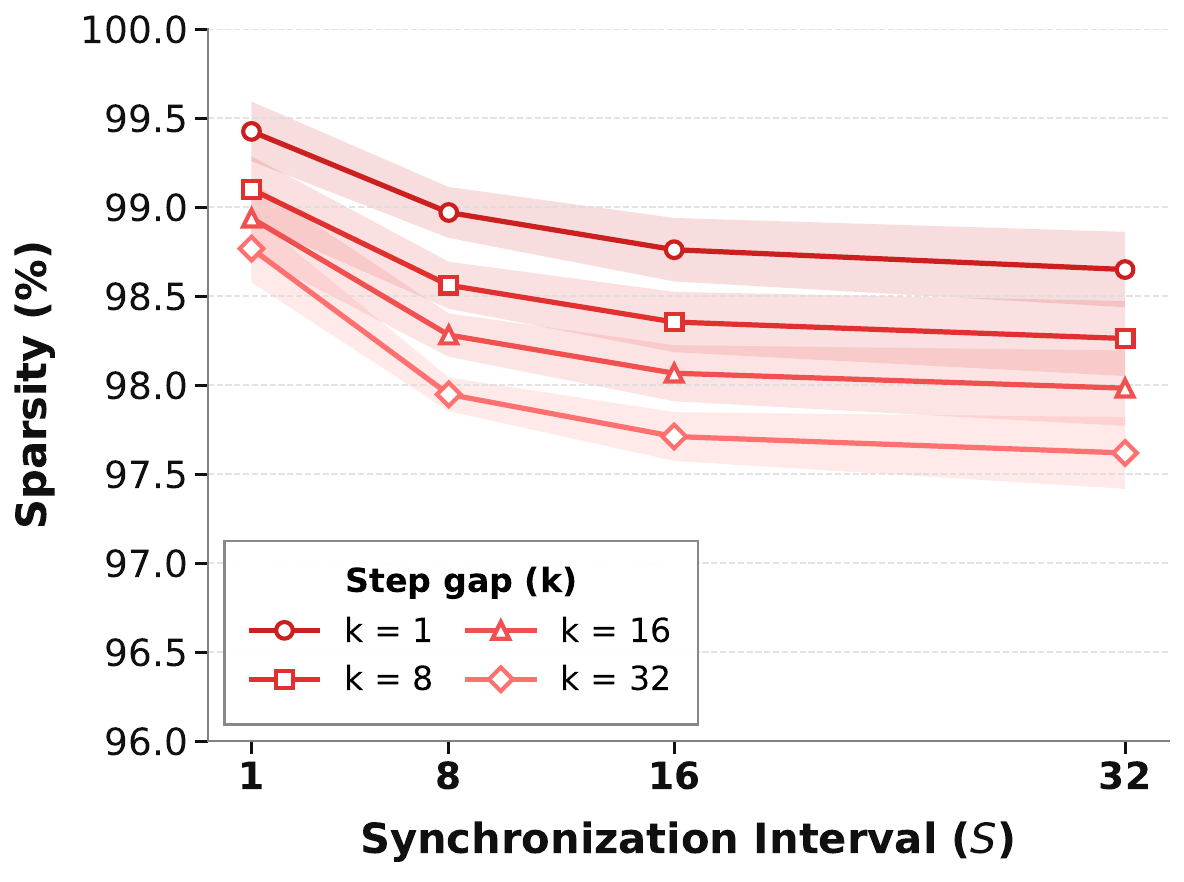}
    \caption{\textbf{Policy staleness effect.} Per-step sparsity remains above 98.5\% at $S=32$; all tested $k$ values remain above 97.5\%.}
    \label{fig:staleness_sparsity_main}
\end{wrapfigure}

Distributed RL pipelines often generate rollouts with policy weights that lag behind the learner. To isolate this effect, we vary the rollout synchronization interval $S$: rollouts are regenerated every $S$ optimizer steps, so $S=1$ is fully on-policy and larger $S$ induces off-policy delays $\tau \in \{0,\ldots,S-1\}$. In the figure, $k$ only specifies which two optimizer steps are compared for sparsity; rollout delay is controlled by $S$. \Cref{fig:staleness_sparsity_main} shows that staleness only modestly reduces sparsity: per-step sparsity remains above 98.5\% even at $S=32$, and all tested $k$ values remain above 97.5\%. Thus high compute-visible sparsity is not an artifact of fully fresh rollouts; it persists under delayed rollout regimes used in asynchronous RL. Together, the results in this section show that sparsity is high, explained by BF16 absorption at RL learning rates, and robust to rollout delays. \Cref{sec:method} turns this criterion into synchronization algorithms for distributed RL. \vspace{-5pt}


\section{The \protoname Methods}
\label{sec:method}

We present \protoname, two synchronization methods built around one rule: send only updates that would change the next forward pass. \emph{PULSESync} applies this rule to trainer-to-inference weight synchronization. \emph{PULSELoCo} applies it to DiLoCo-style pseudo-gradient synchronization~\citep{douillard2023diloco}, with error feedback for entries that are not sent in the current round. \Cref{subsec:cvs} defines the rule; \Cref{subsec:pulsesync} and \Cref{subsec:pulseloco} describe the two algorithms; \Cref{fig:sync_topology} summarizes the implementation topology.

\subsection{Compute-Visible Sparsification}
\label{subsec:cvs}

In typical RL training pipelines, FP32 master weights are cast to BF16 or lower precision for each forward pass, while optimizer state remains in FP32. If an FP32 update does not change the BF16 value of a parameter, then the next forward pass sees the same operand for that parameter. We call this criterion the \emph{compute-visibility gate}:
\begin{equation}
G_D(\theta, s) := \{\, i : \text{cast}_D(\theta_i) \neq \text{cast}_D(\theta_i - s_i) \,\},
\label{eq:gate}
\end{equation}
where $\theta$ is the parameter vector, $s$ is the proposed update, and $D$ is the compute dtype. We use $D = \text{BF16}$ throughout the main paper. PULSE sends only the updates that pass this gate. Updates that fail the gate are kept, not dropped. For trainer-to-inference weight synchronization, they remain in the trainer's FP32 master weights and are sent later if they change the BF16 view. For trainer-to-trainer pseudo-gradient synchronization, each worker keeps them in an FP32 error-feedback buffer for the next outer round. The next two subsections instantiate this rule for BF16 weight patches and FP32 pseudo-gradient synchronization.

\subsection{PULSESync: Lossless Weight Synchronization}
\label{subsec:pulsesync}

For trainer-to-inference synchronization, PULSESync compares consecutive BF16 checkpoints at the trainer and sends the changed values as a sparse patch. Encoding applies the gate with bitwise comparison and packages the selected indices and new values (\Cref{alg:sparse_patch}). At the inference worker, decoding overwrites those parameters. Because patches store values rather than arithmetic differences, reconstruction is bit-identical to the trainer's BF16 view, so PULSESync is lossless for the next forward pass.

\begin{figure}[!htbp]
\centering

\begin{minipage}[t]{0.41\textwidth}
\centering
\refstepcounter{algorithm}
\label{alg:sparse_patch}
\begin{tcolorbox}[
    width=\linewidth,
    colback=gray!4,
    colframe=gray!50,
    boxrule=0.4pt,
    arc=2pt,
    left=5pt,
    right=5pt,
    top=2pt,
    bottom=4pt,
    toptitle=2pt,
    bottomtitle=1pt,
    title={\small\textbf{Algorithm \thealgorithm:} Sparse Value Patching},
    coltitle=black,
    fonttitle=\small
]
\scriptsize
\begin{algorithmic}[1]
\Procedure{Encode}{$W_t, W_{t-1}$}
    \State $\cI \gets \{i : W_t^{(i)} \neq W_{t-1}^{(i)}\}$
    \State $\cV \gets W_t[\cI]$
    \State $(\cI, \cV) \gets \textsc{DeltaEncode}(\cI, \cV)$
    \State $\cI \gets \textsc{Downcast}(\cI)$
    \State $P \gets \textsc{Compress}(\cI, \cV)$
    \State \Return $P$
\EndProcedure
\Statex
\Procedure{Decode}{$W_{t-1}, P$}
    \State $(\cI, \cV) \gets \textsc{Decompress}(P)$
    \State $\cI \gets \textsc{Upcast}(\cI)$
    \State $(\cI, \cV) \gets \textsc{DeltaDecode}(\cI, \cV)$
    \State $W_t \gets W_{t-1}$; $W_t[\cI] \gets \cV$
    \State \Return $W_t$
\EndProcedure
\end{algorithmic}
\end{tcolorbox}
\end{minipage}
\hfill
\begin{minipage}[t]{0.55\textwidth}
\centering
\refstepcounter{algorithm}
\label{alg:pulseloco}
\begin{tcolorbox}[
    width=\linewidth,
    colback=gray!4,
    colframe=gray!50,
    boxrule=0.4pt,
    arc=2pt,
    left=5pt,
    right=5pt,
    top=2pt,
    bottom=4pt,
    toptitle=2pt,
    bottomtitle=1pt,
    title={\small\textbf{Algorithm \thealgorithm:} PULSELoCo Outer Loop},
    coltitle=black,
    fonttitle=\small
]
\scriptsize
\begin{algorithmic}[1]
\Require $\theta^{(0)}$; $e_r^{(0)} \gets 0$; $m^{(0)} \gets 0$; $T,H,\mu,\alpha$
\For{$t = 1$ to $T$}
    \For{$r = 1$ to $R$ \textbf{in parallel}}
        \State $w_r \gets \theta^{(t-1)}$
        \For{$h = 1$ to $H$}
            \State $\xi_{r,t,h} \sim \cD_r$
            \State $w_r \gets w_r - \textsc{AdamStep}(w_r; \xi_{r,t,h})$
        \EndFor
        \State $s_r^{(t)} \gets (\theta^{(t-1)} - w_r) + e_r^{(t-1)}$
        \State $\cI_r^{(t)} \gets G_\text{BF16}(\theta^{(t-1)}, s_r^{(t)})$
        \State $e_r^{(t)}[\cI_r^{(t)}] \gets 0$
        \State $e_r^{(t)}[\overline{\cI_r^{(t)}}] \gets s_r^{(t)}[\overline{\cI_r^{(t)}}]$
    \EndFor
    \State $(\cU^{(t)}, \bar{V}^{(t)}) \gets \textsc{SparseSync}_{r=1}^{R}(\cI_r^{(t)}, s_r^{(t)}[\cI_r^{(t)}])$
    \State $g^{(t)} \gets \mathbf{0}_d$; $g^{(t)}[\cU^{(t)}] \gets \bar{V}^{(t)}$
    \State $m^{(t)} \gets \mu \cdot m^{(t-1)} + g^{(t)}$
    \State $\theta^{(t)} \gets \theta^{(t-1)} - \alpha(\mu \cdot m^{(t)} + g^{(t)})$
\EndFor
\end{algorithmic}
\end{tcolorbox}
\end{minipage}

\end{figure}

\noindent\textbf{Encoding and decoding.} Given consecutive BF16 checkpoints $W_{t-1}$ and $W_t$, we identify differing positions in one pass. For each changed position, we store its index and new value, not an arithmetic difference. Storing values avoids floating-point drift from repeatedly adding deltas. Delta-encoding and downscaling the indices provide approximately $23\%$ additional compression before the general-purpose codec. Reconstruction reverses the pipeline: decompress, recover absolute indices, and overwrite $W_t[\cI] \gets \cV$. This is a direct memory copy with no floating-point arithmetic, so chained patches remain bit-identical. Implementation-level recovery paths, ready markers, and anchor handling are described in \Cref{app:distributed_sync}.

\noindent\textbf{Compression codec selection.} Sparse patches compose with a general-purpose codec (lz4 / zstd-1 / zstd-3); zstd-1 is our default at typical-cloud bandwidth, achieving approximately $79\times$ total reduction. Per-codec trade-offs and the bandwidth-regime selection table are in \Cref{app:compression}.

\subsection{PULSELoCo: Error-Feedback Pseudo-Gradient Synchronization}
\label{subsec:pulseloco}
\begin{wrapfigure}[22]{r}{0.47\textwidth}
\vspace{-0.4em}
\centering
\includegraphics[width=\linewidth]{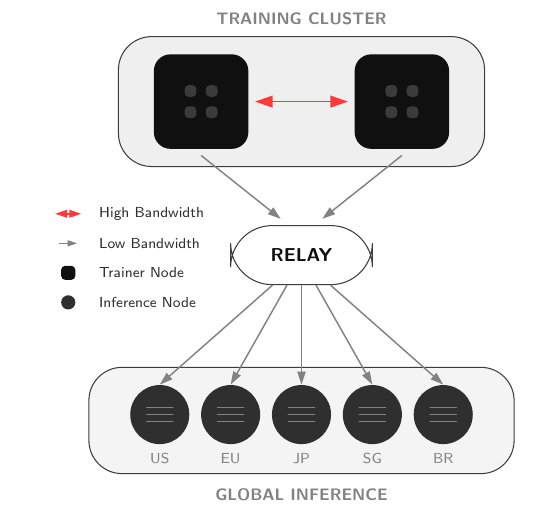}
\caption{\textbf{\protoname topology.} PULSESync sends BF16 patches through a relay to inference workers; PULSELoCo synchronizes FP32 pseudo-gradients across trainers.}
\label{fig:sync_topology}
\end{wrapfigure}

Gradients are dense, so applying the compute-visibility gate directly to raw gradients would not reduce the trainer-to-trainer payload. PULSELoCo instead follows DiLoCo~\citep{douillard2023diloco}: workers synchronize parameter-space updates after local optimization, not per-step gradients. The BF16 gate is therefore applied to DiLoCo-style pseudo-gradients, while the selected pseudo-gradient values are transmitted in FP32 because they are inputs to the outer optimizer.

In DiLoCo, each outer round starts from shared parameters $\theta^{(t-1)}$. Each worker $r \in \{1,\ldots,R\}$ copies these parameters, runs $H$ local Adam steps, and reaches local weights $w_r^{(t,H)}$. The worker then forms a pseudo-gradient, $\Delta_r^{(t)} := \theta^{(t-1)} - w_r^{(t,H)}$, which is the parameter-space update produced by local training. Standard DiLoCo synchronizes this full FP32 pseudo-gradient across workers and applies the aggregate with an outer Sutskever-form Nesterov optimizer using momentum $\mu = 0.9$ and step size $\alpha = 0.7$. Larger $H$ reduces how often workers communicate.

PULSELoCo keeps DiLoCo's local Adam steps and outer optimizer unchanged. The only change is the synchronization payload: worker $r$ adds its error-feedback buffer, forming $s_r^{(t)} := \Delta_r^{(t)} + e_r^{(t-1)}$, applies the compute-visibility gate $G_\text{BF16}(\theta^{(t-1)}, s_r^{(t)})$, and synchronizes only the selected pseudo-gradient entries. After synchronization, worker $r$ clears the entries that were sent and stores the entries that were not sent in $e_r^{(t)}$ for the next round. This buffer lets pseudo-gradient entries that are too small to change the BF16 value accumulate until they become visible, mirroring how small updates accumulate in FP32 master weights before changing the BF16 weights used in forward passes. PULSELoCo applies outer momentum only after synchronization, so the momentum state tracks the same global update as DiLoCo rather than each worker's local sparse payload. Thus PULSELoCo is an error-feedback sparsification method whose sparsity level is set by BF16 compute visibility. Algorithm~\ref{alg:pulseloco} shows the full outer loop; \textsc{SparseSync} returns the union support and averages selected FP32 values over all $R$ workers, treating missing entries as zeros.

\section{Distributed RL Synchronization Evaluation}
\label{sec:experiments}

We evaluate the two synchronization channels from \Cref{fig:hero}: PULSESync for trainer-to-inference weight broadcast and PULSELoCo for trainer-to-trainer pseudo-gradient synchronization. Because PULSESync is lossless in standard pipelines, \Cref{sec:empirical} already characterizes the sparsity available for that broadcast. We additionally demonstrate PULSESync in a live globally distributed training run, then compare PULSELoCo against DiLoCo.

\noindent\textbf{PULSESync in a real-world deployment.} PULSESync is deployed as the weight-synchronization layer in grail, a geo-distributed RL training framework running over the public internet. Training compute nodes use high-bandwidth links, while rollout nodes are globally distributed; a relay network~\citep{primeintellectteam2025intellect2} distributes sparse BF16 weight patches from trainers to inference workers, as shown in \Cref{fig:sync_topology}. This deployment uses PULSESync only, not PULSELoCo.

We run Qwen2.5-7B-Instruct on MATH and Qwen2.5-Coder-7B-Instruct on MBPP with 3 independent seeds per task; setup, rewards, and rollout integrity verification are in \Cref{app:grail}.

\Cref{fig:grail_training_curves} shows the main deployment result. Validation pass@1 improves steadily on both tasks, while upload sizes stay near $108$\,MB (SE: $1.1$\,MB). A full 7B BF16 checkpoint is $14$\,GB, so the measured mean corresponds to approximately $130\times$ reduction, and every run remains above $100\times$ reduction. All transfers pass checksum verification, confirming bit-identical reconstruction at inference workers.

\begin{figure}[ht]
\centering
\includegraphics[width=\linewidth]{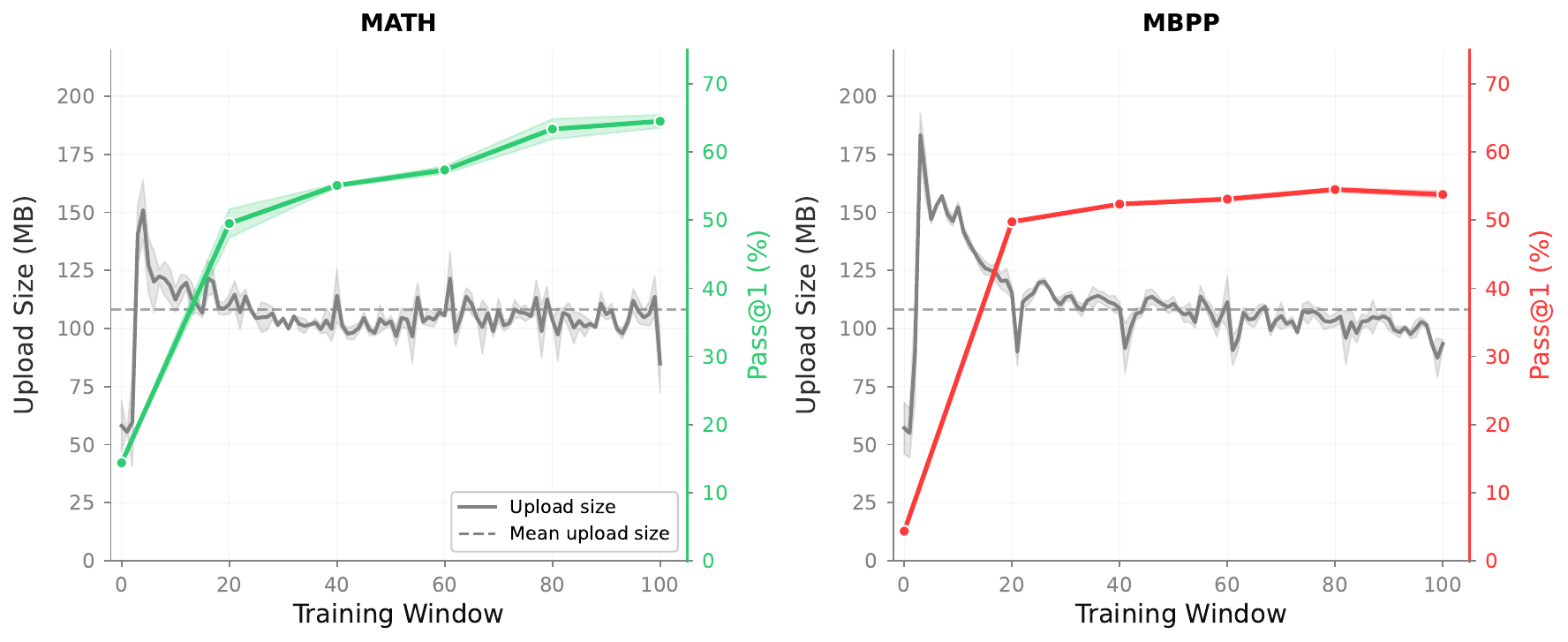}
\caption{\textbf{Training progress with PULSESync on grail.} Validation pass@1 (colored lines) improves steadily while upload sizes (gray lines) remain stable throughout training. The dashed line indicates the mean upload size of $108$\,MB, representing more than $100\times$ reduction compared to the $14$\,GB required for full checkpoint synchronization. Shaded regions indicate $\pm 1$ standard error across 3 independent runs.}
\label{fig:grail_training_curves}
\end{figure}

\noindent\textbf{PULSELoCo experimental setup.} For trainer-to-trainer synchronization, we run DDP, DiLoCo~\citep{douillard2023diloco}, and PULSELoCo in the same modified TRL GRPO loop, keeping batching, rewards, rollout generation, and evaluation fixed across methods. We evaluate Qwen2.5-1.5B/3B/7B-Instruct~\citep{qwen2.5} and Llama-3.2-3B-Instruct~\citep{llama3} on MATH~\citep{hendrycks2021math} with 3 seeds. All runs use $R=4$ workers. For both DiLoCo and PULSELoCo, rollout workers use shared global checkpoints and are refreshed only at outer-round boundaries. This makes very large $H$ less practical in RL than in pre-training DiLoCo settings~\citep{douillard2023diloco}: as $H$ grows, rollouts become increasingly off-policy relative to local trainer weights. We therefore use the largest stable windows we found, $H=8$ for the Qwen models and $H=4$ for Llama-3.2-3B-Instruct. Setup details, hyperparameters and bandwidth accounting are in \Cref{app:dataset_details,app:hyperparameters,app:bandwidth_accounting}.

\Cref{fig:learning_curves} shows the trainer-to-trainer results. DiLoCo remains close to DDP in most model settings, confirming that the local-update baseline is viable in RL post-training. PULSELoCo then recovers DiLoCo's learning behavior over the course of training. In the first checkpoints, PULSELoCo sometimes improves more slowly because entries that fail the BF16 gate have not yet accumulated in the error-feedback buffer. Once these residuals are carried into later rounds, the gap closes: by the final checkpoints, PULSELoCo is within seed variance of DiLoCo on all models. The $H$ sensitivity sweep in \Cref{app:pulseloco_h_sensitivity} shows that increasing $H$ modestly reduces sparsity but keeps PULSELoCo payloads far below dense synchronization, so the chosen $H$ values are driven mainly by RL staleness and stability rather than by the sparse payload itself.

\begin{figure}[H]
\centering
\includegraphics[width=\linewidth]{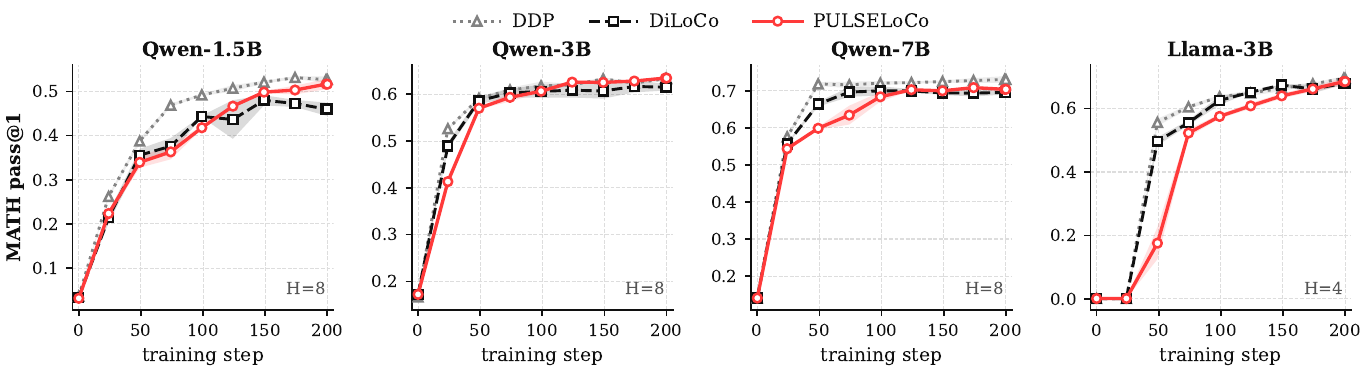}
\caption{\textbf{MATH validation pass@1 over training steps} for DDP, DiLoCo~\citep{douillard2023diloco}, and PULSELoCo at the chosen local-update windows ($H{=}8$ for the Qwen-2.5 family; $H{=}4$ for Llama-3.2) and $R = 4$. Shaded regions indicate $\pm 1$ standard error across $3$ seeds. PULSELoCo can lag early but catches up as error feedback accumulates, matching DiLoCo within seed variance by the end of training.}
\label{fig:learning_curves}
\end{figure}

\noindent\textbf{PULSELoCo sparse payloads.} PULSELoCo reduces trainer-to-trainer communication through both local updates and sparse pseudo-gradient exchange. Across the four model settings, each worker sends only $3.6$--$5.2\%$ of FP32 pseudo-gradient values per outer round ($94.8$--$96.4\%$ sparsity), a $19$--$28\times$ value reduction before index bytes. After accounting for delta-varint indices but no general-purpose codec, the conservative raw payload is still $12.8\times$ smaller than DiLoCo's full FP32 pseudo-gradient at the same outer-round cadence. On the 7B setting used in \Cref{fig:hero}, encoding the same sparse stream reduces the final-round payload from $2.39$\,GB to $1.77$\,GB, or $>17\times$ below DiLoCo's $30.5$\,GB payload. Relative to DDP over the same local-update window, the local-update structure adds another factor of $H$ by reducing synchronization frequency. Additional PULSELoCo sparse-payload measurements are in \Cref{app:pulseloco_sparsity}; encoded payload sizes and codec curves are in \Cref{app:bandwidth_accounting}. \Cref{app:pulseloco_sparsity} also reports the paired PULSESync checkpoint-patch sparsity measured from these runs.

\section{Conclusion}
\label{sec:conclusion}

We introduced \emph{compute-visible sparsification} for distributed RL post-training: communicate an update only when it can change the next forward pass. We show that gradients are dense, but approximately $99\%$ of per-step weight updates are invisible after the BF16 cast for the next forward pass, consistently across model families and scales (\Cref{sec:sparsity_analysis}).

\protoname uses this observation in two settings. PULSESync sends lossless sparse BF16 patches for trainer-to-inference weight synchronization, reducing bandwidth by over $100\times$ in a live decentralized RL deployment. PULSELoCo applies the same rule to DiLoCo-style trainer synchronization with error feedback. Across four MATH settings, PULSELoCo matches DiLoCo. In the 7B setting, it is $>17\times$ smaller than DiLoCo's full FP32 pseudo-gradient and $>100\times$ lower communication than DDP over the same local-update window. Together, these methods reduce the two bandwidth bottlenecks needed for geo-distributed RL post-training over commodity links: keeping rollout workers current and keeping trainers synchronized (\Cref{sec:experiments,app:bandwidth_accounting,app:grail}).

Several questions remain open. We evaluate PULSELoCo up to 7B parameters on MATH and use modest local-update windows because our experiments keep rollout workers on a shared global checkpoint. As $H$ grows, rollouts become increasingly stale relative to each trainer's local weights, which can destabilize training. Future work should test whether other weight-synchronization approaches can support larger $H$. It should also evaluate PULSELoCo at a larger scale, with lower-precision forward passes, and on multi-turn agentic tasks, while exploring extensions beyond DiLoCo-style synchronization.


\bibliography{references}

\begin{thebibliography}{35}
\providecommand{\natexlab}[1]{#1}
\providecommand{\url}[1]{\texttt{#1}}
\expandafter\ifx\csname urlstyle\endcsname\relax
  \providecommand{\doi}[1]{doi: #1}\else
  \providecommand{\doi}{doi: \begingroup \urlstyle{rm}\Url}\fi

\bibitem[Alistarh et~al.(2017)Alistarh, Grubic, Li, Tomioka, and
  Vojnovic]{alistarh2017qsgd}
Dan Alistarh, Demjan Grubic, Jerry Li, Ryota Tomioka, and Milan Vojnovic.
\newblock {QSGD}: Communication-efficient {SGD} via gradient quantization and
  encoding.
\newblock In \emph{Advances in Neural Information Processing Systems},
  volume~30, 2017.

\bibitem[Austin et~al.(2021)Austin, Odena, Nye, Bosma, Michalewski, Dohan,
  Jiang, Cai, Terry, Le, and Sutton]{austin2021program}
Jacob Austin, Augustus Odena, Maxwell Nye, Maarten Bosma, Henryk Michalewski,
  David Dohan, Ellen Jiang, Carrie Cai, Michael Terry, Quoc Le, and Charles
  Sutton.
\newblock Program synthesis with large language models.
\newblock \emph{arXiv preprint arXiv:2108.07732}, 2021.

\bibitem[Balles and Hennig(2018)]{balles2018dissecting}
Lukas Balles and Philipp Hennig.
\newblock Dissecting {Adam}: The sign, magnitude and variance of stochastic
  gradients.
\newblock In \emph{Proceedings of the 35th International Conference on Machine
  Learning}, volume~80 of \emph{Proceedings of Machine Learning Research},
  pages 404--413. PMLR, 2018.

\bibitem[{DeepSeek-AI}(2024)]{deepseekv3}
{DeepSeek-AI}.
\newblock {DeepSeek-V3} technical report.
\newblock \emph{arXiv preprint arXiv:2412.19437}, 2024.

\bibitem[{DeepSeek-AI} et~al.(2025){DeepSeek-AI}, Guo, Yang, Zhang, Song, Wang,
  et~al.]{deepseekr1}
{DeepSeek-AI}, Daya Guo, Dejian Yang, Haowei Zhang, Junxiao Song, Peiyi Wang,
  et~al.
\newblock {DeepSeek-R1}: Incentivizing reasoning capability in {LLMs} via
  reinforcement learning.
\newblock \emph{arXiv preprint arXiv:2501.12948}, 2025.

\bibitem[Douillard et~al.(2023)Douillard, Feng, Rusu, Chhaparia, Donchev,
  Kuncoro, Ranzato, Szlam, and Shen]{douillard2023diloco}
Arthur Douillard, Qixuan Feng, Andrei~A. Rusu, Rachita Chhaparia, Yani Donchev,
  Adhiguna Kuncoro, Marc'Aurelio Ranzato, Arthur Szlam, and Jiajun Shen.
\newblock {DiLoCo}: Distributed low-communication training of language models.
\newblock \emph{arXiv preprint arXiv:2311.08105}, 2023.

\bibitem[{Gemini Team, Google}(2025)]{google2025gemini}
{Gemini Team, Google}.
\newblock {Gemini} 2.5: Pushing the frontier with advanced reasoning,
  multimodality, long context, and next generation agentic capabilities.
\newblock \emph{arXiv preprint arXiv:2507.06261}, 2025.

\bibitem[{Gemma Team}(2025)]{gemma3}
{Gemma Team}.
\newblock Gemma 3 technical report.
\newblock \emph{arXiv preprint arXiv:2503.19786}, 2025.

\bibitem[Grattafiori et~al.(2024)]{llama3}
Aaron Grattafiori et~al.
\newblock The llama 3 herd of models.
\newblock \emph{arXiv preprint arXiv:2407.21783}, 2024.

\bibitem[Hendrycks et~al.(2021)Hendrycks, Burns, Kadavath, Arora, Basart, Tang,
  Song, and Steinhardt]{hendrycks2021math}
Dan Hendrycks, Collin Burns, Saurav Kadavath, Akul Arora, Steven Basart, Eric
  Tang, Dawn Song, and Jacob Steinhardt.
\newblock Measuring mathematical problem solving with the {MATH} dataset.
\newblock In \emph{Advances in Neural Information Processing Systems (Datasets
  and Benchmarks Track)}, 2021.

\bibitem[Hu et~al.(2024)Hu, Wu, Shen, Liu, Zhu, Wang, Jiang, Wang, Chen, Chen,
  Fang, Xianyu, Cao, Xu, and Liu]{hu2024openrlhf}
Jian Hu, Xibin Wu, Wei Shen, Jason~Klein Liu, Zilin Zhu, Weixun Wang, Songlin
  Jiang, Haoran Wang, Hao Chen, Bin Chen, Weikai Fang, Xianyu, Yu~Cao, Haotian
  Xu, and Yiming Liu.
\newblock {OpenRLHF}: An easy-to-use, scalable and high-performance {RLHF}
  framework.
\newblock \emph{arXiv preprint arXiv:2405.11143}, 2024.

\bibitem[Khatri et~al.(2025)Khatri, Madaan, Tiwari, Bansal, Duvvuri, Zaheer,
  Dhillon, Brandfonbrener, and Agarwal]{scalerl}
Devvrit Khatri, Lovish Madaan, Rishabh Tiwari, Rachit Bansal, Sai~Surya
  Duvvuri, Manzil Zaheer, Inderjit~S. Dhillon, David Brandfonbrener, and
  Rishabh Agarwal.
\newblock The art of scaling reinforcement learning compute for {LLMs}.
\newblock \emph{arXiv preprint arXiv:2510.13786}, 2025.

\bibitem[Lee et~al.(2024)Lee, Phatale, Mansoor, Mesnard, Ferret, Lu, Bishop,
  Hall, Carbune, Rastogi, and Prakash]{lee2024rlaif}
Harrison Lee, Samrat Phatale, Hassan Mansoor, Thomas Mesnard, Johan Ferret,
  Kellie Lu, Colton Bishop, Ethan Hall, Victor Carbune, Abhinav Rastogi, and
  Sushant Prakash.
\newblock {RLAIF} vs. {RLHF}: Scaling reinforcement learning from human
  feedback with {AI} feedback.
\newblock In \emph{Proceedings of the 41st International Conference on Machine
  Learning}, volume 235 of \emph{Proceedings of Machine Learning Research},
  pages 26874--26901. PMLR, 2024.

\bibitem[Lin et~al.(2018)Lin, Han, Mao, Wang, and Dally]{lin2018deep}
Yujun Lin, Song Han, Huizi Mao, Yu~Wang, and William~J Dally.
\newblock Deep gradient compression: Reducing the communication bandwidth for
  distributed training.
\newblock In \emph{International Conference on Learning Representations}, 2018.

\bibitem[Liu et~al.(2025)Liu, Chen, Li, Qi, Pang, Du, Lee, and
  Lin]{liu2025drgrpo}
Zichen Liu, Changyu Chen, Wenjun Li, Penghui Qi, Tianyu Pang, Chao Du, Wee~Sun
  Lee, and Min Lin.
\newblock Understanding {R1}-zero-like training: A critical perspective.
\newblock \emph{arXiv preprint arXiv:2503.20783}, 2025.

\bibitem[Micikevicius et~al.(2018)Micikevicius, Narang, Alben, Diamos, Elsen,
  Garcia, Ginsburg, Houston, Kuchaiev, Venkatesh, and
  Wu]{micikevicius2018mixed}
Paulius Micikevicius, Sharan Narang, Jonah Alben, Gregory~F. Diamos, Erich
  Elsen, David Garcia, Boris Ginsburg, Michael Houston, Oleksii Kuchaiev,
  Ganesh Venkatesh, and Hao Wu.
\newblock Mixed precision training.
\newblock In \emph{International Conference on Learning Representations}, 2018.

\bibitem[Mukherjee et~al.(2025)Mukherjee, Yuan, Hakkani-T{\"u}r, and
  Peng]{mukherjee2024sparse}
Sagnik Mukherjee, Lifan Yuan, Dilek Hakkani-T{\"u}r, and Hao Peng.
\newblock Reinforcement learning finetunes small subnetworks in large language
  models.
\newblock In \emph{Advances in Neural Information Processing Systems}, 2025.
\newblock NeurIPS 2025.

\bibitem[Peng et~al.(2024)Peng, Chen, Su, Quesnelle, Kingma, and
  Liu]{peng2024demo}
Bowen Peng, Lizhang Chen, Baiyu Su, Jeffrey Quesnelle, Diederik~P. Kingma, and
  Qiang Liu.
\newblock {DeMo}: Decoupled momentum optimization.
\newblock \emph{arXiv preprint arXiv:2411.19870}, 2024.
\newblock URL \url{https://arxiv.org/abs/2411.19870}.

\bibitem[{Prime Intellect Team} et~al.(2025){Prime Intellect Team}, Jaghouar,
  Mattern, Ong, Straube, Basra, Pazdera, Thaman, Di~Ferrante, Gabriel, Obeid,
  Erdem, Keiblinger, and Hagemann]{primeintellectteam2025intellect2}
{Prime Intellect Team}, Sami Jaghouar, Justus Mattern, Jack~Min Ong, Jannik
  Straube, Manveer Basra, Aaron Pazdera, Kushal Thaman, Matthew Di~Ferrante,
  Felix Gabriel, Fares Obeid, Kemal Erdem, Michael Keiblinger, and Johannes
  Hagemann.
\newblock {INTELLECT-2}: A reasoning model trained through globally
  decentralized reinforcement learning, 2025.

\bibitem[{Qwen Team}(2024)]{qwen2.5}
{Qwen Team}.
\newblock Qwen2.5 technical report.
\newblock \emph{arXiv preprint arXiv:2412.15115}, 2024.

\bibitem[Schulman et~al.(2017)Schulman, Wolski, Dhariwal, Radford, and
  Klimov]{schulman2017proximal}
John Schulman, Filip Wolski, Prafulla Dhariwal, Alec Radford, and Oleg Klimov.
\newblock Proximal policy optimization algorithms.
\newblock \emph{arXiv preprint arXiv:1707.06347}, 2017.

\bibitem[Shao et~al.(2024)Shao, Wang, Zhu, Xu, Song, Bi, Zhang, Zhang, Li, Wu,
  and Guo]{shao2024deepseekmath}
Zhihong Shao, Peiyi Wang, Qihao Zhu, Runxin Xu, Junxiao Song, Xiao Bi, Haowei
  Zhang, Mingchuan Zhang, Y.K. Li, Y.~Wu, and Daya Guo.
\newblock {DeepSeekMath}: Pushing the limits of mathematical reasoning in open
  language models.
\newblock \emph{arXiv preprint arXiv:2402.03300}, 2024.

\bibitem[Shen et~al.(2024)Shen, Wang, Delalleau, Zeng, Dong, Egert, Sun, Zhang,
  Jain, Taghibakhshi, Sanz~Ausin, Aithal, and Kuchaiev]{shen2024nemoaligner}
Gerald Shen, Zhilin Wang, Olivier Delalleau, Jiaqi Zeng, Yi~Dong, Daniel Egert,
  Shengyang Sun, Jimmy Zhang, Sahil Jain, Ali Taghibakhshi, Markel Sanz~Ausin,
  Ashwath Aithal, and Oleksii Kuchaiev.
\newblock {NeMo-Aligner}: Scalable toolkit for efficient model alignment.
\newblock In \emph{Conference on Language Modeling (COLM)}, 2024.

\bibitem[Shenfeld et~al.(2025)Shenfeld, Pari, and Agrawal]{shenfeld2025rlrazor}
Idan Shenfeld, Jyothish Pari, and Pulkit Agrawal.
\newblock {RL}'s razor: Why online reinforcement learning forgets less.
\newblock \emph{arXiv preprint arXiv:2509.04259}, 2025.

\bibitem[Sheng et~al.(2025)Sheng, Zhang, Ye, Wu, Zhang, Zhang, Peng, Lin, and
  Wu]{sheng2025hybridflow}
Guangming Sheng, Chi Zhang, Zilingfeng Ye, Xibin Wu, Wang Zhang, Ru~Zhang,
  Yanghua Peng, Haibin Lin, and Chuan Wu.
\newblock {HybridFlow}: A flexible and efficient {RLHF} framework.
\newblock In \emph{Twentieth European Conference on Computer Systems (EuroSys
  '25)}, 2025.
\newblock \doi{10.1145/3689031.3696075}.

\bibitem[Srivastava and Aggarwal(2025)]{srivastava2025rlsurvey}
Saksham~Sahai Srivastava and Vaneet Aggarwal.
\newblock A technical survey of reinforcement learning techniques for large
  language models.
\newblock \emph{arXiv preprint arXiv:2507.04136}, 2025.

\bibitem[Stich(2019)]{stich2018localsgd}
Sebastian~U. Stich.
\newblock Local {SGD} converges fast and communicates little.
\newblock In \emph{International Conference on Learning Representations}, 2019.
\newblock URL \url{https://openreview.net/forum?id=S1g2JnRcFX}.

\bibitem[{Team OLMo} et~al.(2025{\natexlab{a}}){Team OLMo}, Walsh, Soldaini,
  Groeneveld, Lo, Arora, Bhagia, Gu, Huang, Jordan, Lambert, Schwenk, Tafjord,
  et~al.]{olmo2025olmo2}
{Team OLMo}, Pete Walsh, Luca Soldaini, Dirk Groeneveld, Kyle Lo, Shane Arora,
  Akshita Bhagia, Yuling Gu, Shengyi Huang, Matt Jordan, Nathan Lambert, Dustin
  Schwenk, Oyvind Tafjord, et~al.
\newblock 2 {OLMo} 2 furious.
\newblock \emph{arXiv preprint arXiv:2501.00656}, 2025{\natexlab{a}}.

\bibitem[{Team OLMo} et~al.(2025{\natexlab{b}})]{olmo2025olmo3}
{Team OLMo} et~al.
\newblock Olmo 3.
\newblock \emph{arXiv preprint arXiv:2512.13961}, 2025{\natexlab{b}}.

\bibitem[Touvron et~al.(2023)Touvron, Martin, Stone, Albert, Almahairi, Babaei,
  Bashlykov, Batra, Bhargava, Bhosale, et~al.]{touvron2023llama}
Hugo Touvron, Louis Martin, Kevin Stone, Peter Albert, Amjad Almahairi, Yasmine
  Babaei, Nikolay Bashlykov, Soumya Batra, Prajjwal Bhargava, Shruti Bhosale,
  et~al.
\newblock {Llama 2}: Open foundation and fine-tuned chat models.
\newblock \emph{arXiv preprint arXiv:2307.09288}, 2023.

\bibitem[Vogels et~al.(2019)Vogels, Karimireddy, and Jaggi]{vogels2019powersgd}
Thijs Vogels, Sai~Praneeth Karimireddy, and Martin Jaggi.
\newblock {PowerSGD}: Practical low-rank gradient compression for distributed
  optimization.
\newblock In \emph{Advances in Neural Information Processing Systems},
  volume~32, 2019.

\bibitem[von Werra et~al.(2020)von Werra, Belkada, Tunstall, Beeching, Thrush,
  Lambert, Huang, Rasul, and Gallou{\'e}dec]{vonwerra2022trl}
Leandro von Werra, Younes Belkada, Lewis Tunstall, Edward Beeching, Tristan
  Thrush, Nathan Lambert, Shengyi Huang, Kashif Rasul, and Quentin
  Gallou{\'e}dec.
\newblock {TRL}: Transformers reinforcement learning.
\newblock \url{https://github.com/huggingface/trl}, 2020.

\bibitem[Yang et~al.(2025)Yang, Li, Yang, Zhang, Hui, Zheng,
  et~al.]{yang2025qwen3}
An~Yang, Anfeng Li, Baosong Yang, Beichen Zhang, Binyuan Hui, Bo~Zheng, et~al.
\newblock Qwen3 technical report.
\newblock \emph{arXiv preprint arXiv:2505.09388}, 2025.

\bibitem[Yu et~al.(2025)]{yu2025dapo}
Qiying Yu et~al.
\newblock {DAPO}: An open-source {LLM} reinforcement learning system at scale.
\newblock \emph{arXiv preprint arXiv:2503.14476}, 2025.

\bibitem[Zhu et~al.(2025)Zhu, Zhang, Huang, Su, Liu, Zhao, Fedorov, Pirsiavash,
  Sha, Lee, Pan, Wang, Tian, and Tai]{zhu2025pathnotaken}
Hanqing Zhu, Zhenyu Zhang, Hanxian Huang, DiJia Su, Zechun Liu, Jiawei Zhao,
  Igor Fedorov, Hamed Pirsiavash, Zhizhou Sha, Jinwon Lee, David~Z. Pan,
  Zhangyang Wang, Yuandong Tian, and Kai~Sheng Tai.
\newblock The path not taken: {RLVR} provably learns off the principals.
\newblock \emph{arXiv preprint arXiv:2511.08567}, 2025.

\end{thebibliography}
\bibliographystyle{plainnat}

\newpage
\appendix

\paragraph{Appendix organization.}
The appendix is organized as follows. \Cref{app:sparsity_foundations} gives the mathematical and empirical details behind BF16 update absorption. \Cref{app:pulseloco_sparsity,app:compression} extend the main analysis to PULSELoCo sparse payloads, paired checkpoint patches, and codec selection. \Cref{app:grail} reports the PULSESync-only grail deployment. \Cref{app:experimental_details,app:extended_results} give experimental setup, the rationale for local-update windows, and additional results. \Cref{app:method_details,app:comparison,app:distributed_sync} collect protocol and implementation details.

\section{Sparsity Foundations}
\label{app:sparsity_foundations}
\label{app:theory}

\subsection{Formal Sparsity Definitions}
\label{app:sparsity_definitions}

We define the sparsity metrics used throughout the paper. These metrics quantify how many parameters remain unchanged in the compute view used by the next forward pass.

\begin{definition}[Compute-view Weight Update]
\label{def:weight_update}
Given model parameters $\theta_t \in \R^d$ at optimization step $t$ and compute dtype $D$, define $\bar{\theta}^{D}_t := \operatorname{cast}_{D}(\theta_t)$. The \textbf{$k$-step compute-view weight update} is:
\begin{equation}
    \Delta^D_{t,k} = \bar{\theta}^{D}_{t+k} - \bar{\theta}^{D}_t
\end{equation}
For consecutive steps ($k=1$), we simplify the notation to $\Delta^D_t = \bar{\theta}^{D}_{t+1} - \bar{\theta}^{D}_t$. The main paper uses $D=\mathrm{BF16}$.
\end{definition}

\begin{definition}[Update Sparsity]
\label{def:sparsity}
The \textbf{sparsity} of a $k$-step compute-view weight update is the fraction of parameters that remain bitwise identical between step $t$ and $t+k$ after casting to $D$:
\begin{equation}
    S^D_k(t) = \frac{1}{d} \sum_{i=1}^{d} \mathbb{1}[\bar{\theta}_{t+k}^{D,(i)} = \bar{\theta}_{t}^{D,(i)}]
\end{equation}
where $\mathbb{1}[\cdot]$ is the indicator function and equality is evaluated bitwise in dtype $D$. Higher values indicate fewer compute-visible parameter changes, enabling greater compression.
\end{definition}

The key insight exploited by \protoname is that $S^{\mathrm{BF16}}_1(t)$ is approximately 99\% in RL fine-tuning (\Cref{sec:sparsity_analysis}), enabling dramatic communication reduction.

\subsection{BF16 Precision and Update Absorption}
\label{app:bf16_details}

BF16 (bfloat16) uses 1 sign bit, 8 exponent bits, and 7 mantissa bits. The 7-bit mantissa provides $2^7 = 128$ distinct values between consecutive powers of two, making the smallest representable relative change approximately $\epsilon_{\text{bf16}} = 2^{-7} \approx 0.0078$. 

Critically, this representable gap \emph{scales with weight magnitude}. For example, between 1.0 and 2.0, the gap is $2^{-7} \approx 0.0078$, but between 8.0 and 16.0, the gap is $2^{-4} = 0.0625$ (8$\times$ larger).

\begin{definition}[Update Absorption]
\label{def:absorption}
An optimizer update $\Delta w$ to parameter $w$ is \emph{absorbed} if the BF16 representation remains unchanged: $\texttt{bf16}(w + \Delta w) = \texttt{bf16}(w)$. Equivalently, $w$ and $w+\Delta w$ remain in the same BF16 rounding cell. For a normalized BF16 value with $2^e \le |w| < 2^{e+1}$, the BF16 spacing is $2^{e-7}$, so half a unit in the last place (ULP) is:
\begin{equation}
    2^{e-8}
    \label{eq:absorption_threshold}
\end{equation}
Equivalently, this characteristic relative cell radius satisfies $2^{-9} < 2^{e-8}/|w| \le 2^{-8}$. For BF16-stored weights, an update smaller than this radius cannot leave the cell when starting from the cell center. For FP32 master weights, the exact threshold is the distance from $w$ to the nearest BF16 rounding boundary; this distance can be smaller if the FP32 accumulator is already near a boundary. Thus absorption and survival occur at a relative scale on the order of $2^{-8}$, while the exact criterion is always the bitwise cast comparison.
\end{definition}

The main-text \Cref{fig:bf16_threshold} illustrates this mechanism at two scales. Panel~(a) shows a local BF16 rounding interval: a one-step update can be absorbed by the next BF16 cast, while later FP32 updates can accumulate in the master weight until the cast changes. Panel~(b) shows the same boundary across weight magnitudes. The key insight is that absorption depends on the \emph{ratio} $|\Delta w|/|w|$, not $|\Delta w|$ alone. Small weights can change under small updates; large weights require proportionally larger updates.

\paragraph{Critical weight magnitude.}
Combining the characteristic BF16 cell scale with Adam's update bounds (\Cref{app:bounds}) yields a critical weight scale below which one-step updates are more likely to be visible. \Cref{cor:weight_threshold} formalizes this relationship; the key result is that for typical training regimes, weights with $|w| \gg 256\eta$ have per-step Adam updates that are small compared with a BF16 rounding cell, except when the FP32 master is already close to a cell boundary. At learning rate $\eta = 3 \times 10^{-6}$, this scale is $|w|_{\text{crit}} \approx 7.7 \times 10^{-4}$. Since typical LLM weights have magnitudes in $[0.01, 1.0]$ (\Cref{app:weight_magnitude}), the vast majority exceed this scale.

\paragraph{Empirical validation with mixed-precision training.}
\label{app:mixed_precision}
The standard configuration for LLM post-training uses mixed-precision training~\citep{micikevicius2018mixed}: the optimizer maintains FP32 master weights for numerical stability, while forward and backward passes execute in BF16. This differs from pure FP32 training, where computation and storage both use FP32 and sparsity is eliminated entirely~\citep{shenfeld2025rlrazor}.

In mixed-precision training, although the optimizer updates FP32 master weights (where small updates \emph{do} accumulate), inference still requires BF16 weights since all forward computation happens in BF16. Therefore, the relevant question for PULSESync is: \emph{how sparse are the weight updates when viewed in BF16?} We measure this by casting the FP32 master weights to BF16 after each optimization step and comparing consecutive BF16 snapshots. This reflects the actual weights that inference nodes receive and use.

\Cref{fig:mixed_precision_sparsity} validates that sparsity remains high under this standard setup. We train Qwen2.5-1.5B-Instruct with GRPO using FP32 master weights and BF16 computation; the resulting BF16-cast weight updates exhibit sparsity consistently above 99.4\%, comparable to the pure BF16 results in \Cref{sec:sparsity_analysis}.

This high sparsity persists because per-step updates are so small that even when accumulated in FP32, crossing a BF16 rounding cell typically requires many steps. At learning rate $\eta = 3 \times 10^{-6}$, a typical update magnitude is ${\sim}\eta$, while the characteristic BF16 cell radius for a weight with $|w| = 0.01$ is $|w|/256 \approx 4 \times 10^{-5}$. Since this scale is roughly $13\times$ larger than a single update, approximately 13 steps of consistent updates would be needed to cross a cell from its center. Consequently, at any given step, only a small fraction of weights have accumulated enough change to affect their BF16 representation. Unlike pure BF16 training where absorbed updates are permanently lost, mixed-precision updates do eventually manifest, but spread across many steps, preserving high per-step sparsity. Practitioners using standard mixed-precision pipelines benefit from this sparsity without modification.

\begin{figure}[t]
    \centering
    \includegraphics[width=0.85\linewidth]{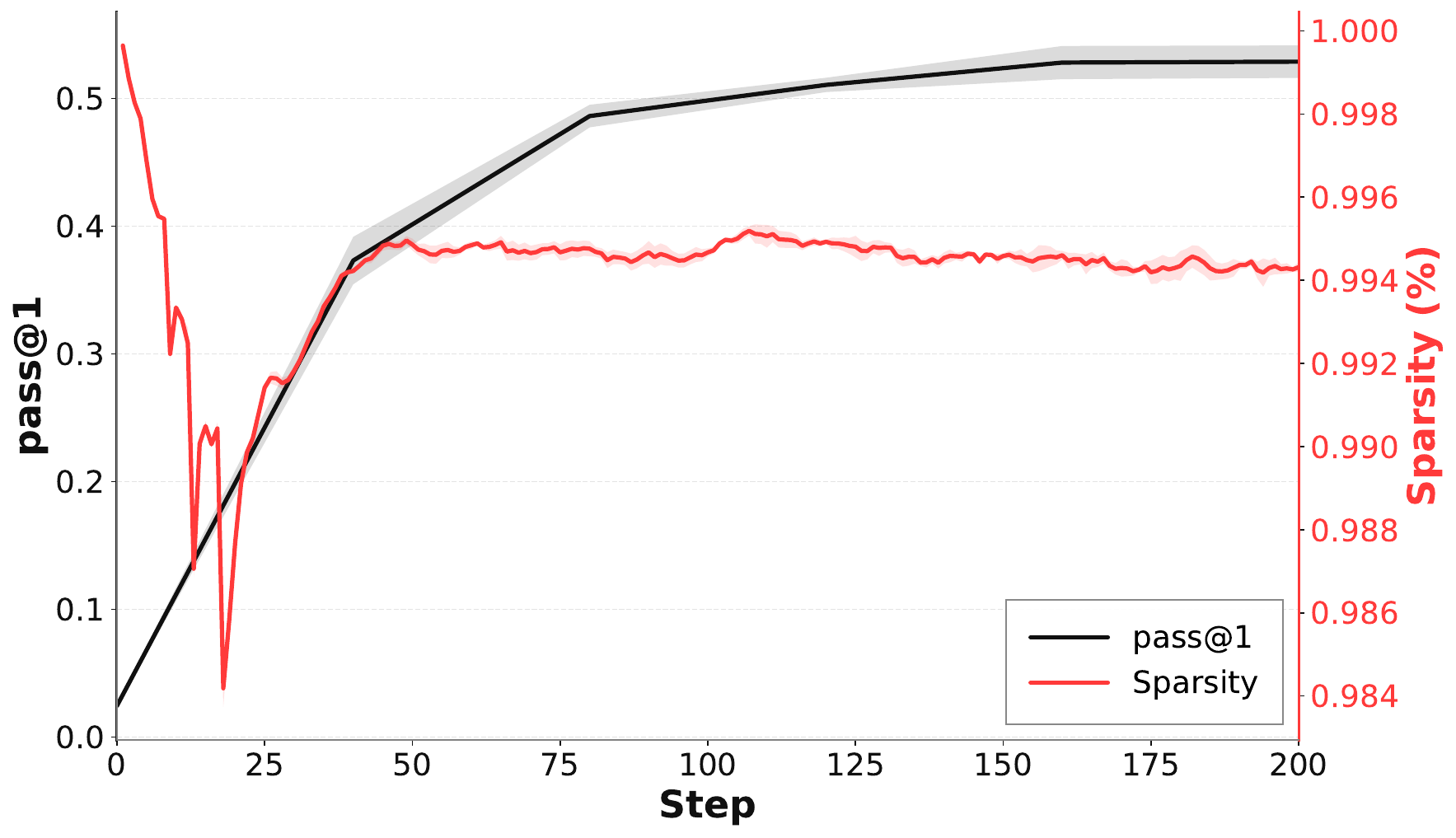}
    \caption{\textbf{Sparsity with mixed-precision training (FP32 master weights, BF16 computation).} Training Qwen2.5-1.5B-Instruct with GRPO on MATH tasks. Validation pass@1 improves steadily while weight update sparsity (measured by casting FP32 master weights to BF16 and comparing consecutive steps) remains consistently above 99.4\%. Shaded regions indicate $\pm$1 standard deviation across 4 seeds.}
    \label{fig:mixed_precision_sparsity}
\end{figure}

\subsection{Adam Update Bounds}
\label{app:bounds}

We derive an upper bound on the per-step update magnitude in Adam. This bound, combined with the BF16 rounding-cell scale, explains why most one-step updates remain compute-invisible after the BF16 cast.

\begin{theorem}[Adam Update Upper Bound]
\label{thm:adam_bound}
For one scalar parameter updated by Adam with hyperparameters $0 < \beta_1 < \beta_2 < 1$, learning rate $\eta$, non-negative numerical constant $\epsilon$, and no decoupled weight decay term, the update magnitude at step $t$ satisfies:
\begin{equation}
    |\Delta w_t| \leq \eta \sqrt{\frac{1-\beta_1}{1-\beta_2} \cdot \frac{1-\beta_2^t}{1-\beta_1^t}}
    \label{eq:adam_bound}
\end{equation}
As $t \to \infty$, this simplifies to:
\begin{equation}
    |\Delta w_t| \leq \eta \sqrt{\frac{1-\beta_1}{1-\beta_2}}
    \label{eq:adam_bound_asymptotic}
\end{equation}
\end{theorem}

\begin{proof}
The Adam update is $\Delta w_t = \eta \cdot \rho_t$ where $\rho_t = \hat{m}_t / (\sqrt{\hat{v}_t} + \epsilon)$. Our goal is to bound $|\rho_t|$, which requires showing that $\hat{v}_t$ cannot be too small relative to $\hat{m}_t^2$.

\textbf{Step 1: Express moments as weighted averages.} The bias-corrected moments can be written as weighted sums over the gradient history:
\begin{equation}
    \hat{m}_t = \sum_{i=1}^{t} p_i g_i, \qquad \hat{v}_t = \sum_{i=1}^{t} q_i g_i^2
\end{equation}
where the weights $p_i$ and $q_i$ are non-negative and sum to 1:
\begin{equation}
    p_i = \frac{(1-\beta_1)\beta_1^{t-i}}{1-\beta_1^t}, \qquad q_i = \frac{(1-\beta_2)\beta_2^{t-i}}{1-\beta_2^t}
\end{equation}
These weights arise from expanding the EMA recursion $m_t = \beta_1 m_{t-1} + (1-\beta_1)g_t$ and applying bias correction.

\textbf{Step 2: Compare the weight distributions.} Since $\beta_2 > \beta_1$, the $q_i$ weights decay more slowly than the $p_i$ weights, placing relatively more mass on older gradients. Crucially, the ratio $q_i/p_i = \frac{(1-\beta_2)}{(1-\beta_1)} \cdot \left(\frac{\beta_2}{\beta_1}\right)^{t-i} \cdot \frac{1-\beta_1^t}{1-\beta_2^t}$ is minimized at $i=t$ (the most recent gradient), giving a uniform lower bound:
\begin{equation}
    \frac{q_i}{p_i} \geq c_t \quad \text{for all } i, \qquad \text{where } c_t = \frac{1-\beta_2}{1-\beta_1} \cdot \frac{1-\beta_1^t}{1-\beta_2^t}
\end{equation}

\textbf{Step 3: Lower bound $\hat{v}_t$.} Using the weight ratio bound, we can relate $\hat{v}_t$ to $\hat{m}_t$:
\begin{align}
    \hat{v}_t = \sum_{i=1}^{t} q_i g_i^2
    &\geq c_t \sum_{i=1}^{t} p_i g_i^2 && \text{(since } q_i \geq c_t \cdot p_i \text{)} \\
    &\geq c_t \left(\sum_{i=1}^{t} p_i g_i\right)^2 && \text{(Jensen's inequality: } \mathbb{E}[X^2] \geq \mathbb{E}[X]^2 \text{)} \\
    &= c_t \cdot \hat{m}_t^2
\end{align}

\textbf{Step 4: Conclude.} Taking square roots and rearranging:
\begin{equation}
    \frac{|\hat{m}_t|}{\sqrt{\hat{v}_t}} \leq \frac{1}{\sqrt{c_t}} = \sqrt{\frac{1-\beta_1}{1-\beta_2} \cdot \frac{1-\beta_2^t}{1-\beta_1^t}}
\end{equation}
Since $\sqrt{\hat{v}_t} + \epsilon \geq \sqrt{\hat{v}_t}$, we have $|\rho_t| \leq 1/\sqrt{c_t}$, and thus $|\Delta w_t| = \eta|\rho_t| \leq \eta/\sqrt{c_t}$.
\end{proof}

\paragraph{Validity of the condition $\beta_2 > \beta_1$.}
The theorem requires $\beta_2 > \beta_1$, which holds across standard AdamW configurations used for LLM training. The choice of $\beta_2 = 0.95$ (rather than the default 0.999) has become prevalent in modern LLM training and post-training pipelines. This yields a tighter bound of $\sqrt{(1-0.9)/(1-0.95)} = \sqrt{2} \approx 1.41$ compared to $\sqrt{100} = 10$ for PyTorch defaults. If decoupled AdamW weight decay with coefficient $\lambda$ is enabled, the per-parameter update receives an additional term of magnitude $\eta\lambda |w_t|$; our sparsity experiments set weight decay to zero (\Cref{app:hyperparameters}).

\paragraph{Implications for standard hyperparameters.}
For the PyTorch default parameters $(\beta_1, \beta_2) = (0.9, 0.999)$:
\begin{equation}
    |\Delta w_t| \leq \eta \sqrt{\frac{0.1}{0.001}} = 10\eta
\end{equation}
With learning rate $\eta = 3 \times 10^{-6}$, this gives $|\Delta w_t| \leq 3 \times 10^{-5}$.

\Cref{tab:adamw_configs} summarizes the hyperparameters used by major LLM training pipelines. Modern LLM training often uses $\beta_2 = 0.95$, which yields a tighter bound of $\sqrt{2}\eta \approx 1.41\eta$ compared to $10\eta$ for the PyTorch default ($\beta_2 = 0.999$).

\paragraph{Why the sparsity analysis uses $\beta_2=0.999$.}
The controlled sparsity characterization in \Cref{sec:sparsity_analysis} uses the PyTorch-default Adam setting $(\beta_1,\beta_2)=(0.9,0.999)$, while the grail deployment study and PULSELoCo experiments use the post-training setting $\beta_2=0.95$. This split is intentional. The larger $\beta_2=0.999$ gives the looser worst-case bound $|\Delta w_t| \le 10\eta$, whereas $\beta_2=0.95$ gives $|\Delta w_t| \le \sqrt{2}\eta$. Thus, observing ${\sim}99\%$ BF16-visible sparsity under $\beta_2=0.999$ is a conservative stress test of the absorption mechanism relative to the $\beta_2=0.95$ regime used in the deployment and PULSELoCo evaluations. In typical, non-adversarial gradient histories, the ratio $|\hat{m}_t|/\sqrt{\hat{v}_t}$ remains close to 1, so the effective critical scale is governed primarily by $\eta$ and the BF16 cell size rather than by this worst-case $\beta_2$ bound.

\begin{table}[t]
\centering
\caption{Adam hyperparameters used by major LLM training pipelines. All configurations satisfy $\beta_2 > \beta_1$.}
\label{tab:adamw_configs}
\vspace{0.5em}
\small
\begin{tabular}{@{}lcccc@{}}
\toprule
\textbf{Model / Framework} & $\boldsymbol{\beta_1}$ & $\boldsymbol{\beta_2}$ & \textbf{Asymptotic Bound} & \textbf{Reference} \\
\midrule
PyTorch default & 0.9 & 0.999 & $10\eta$ & -- \\
LLaMA 2/3 & 0.9 & 0.95 & $\sqrt{2}\eta \approx 1.41\eta$ & \citet{touvron2023llama,llama3} \\
DeepSeek-V3/R1 & 0.9 & 0.95 & $\sqrt{2}\eta \approx 1.41\eta$ & \citet{deepseekv3} \\
Qwen 2.5 & 0.9 & 0.95 & $\sqrt{2}\eta \approx 1.41\eta$ & \citet{qwen2.5} \\
OLMo 2 & 0.9 & 0.95 & $\sqrt{2}\eta \approx 1.41\eta$ & \citet{olmo2025olmo2} \\
\textbf{This work} (controlled sparsity analysis) & 0.9 & 0.999 & $10\eta$ & \Cref{app:hyperparameters} \\
\textbf{This work} (grail / PULSELoCo) & 0.9 & 0.95 & $\sqrt{2}\eta \approx 1.41\eta$ & \Cref{app:hyperparameters} \\
\bottomrule
\end{tabular}
\end{table}

\begin{corollary}[Weight Magnitude Scale for BF16]
\label{cor:weight_threshold}
For a weight $w$ to receive a non-absorbed update in BF16 arithmetic, the update must cross the nearest BF16 rounding boundary. A characteristic scale for this boundary distance is half a BF16 ULP, giving $|\Delta w| / |w| \approx 2^{-8}$ within a factor of two (cf.\ \Cref{def:absorption}). Combined with \Cref{thm:adam_bound}, this gives the characteristic weight scale:
\begin{equation}
    |w| < 256 \cdot |\Delta w|_{\max} = 256 \eta \sqrt{\frac{1-\beta_1}{1-\beta_2}}
\end{equation}
For PyTorch defaults $(\beta_1, \beta_2) = (0.9, 0.999)$, this simplifies to $|w| < 2560\eta$. For modern LLM configurations with $\beta_2 = 0.95$, this becomes $|w| < 362\eta$. These are scales rather than hard deterministic thresholds for FP32 master weights, because residual accumulation can place a parameter close to a BF16 rounding boundary. Our controlled sparsity analysis (\Cref{sec:sparsity_analysis}) uses PyTorch defaults, while the grail deployment study (\Cref{app:grail}) and PULSELoCo experiments (\Cref{sec:experiments}) use $\beta_2 = 0.95$.

In practice, the ratio $|\hat{m}_t|/\sqrt{\hat{v}_t} \approx 1$ for most gradient patterns (see \Cref{fig:adam_ratio_adversarial}), yielding an \emph{effective} scale that is independent of $\beta_2$:
\begin{equation}
    |w|_{\text{crit}}^{\text{effective}} \approx 256\eta \approx 7.68 \times 10^{-4} \quad \text{(for } \eta = 3 \times 10^{-6}\text{)}.
\end{equation}
This effective scale predicts practical sparsity, explaining why observed sparsity is consistent across different $(\beta_1, \beta_2)$ configurations.
\end{corollary}

\subsection{Weight Magnitude Distribution}
\label{app:weight_magnitude}

To validate that the critical scale $|w|_{\text{crit}}$ is relevant to actual LLM weight distributions, we analyze weight magnitudes across several model families. \Cref{tab:weight_magnitude} shows that the vast majority of LLM weights have magnitudes well above the critical scale.

\begin{table}[t]
    \centering
    \caption{Weight magnitude statistics across model families. The characteristic scale for BF16 update survival at learning rate $\eta = 3 \times 10^{-6}$ is $|w|_{\text{crit}} \approx 7.7 \times 10^{-4}$ (typical, ratio $\approx 1$). Weights above this scale are likely to absorb per-step updates in BF16, explaining the observed ${\sim}$99\% per-step sparsity.}
    \label{tab:weight_magnitude}
    \vspace{0.5em}
    \small
    \begin{tabular}{@{}lccccc@{}}
    \toprule
    \textbf{Model} & \textbf{Median $|w|$} & \textbf{Mean $|w|$} & \textbf{5th \%ile} & \textbf{95th \%ile} & \textbf{\% $> |w|_{\text{crit}}$} \\
    \midrule
    Qwen2.5-0.5B & 0.0114 & 0.0145 & 0.0010 & 0.0374 & 96.2\% \\
    Qwen2.5-1.5B & 0.0177 & 0.0218 & 0.0016 & 0.0557 & 97.6\% \\
    Llama-3.2-3B & 0.0121 & 0.0149 & 0.0011 & 0.0381 & 96.5\% \\
    Gemma-3-4B & 0.0098 & 0.0157 & 0.0008 & 0.0369 & 95.3\% \\
    Qwen2.5-7B & 0.0099 & 0.0124 & 0.0008 & 0.0320 & 94.8\% \\
    \bottomrule
    \end{tabular}
\end{table}

\paragraph{Connection to observed sparsity.}
The median weight magnitude ranges from ${\sim}0.010$ (Qwen2.5-7B) to ${\sim}0.018$ (Qwen2.5-1.5B), placing it roughly $13$--$23\times$ above the critical scale. Across all five models, 94.8--97.6\% of weights exceed the scale. This scale-based estimate is conservative: the empirically observed sparsity of ${\sim}$99\% (\Cref{sec:sparsity_analysis}) is higher because even among the 2.4--5.2\% of weights below the effective scale, many still have their updates absorbed due to gradient oscillation (which reduces $|m_t|$ via cancellation) and the FP32 master sitting near the center of its current BF16 cell. \Cref{fig:bf16_threshold} in the main text visualizes this relationship.

\paragraph{The bound is loose, not a tight supremum.}
The bound $|\rho_t| \leq \sqrt{(1-\beta_1)/(1-\beta_2)}$ is an \emph{upper bound}, not a tight supremum. A sharper per-parameter supremum over nonzero gradient histories follows from Cauchy's inequality:
\begin{equation}
    \sup_{\{g_i\}_{i=1}^t \neq 0}
    \frac{|\sum_{i=1}^t p_i g_i|}{\sqrt{\sum_{i=1}^t q_i g_i^2}}
    =
    \left(\sum_{i=1}^t \frac{p_i^2}{q_i}\right)^{1/2}.
\end{equation}
For standard settings with $\beta_1^2 < \beta_2$, the infinite-horizon value is
\begin{equation}
    \frac{1-\beta_1}{\sqrt{(1-\beta_2)(1-\beta_1^2/\beta_2)}}.
\end{equation}
This equals approximately $7.27$ for $(\beta_1,\beta_2)=(0.9,0.999)$ and $1.16$ for $(0.9,0.95)$, compared with the simpler bounds $10$ and $1.41$. We use the looser bound in \Cref{thm:adam_bound} because it has a simpler form and is sufficient to identify the BF16 cell scale that drives sparsity.

\paragraph{Approaching the bound with adversarial sequences.}
To understand how close we can get to the absorption bound, consider the following adversarial gradient sequence: a long ``quiet'' period of near-zero gradients followed by constant large gradients. This exploits the fact that $\beta_2 > \beta_1$ causes the second moment $v_t$ to respond more slowly than the first moment $m_t$.

\begin{figure}[t]
\centering
\begin{tikzpicture}
    \definecolor{covenantred}{HTML}{FF3A3A}
    \definecolor{covenantblack}{HTML}{101010}
    \definecolor{covenantgray}{HTML}{828282}
    \begin{axis}[
        width=0.9\linewidth,
        height=5.5cm,
        xlabel={Steps after quiet period},
        ylabel={Ratio $|\hat{m}_t|/\sqrt{\hat{v}_t}$},
        xmin=0, xmax=55,
        ymin=0, ymax=11,
        grid=major,
        grid style={dashed, gray!15},
        legend pos=north east,
        legend style={font=\footnotesize, fill=white, fill opacity=0.95, draw=none},
    ]

    \addplot[dashed, very thick, color=covenantgray, domain=0:55, samples=2] {10};
    \addlegendentry{Absorption bound: $10$}

    \addplot[very thick, color=covenantred, mark=*, mark size=1.5pt] coordinates {
        (1, 3.16) (2, 4.25) (3, 4.95) (4, 5.44) (5, 5.80)
        (6, 6.06) (7, 6.24) (8, 6.38) (9, 6.47) (10, 6.53)
        (11, 6.56) (12, 6.57) (13, 6.56) (14, 6.54) (15, 6.51)
        (16, 6.46) (17, 6.41) (18, 6.35) (19, 6.29) (20, 6.24)
        (25, 5.92) (30, 5.57) (35, 5.22) (40, 4.89) (45, 4.59) (50, 4.32)
    };
    \addlegendentry{Adversarial sequence}

    \node[font=\scriptsize, color=covenantred, anchor=south] at (axis cs:12, 6.7) {Peak: $6.57$};

    \addplot[dotted, thick, color=covenantblack, domain=0:55, samples=2] {1};
    \addlegendentry{Typical (constant $g$): $1$}

    \end{axis}
\end{tikzpicture}
\caption{\textbf{Ratio $|\hat{m}_t|/\sqrt{\hat{v}_t}$ for an adversarial gradient sequence.} The sequence consists of $10^5$ near-zero gradients followed by constant gradients of magnitude 1. The ratio peaks at 6.57 after 12 large gradients, then decays as $v_t$ catches up. Despite this extreme construction, the ratio only reaches 66\% of the absorption bound of 10. For constant gradients (typical case), the ratio equals 1.}
\label{fig:adam_ratio_adversarial}
\end{figure}

\Cref{fig:adam_ratio_adversarial} shows the ratio $|\hat{m}_t|/\sqrt{\hat{v}_t}$ for the sequence $[10^{-20}] \times 10^5 + [1.0] \times k$. Key observations:
\begin{itemize}[leftmargin=*, itemsep=2pt]
    \item The ratio peaks at \textbf{6.57} after 12 large gradients, only 66\% of the absorption bound
    \item After the peak, $v_t$ accumulates and the ratio decays back toward 1
    \item For constant gradients (the typical case in training), the ratio equals exactly 1
    \item Even this highly adversarial sequence, which requires $10^5$ steps of setup, cannot approach the bound
\end{itemize}

This analysis confirms that the $10\eta$ bound is loose in practice: the ratio $|\hat{m}_t|/\sqrt{\hat{v}_t}$ rarely exceeds 2 for realistic gradient sequences encountered during training.

\paragraph{Why typical training stays near $\rho = 1$.}
The adversarial construction reveals exactly what is required to push $\rho$ above 1: a sudden \emph{distribution shift} where a long period of small gradients is followed by large gradients. The fast-responding first moment ($\beta_1 = 0.9$, half-life $\approx 7$ steps) spikes immediately, while the slow-responding second moment ($\beta_2 = 0.999$, half-life $\approx 700$ steps) takes much longer to catch up. This transient mismatch allows $\rho > 1$ temporarily.

RL fine-tuning lacks the adversarial structure required to push $\rho$ significantly above 1. The adversarial construction requires a long ``quiet'' period of near-zero gradients followed by sudden large gradients. As shown in \Cref{app:gradient_sparsity}, gradients are dense throughout training (${\sim}$99\% non-zero at every step), precluding such quiet periods. The consistently high observed sparsity (${\sim}$99\%) is consistent with $\rho$ remaining near 1, though directly verifying this would require measuring $|\hat{m}_t|/\sqrt{\hat{v}_t}$ during training.

\paragraph{Mean-field intuition.} A simple argument further supports the expectation that $\rho \approx 1$ under stable training conditions. If gradient statistics are approximately stable over time (mean $\mu$, variance $\sigma^2$), then Adam's bias-corrected moments converge to $\hat{m}_t \to \mu$ and $\hat{v}_t \to \mu^2 + \sigma^2$, giving a ratio of $|\mu|/\sqrt{\mu^2 + \sigma^2} \le 1$. In other words, when gradients have any variance at all, the second moment grows faster than the first, keeping the ratio below 1. This is only a heuristic: it bounds the ratio of expectations rather than the per-step ratio $\rho_t$ itself, and we have not verified stationarity empirically. We plan to validate this by measuring $|\hat{m}_t|/\sqrt{\hat{v}_t}$ directly during training in future work. Nonetheless, combined with the consistently high observed sparsity (${\sim}$99\%), it motivates the \emph{effective} critical scale $|w|_{\text{crit}} \approx 256\eta$ rather than the worst-case scale $|w|_{\text{crit}} = 2560\eta$ derived from the upper bound in \Cref{thm:adam_bound}.

\subsection{Conditions for Sparse Adam Updates}
\label{app:sparse_conditions}

\Cref{tab:sparse_conditions} enumerates the conditions under which Adam updates become sparse (absorbed by BF16 precision). We detail each condition below.

\paragraph{Condition 1: Very small gradients ($|g| \ll \epsilon$).}
When gradients are extremely small (e.g., $|g| = 10^{-12}$), the Adam update simplifies to $|\Delta w| \approx \eta \cdot |g| / \epsilon$. With $\eta = 3 \times 10^{-6}$ and $\epsilon = 10^{-8}$, this gives $|\Delta w| \approx 3 \times 10^{-10}$, far below the absorption threshold for any weight magnitude. This condition is rare in practice since RL gradients are typically dense and non-negligible.

\paragraph{Condition 2: Oscillating gradients ($m_t \to 0$).}
When gradients oscillate around zero (e.g., alternating $+g$ and $-g$), the first moment $m_t$ cancels while the second moment $v_t$ accumulates: $v_t \approx g^2$. This yields $|\Delta w| \approx \eta \cdot 0 / |g| \approx 0$. This occurs for parameters where the gradient sign changes frequently across batches.

\paragraph{Remark: Adam updates depend on temporal dynamics, not gradient magnitude.}
A key property of Adam is that update magnitude depends primarily on \emph{how gradients change over time}, not their absolute scale. When $\sqrt{\hat{v}_t} \gg \epsilon$ (which holds for typical gradient magnitudes $|g| \gg 10^{-8}$), the ratio $\hat{m}_t/\sqrt{\hat{v}_t}$ is approximately scale-invariant: scaling all gradients by $k$ scales both numerator and denominator equally. For constant gradients in this regime, $\rho \approx 1$ regardless of magnitude, so $|\Delta w| \approx \eta$. What causes $\rho$ to deviate from 1 is \emph{temporal variation}: when gradient statistics shift, the fast-responding $m_t$ ($\beta_1 = 0.9$) and slow-responding $v_t$ ($\beta_2 = 0.999$) temporarily diverge (\Cref{fig:adam_ratio_adversarial}). This approximate scale-invariance explains why Adam's sparsity behavior is predictable across different gradient regimes.

\paragraph{Condition 3: Large weight magnitudes ($|w| \gtrsim 10^{-2}$).}
The BF16 rounding-cell radius scales with weight magnitude: updates smaller than $|w|/256$ remain inside the cell when the FP32 master is near the cell center. Even at the absorption bound ($|\Delta w| \approx 10\eta$), weights with $|w| > 2560\eta$ have a single-step update that is small relative to a cell, so a step rarely crosses a cell boundary. For $\eta = 3 \times 10^{-6}$, this scale is $|w| > 7.68 \times 10^{-3}$. Since typical LLM weights have $|w| \approx 0.01$, this is the \textbf{dominant} driver of sparsity.

\paragraph{Condition 4: Small learning rate.}
The learning rate $\eta$ directly scales all updates: $|\Delta w| = \eta \cdot |\rho_t|$. Standard RL fine-tuning uses $\eta \approx 10^{-6}$, which is 100--1000$\times$ smaller than pre-training rates. This amplifies the absorption effect, making it a \textbf{dominant} factor alongside weight magnitudes.

\paragraph{Condition 5: Momentum effects from $\beta_2$.}
The second moment decay rate $\beta_2$ affects how $v_t$ tracks gradient history. With high $\beta_2$ (e.g., 0.999, half-life $\approx 700$ steps), past gradients influence $v_t$ for longer. This can suppress updates when past gradients were large (inflated $v_t$), but can also amplify updates when past gradients were small (the adversarial case enabling $\rho > 1$). Lower $\beta_2$ (e.g., 0.95) yields a tighter absorption bound ($\sqrt{2}\eta$ vs $10\eta$). In stationary regimes where $\rho \approx 1$, the choice of $\beta_2$ has minimal effect on sparsity.

\begin{table}[t]
\centering
\caption{Conditions leading to sparse Adam updates in RL fine-tuning.}
\label{tab:sparse_conditions}
\vspace{0.5em}
\small
\begin{tabular}{@{}p{3.5cm}p{5cm}p{3cm}@{}}
\toprule
\textbf{Condition} & \textbf{Mechanism} & \textbf{Prevalence} \\
\midrule
Small gradients ($|g| \ll \epsilon$) & Update numerator $\to 0$ & Rare \\
Oscillating gradients & $m_t \to 0$ by cancellation & Moderate \\
Large weights ($|w| > 10^{-3}$) & BF16 threshold $|w|/256$ too high & \textbf{Dominant} \\
Small learning rate ($\eta = 3 \times 10^{-6}$) & All updates scaled down & \textbf{Dominant} \\
$\beta_2$ momentum effects & $v_t$ lags gradient changes & Context-dependent \\
\bottomrule
\end{tabular}
\end{table}

\subsection{Optimizer Dependence}
\label{app:optimizer_dependence}

The sparsity analysis throughout this paper assumes Adam-style optimizers. This choice is not incidental: Adam's adaptive scaling fundamentally changes how gradient magnitudes translate to update magnitudes, making sparsity robust in ways that would not hold for SGD.

\paragraph{Adam vs.\ SGD update dynamics.}
In SGD, the update is $\Delta w = \eta g$, so update magnitude scales directly with gradient magnitude. Large gradients produce large updates that may exceed the BF16 absorption threshold. In contrast, Adam computes $\Delta w = \eta \cdot \hat{m}_t / (\sqrt{\hat{v}_t} + \epsilon)$, where both $\hat{m}_t$ and $\hat{v}_t$ track gradient statistics. When gradients are consistently large, both the first moment $m_t$ and second moment $v_t$ grow proportionally, keeping their ratio bounded near 1 (\Cref{fig:adam_ratio_adversarial}). This normalization effect means that Adam's update magnitude is largely independent of gradient magnitude.

\paragraph{Upper bounds on updates.}
A key distinction emerges when gradient clipping is disabled or ineffective. Adam has a theoretical upper bound on update magnitude regardless of gradient size (\Cref{thm:adam_bound}), while SGD has no such bound: $|\Delta w| = \eta|g|$ grows without limit as $|g|$ increases. This bound makes Adam's sparsity predictable even under extreme gradient conditions.

\paragraph{Implications for sparsity.}
In practice, layer normalization and gradient clipping constrain gradient magnitudes. With clipping at global norm 1.0, per-parameter gradients are often $|g| \leq 1$, which means SGD updates ($\eta|g|$) may be smaller than Adam updates (${\sim}\eta$), potentially yielding comparable or even higher sparsity. However, without clipping or when gradients spike, SGD's unbounded updates could significantly reduce sparsity. We have not empirically verified sparsity under SGD; our analysis assumes Adam throughout.

\paragraph{Practical relevance.}
Since modern LLM training universally uses Adam variants (AdamW, Adam with decoupled weight decay), this distinction is primarily of theoretical interest. However, practitioners considering alternative optimizers (e.g., Muon) should be aware that the sparsity guarantees established in this paper may not transfer.

\section{PULSELoCo Sparse Payloads}
\label{app:pulseloco_sparsity}

This appendix reports the sparse-payload measurements used for PULSELoCo in \Cref{sec:experiments}. DiLoCo synchronizes the full FP32 pseudo-gradient by construction, so it serves as the dense reference. For PULSELoCo, we log two related but distinct quantities: BF16 weight-update sparsity between consecutive global checkpoints, and FP32 pseudo-gradient communication sparsity after error feedback is applied.

\Cref{fig:pulseloco_sparsity_bar} summarizes the operating points used in the main comparison. The BF16 weight-update measurement is the PULSESync patch that would synchronize each PULSELoCo global checkpoint to inference workers; it is not the standalone PULSESync regime evaluated in \Cref{app:grail}. Because each global checkpoint includes $H$ local steps and one outer update, this paired setting is moderately denser than per-step PULSESync, with $93.9$--$95.7\%$ sparsity. The trainer-to-trainer pseudo-gradient payload remains $94.8$--$96.4\%$ sparse, so each PULSELoCo worker communicates only $3.6$--$5.2\%$ of FP32 pseudo-gradient values per outer round. This gives a $19$--$28\times$ FP32-value reduction relative to dense DiLoCo at the same outer-round cadence. After index metadata, all measured operating points still use at least $12\times$ less bandwidth than dense DiLoCo (\Cref{tab:bw_operating_points}).

\begin{figure}[t]
\centering
\includegraphics[width=0.82\linewidth]{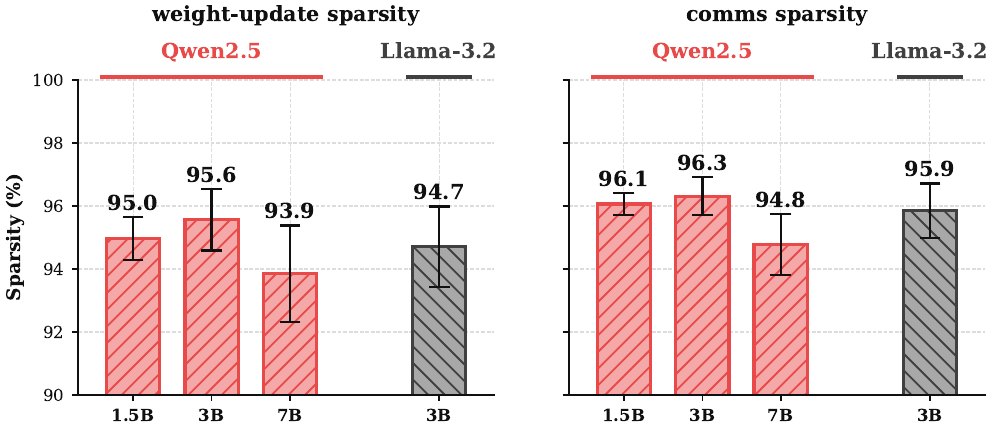}
\caption{\textbf{PULSELoCo sparsity at the main operating points.} Left: BF16 checkpoint-patch sparsity for PULSESync when paired with PULSELoCo, measured between consecutive global checkpoints. This setting is denser than standalone per-step PULSESync because each checkpoint includes $H$ local steps and one outer update. Right: pseudo-gradient communication sparsity after error feedback, which determines the trainer-to-trainer sparse payload. Qwen2.5 models use $H=8$ and Llama-3.2-3B uses $H=4$. Error bars indicate $\pm 1$ standard deviation across logged outer rounds and seeds.}
\label{fig:pulseloco_sparsity_bar}
\end{figure}

\begin{table}[t]
\centering
\caption{\textbf{PULSELoCo communication sparsity and FP32-value savings.} Means are computed by averaging each seed over its logged outer rounds and then averaging across the 3 seeds used in \Cref{fig:learning_curves}. The FP32-value reduction reports the dense FP32 baseline divided by the fraction of entries selected by PULSELoCo; byte-level accounting additionally includes sparse indices, as described in \Cref{app:bandwidth_accounting}.}
\label{tab:pulseloco_sparse_payloads}
\footnotesize
\setlength{\tabcolsep}{5pt}
\begin{tabular}{@{}lcccc@{}}
\toprule
Model & $H$ & Communication sparsity & FP32 values sent & FP32-value reduction \\
\midrule
Qwen2.5-1.5B & 8 & $96.1\%$ & $3.9\%$ & $25.5\times$ \\
Qwen2.5-3B & 8 & $96.4\%$ & $3.6\%$ & $27.8\times$ \\
Qwen2.5-7B & 8 & $94.8\%$ & $5.2\%$ & $19.1\times$ \\
Llama-3.2-3B & 4 & $95.9\%$ & $4.1\%$ & $24.5\times$ \\
\bottomrule
\end{tabular}
\end{table}

\noindent\textbf{Sensitivity to $H$.} The H-sweep in \Cref{app:pulseloco_h_sensitivity} shows that increasing $H$ modestly increases the sent fraction, but the pseudo-gradient payload remains sparse throughout the stable local-update range.

\section{Compression Algorithm Selection}
\label{app:compression}

The sparse patches produced by \protoname (\Cref{subsec:pulsesync}) are further compressed with a general-purpose entropy coder before transmission. This appendix motivates the codec used by default for both \protoname algorithms. The choice of codec is regime-dependent: at high bandwidth the encoding throughput dominates end-to-end latency, while at low bandwidth the achieved compression ratio dominates. We evaluate five widely-deployed codecs on the production sparse representation \texttt{delta\_coo\_downscaled} (\Cref{app:compression_ablation}) and identify the operating regime in which each codec minimizes total transfer time.

\paragraph{Codecs evaluated.} We benchmark snappy, lz4, zstd at compression levels 1 and 3, and gzip at level 6. The first three target the speed end of the spectrum, zstd-3 sits at a moderate-ratio operating point, and gzip-6 serves as a universal-baseline reference. All measurements use 1 warmup iteration plus 3 timed iterations on an AMD EPYC 7763 host, with $n = 270$ sparse-checkpoint payloads drawn from 14 GRPO experiments across the Qwen2.5, Gemma-3, and LLaMA-3.2 families. We verified that \texttt{decompress(compress(x)) == x} on every payload.

\begin{table}[t]
    \centering
    \caption{\textbf{Codec comparison on the production sparse representation.} Sparse ratio is measured against the COO baseline; full ratio is measured against the dense BF16 model. Encode throughput is reported on a single AMD EPYC 7763 core ($n = 270$). The crossover bandwidth is the link rate at which a codec begins to dominate the next-faster Pareto neighbour. The \emph{best regime} column gives the bandwidth tier in which each codec minimizes end-to-end transfer time for a 7B model at 99\% sparsity after BF16 casting.}
    \label{tab:codec_selection}
    \small
    \begin{tabular}{@{}lccccl@{}}
        \toprule
        \textbf{Codec} & \textbf{Sparse Ratio} & \textbf{Full Ratio} & \textbf{Encode (MB/s)} & \textbf{Decode (MB/s)} & \textbf{Best regime} \\
        \midrule
        snappy           & $2.41\times \pm 0.15$ & $56\times$ & $1041 \pm 357$ & $1289 \pm 485$ & datacenter ($>$800\,Mbit/s) \\
        lz4              & $2.40\times \pm 0.13$ & $56\times$ & $\phantom{0}830 \pm 236$ & $1484 \pm 524$ & datacenter ($>$800\,Mbit/s) \\
        \textbf{zstd-1}  & $\mathbf{3.33\times \pm 0.29}$ & $\mathbf{79\times}$ & $\mathbf{534 \pm \phantom{00}56}$ & $\mathbf{851 \pm 108}$ & \textbf{cloud (14--800\,Mbit/s)} \\
        zstd-3           & $3.40\times \pm 0.27$ & $80\times$ & $\phantom{0}197 \pm \phantom{00}21$ & $\phantom{0}670 \pm \phantom{0}69$ & constrained ($<$14\,Mbit/s) \\
        gzip-6           & $3.32\times \pm 0.26$ & $78\times$ & $\phantom{00}14 \pm \phantom{000}2$ & $\phantom{0}192 \pm \phantom{0}11$ & dominated (never optimal) \\
        \bottomrule
    \end{tabular}
\end{table}

\Cref{tab:codec_selection} summarizes the comparison and \Cref{fig:codec_selection_app} visualizes the resulting Pareto structure across bandwidth tiers. We observe a clean three-codec frontier: lz4 wins at high bandwidth, zstd-1 wins across the typical cloud regime, and zstd-3 wins on constrained links. snappy is essentially indistinguishable from lz4 on our payloads (within one standard deviation on both ratio and throughput), so we treat the two as exchangeable at the high-bandwidth end. gzip-6 is dominated everywhere: it matches the zstd-1 ratio ($3.32\times$ vs $3.33\times$) but encodes approximately $38\times$ more slowly (14 vs 534\,MB/s), so it is never on the Pareto frontier.

\begin{figure}[t]
    \centering
    \includegraphics[width=0.85\linewidth]{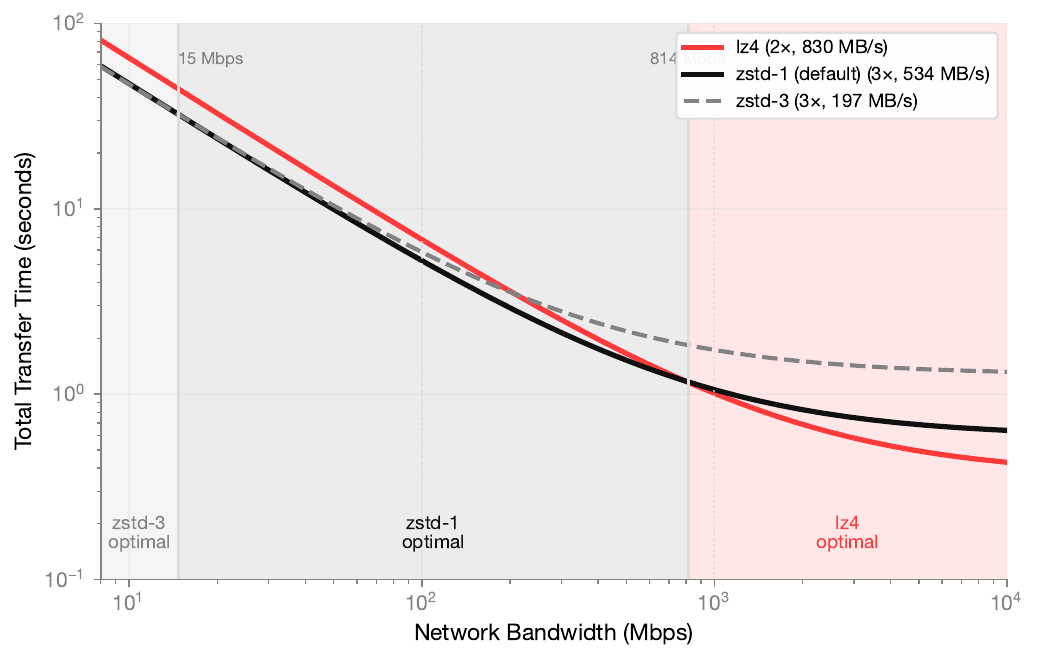}
    \caption{\textbf{Bandwidth-aware codec selection.} End-to-end transfer time (encode + network + decode) versus link bandwidth for a 7B model at 99\% sparsity after BF16 casting. Shaded regions mark the bandwidth tier in which each codec minimizes total time. The crossover from zstd-3 to zstd-1 occurs near 14\,Mbit/s, and the crossover from zstd-1 to lz4 occurs near 800\,Mbit/s; both crossovers shift to higher bandwidth as the payload grows.}
    \label{fig:codec_selection_app}
\end{figure}

\paragraph{Regime selection.} We recommend the following defaults based on the crossovers in \Cref{tab:codec_selection} and \Cref{fig:codec_selection_app}:
\begin{itemize}[leftmargin=*, itemsep=2pt, topsep=2pt]
    \item \textbf{Datacenter} ($>$800\,Mbit/s; e.g.\ NVLink, 10\,GbE, intra-rack): use lz4. Encoding consumes most of the latency budget at this rate, so the $830$\,MB/s encoder beats the higher-ratio alternatives even though the achieved ratio is lower ($56\times$ full).
    \item \textbf{Typical cloud} (14\,Mbit/s to 800\,Mbit/s; e.g.\ commodity internet, cross-region links): use zstd-1. This is the \protoname default. zstd-1 minimizes end-to-end latency across the majority of realistic deployment links while delivering $79\times$ full compression.
    \item \textbf{Constrained} ($<$14\,Mbit/s; slow WAN, tethered links): use zstd-3. Network time dominates and the marginal $+2\%$ ratio over zstd-1 outweighs the lower encoder throughput.
\end{itemize}
The crossovers depend on payload size; larger patches push both crossovers toward higher bandwidth, since transfer time scales linearly with payload while encode and decode times do not. \Cref{app:bandwidth_selection} gives the closed-form crossover expression and reports values for a representative 194\,MB payload.

\paragraph{Per-model sensitivity.} The codec ranking is stable across the model families we studied (\Cref{app:per_model_compression}): zstd-1 produced full ratios of $76\times$, $80\times$, and $100\times$ on Qwen2.5, Gemma-3, and LLaMA-3.2 respectively, with the relative ordering of codecs unchanged. The largest variation is in absolute ratio rather than in which codec wins at a given bandwidth, so the regime recommendations above transfer across architectures without adjustment. Architectures with markedly different weight distributions or tokenizer-induced embedding patterns may shift the absolute ratios; we recommend re-running the benchmark in \Cref{app:algorithm_comparison} when deploying \protoname on a new model family with a substantially different parameter distribution.



\section{Lower-Precision Receivers}
\label{app:lower_precision}

The compute-visibility gate $G_D(\theta, s) := \{ i : \mathrm{cast}_D(\theta_i) \neq \mathrm{cast}_D(\theta_i - s_i) \}$ is parametric in the compute dtype $D$. The main paper instantiates $D = \mathrm{BF16}$ throughout (\Cref{subsec:cvs}), since BF16 is the dominant deployment format for current RL post-training pipelines~\citep{micikevicius2018mixed}. We now ask the natural follow-up question: how might the gate behave when receivers run inference in FP8 E4M3 or MXFP4? We answer this with a first-order projection, not an end-to-end measurement: we carry the BF16 ULP derivation of \Cref{app:sparsity_foundations} through to the lower-precision formats and match the resulting thresholds against the measured weight-magnitude distribution. The projection suggests that the gate should strengthen rather than weaken at lower precision, but the FP8 and MXFP4 numbers in \Cref{tab:tulpscale} should be read as estimates.

\noindent\textbf{Setup.} \Cref{app:sparsity_foundations} shows that BF16 absorbs an Adam parameter update whenever $|\Delta_i| \lesssim |w_i| / 256$. The factor $256 = 2^{8}$ has a clean ULP origin: BF16 carries 7 mantissa bits, so adjacent representable values within a single binade are spaced by $2^{-7}|w_i| = |w_i|/128$, and a midpoint-rounding argument places the absorption boundary at half a ULP, giving the asymmetric cell bound $|w_i|/256$. To project the gate onto FP8 E4M3 and MXFP4, we repeat this construction with each format's mantissa width and any block-scale adjustments, then compose the resulting threshold with the standard Adam update bound to derive a per-format critical weight magnitude.

\noindent\textbf{Derivation.} Let $m_D$ denote the mantissa bit count of format $D$, and write the relative absorption threshold as
\begin{equation}
    \tau_D := |\Delta_i| / |w_i| \quad \text{at which} \quad \mathrm{cast}_D(w_i - \Delta_i) = \mathrm{cast}_D(w_i),
    \label{eq:tau_D}
\end{equation}
so that $\tau_{\mathrm{BF16}} = 2^{-(m_{\mathrm{BF16}}+1)} = 2^{-8}$ matches the bound recalled above. For FP8 E4M3, $m_D = 3$, and the same midpoint-rounding argument gives $\tau_{\mathrm{FP8}} = 2^{-4} = 1/16$. For MXFP4, the picture has one extra wrinkle: the format stores 4 bits per element (sign + 2 exponent + 1 mantissa, the OCP E2M1 layout) but shares an 8-bit exponent scale across each block of 32 elements. Within a block, the per-element ULP is set by the block scale rather than the element exponent, which coarsens the effective spacing. Treating the block scale as fixed during a single optimizer step, the per-element relative threshold reduces to $\tau_{\mathrm{MXFP4}} = 2^{-2} = 1/4$ for elements near the block maximum, with smaller elements seeing an even coarser \emph{absolute} cutoff because their representable values are farther apart relative to their magnitude. We use the optimistic per-element value $\tau_{\mathrm{MXFP4}} = 1/4$ throughout; the projection should therefore be read as a lower bound on the true sparsity floor.

Using the effective Adam update scale $|\Delta_i| \approx \eta$ (motivated by the analysis in \Cref{app:sparsity_foundations}) gives the per-format critical magnitude
\begin{equation}
    |w|^{D}_{\mathrm{crit}} = \eta / \tau_D,
    \label{eq:wcrit_D}
\end{equation}
above which a parameter's update is absorbed and below which it survives the cast. At the standard RL learning rate $\eta = 3 \times 10^{-6}$, \Cref{eq:wcrit_D} gives $|w|^{\mathrm{BF16}}_{\mathrm{crit}} \approx 7.7 \times 10^{-4}$, $|w|^{\mathrm{FP8}}_{\mathrm{crit}} \approx 4.8 \times 10^{-5}$, and $|w|^{\mathrm{MXFP4}}_{\mathrm{crit}} \approx 1.2 \times 10^{-5}$. The fraction of weights that survive the cast is the fraction of parameters with $|w_i| < |w|^{D}_{\mathrm{crit}}$; conversely, the per-step sparsity floor is the fraction with $|w_i| \ge |w|^{D}_{\mathrm{crit}}$.

\begin{table}[t]
    \centering
    \caption{\textbf{T-ULP-Scale: projected absorption thresholds at $\eta = 3 \times 10^{-6}$.} The BF16 row is anchored to the empirical sparsity measurements in \Cref{sec:sparsity_analysis}; the FP8 E4M3 and MXFP4 rows are projections obtained by carrying the ULP derivation of \Cref{app:sparsity_foundations} through to the lower-precision formats and matching against the weight-magnitude distribution measured in \Cref{app:weight_magnitude} on Qwen2.5-1.5B. We treat MXFP4 with an OCP-style block scale of 32 elements and use the optimistic per-element value $\tau_{\mathrm{MXFP4}} = 1/4$. The projected sparsity values are not measurements and depend on scaling, rounding mode, and hardware support for the corresponding receiver format.}
    \label{tab:tulpscale}
    \footnotesize
    \setlength{\tabcolsep}{4pt}
    \begin{tabular}{@{}lccccc@{}}
    \toprule
    \textbf{Format} & \textbf{Mantissa bits} & $\boldsymbol{\tau_D}$ & $\boldsymbol{|w|^{D}_{\mathrm{crit}}}$ & \textbf{Frac.\ above} & \textbf{Projected sparsity} \\
    \midrule
    BF16 (baseline) & 7 & $2^{-8} = 1/256$ & $7.7 \times 10^{-4}$ & 97.6\% & ${\sim}99\%$ \\
    FP8 E4M3        & 3 & $2^{-4} = 1/16$  & $4.8 \times 10^{-5}$ & 99.5\% & ${\sim}99.7\%$ \\
    MXFP4 (E2M1 + block scale) & 1 (eff.) & $2^{-2} = 1/4$ & $1.2 \times 10^{-5}$ & 99.8\% & ${\sim}99.85\%$ \\
    \bottomrule
    \end{tabular}
\end{table}

The \textbf{Frac.\ above} column reports the fraction of Qwen2.5-1.5B weights with $|w_i| \ge |w|^{D}_{\mathrm{crit}}$, computed from the cumulative magnitude distribution. The \textbf{Projected sparsity} column adds a small heuristic adjustment, calibrated from the gap between the BF16 magnitude-only estimate and the empirical BF16 sparsity. It is intended to show scale, not to claim exact sparsity under FP8 or MXFP4 execution.


\noindent\textbf{Implications.} Lower precision should strengthen the compute-visibility gate because coarser formats have larger rounding cells. If the projected sparsity levels hold in deployment, PULSESync and PULSELoCo would transmit fewer parameters than under BF16. The exact gain must be measured on the target inference format and hardware.

\noindent\textbf{Caveats.} The projection is a first-order ULP scaling; three caveats temper it. First, rounding mode matters: round-to-nearest-even gives the asymmetric $|w|/(2 \cdot 2^{m_D})$ cell bound used in \Cref{eq:tau_D}, while stochastic rounding (which some FP8 training stacks adopt) inflates the survival probability near the cell boundary by a magnitude-dependent factor and would lower the projected floor by approximately 0.1 to 0.3 percentage points. Second, FP8 E4M3 and MXFP4 have narrower dynamic range than BF16, so practical deployment typically inserts per-tensor or per-block scaling that shifts the effective ULP relative to the unscaled weight magnitude; the projection assumes this scaling is set so that median-magnitude weights remain in the dense region of the format, as is standard practice. Third, this appendix does not run end-to-end FP8 or MXFP4 training experiments; the table predicts what the gate would observe if PULSE were deployed on receivers running inference in the corresponding format, and the actual measurement requires the matching hardware and a calibrated FP8/MXFP4 inference path.

\section{PULSESync Deployment on grail}
\label{app:grail}

PULSESync is deployed as the weight-synchronization layer on grail, a decentralized reinforcement learning platform. This deployment does not use PULSELoCo; trainer-to-trainer pseudo-gradient synchronization is evaluated separately in \Cref{sec:experiments}. This section summarizes grail's asynchronous architecture.

\subsection{System Overview}
\label{app:grail_overview}

grail separates computationally expensive inference (rollout generation) from training, enabling distributed nodes to contribute compute while a centralized trainer handles gradient updates. The system comprises three node types:

\begin{itemize}[leftmargin=*, itemsep=2pt]
    \item \textbf{Miners}: Generate inference rollouts using the current model checkpoint.
    \item \textbf{Validators}: Verify rollout authenticity via hidden-state fingerprinting and assign performance-based rewards.
    \item \textbf{Trainer}: Consumes verified rollouts to update the model.
\end{itemize}

All coordination occurs through S3-compatible object storage (e.g., Cloudflare R2), which serves as the shared layer for checkpoints and rollout data.

\subsection{Asynchronous Training Architecture}
\label{app:grail_async}

grail employs a fully asynchronous design where the trainer runs continuously without synchronization stalls. Miners and the trainer synchronize only at \emph{window boundaries} (approximately every 6 minutes), but the trainer never blocks. Instead, it continuously samples from a replay buffer while dedicated background processes handle all I/O.

\paragraph{Trainer node processes.} The trainer node runs three concurrent processes:

\begin{enumerate}[leftmargin=*, itemsep=2pt]
    \item \textbf{Training process}: Executes a tight loop that samples batches from the replay buffer and performs gradient updates. This process never waits for I/O, enabling multiple updates per window.

    \item \textbf{Upload process}: Handles checkpoint serialization and upload asynchronously. When the trainer produces a new checkpoint, it is handed off to this process without blocking.

    \item \textbf{Download process}: Fetches verified rollouts from storage at window boundaries and adds them to the replay buffer with staleness metadata.
\end{enumerate}

\paragraph{Replay buffer.} The replay buffer decouples data arrival from training consumption. It stores rollouts from multiple windows, supports staleness-weighted sampling (preferring fresher data), and implements automatic eviction of stale entries. This design ensures the trainer always has data available, even during network delays.

\subsection{Rollout Verification}
\label{app:grail_verification}

Validators verify that rollouts originate from the correct model checkpoint using a lightweight cryptographic mechanism called \emph{grail Proof}:

\begin{itemize}[leftmargin=*, itemsep=2pt]
    \item Select the top-32 hidden-state dimensions per token.
    \item Apply logarithmic quantization to handle heavy-tailed activation distributions.
    \item Generate 4-byte cryptographic sketches per token (${\sim}$148 bits of security).
    \item Use adaptive tolerances to account for numerical drift across different hardware.
\end{itemize}

This verification ensures that miners cannot submit rollouts generated from outdated or modified checkpoints.

\subsection{Deployment Setup}
\label{app:grail_setup}

To demonstrate the domain-agnostic nature of PULSESync, we evaluate on two distinct tasks: (1) mathematical reasoning using the MATH dataset~\citep{hendrycks2021math} with Qwen2.5-7B-Instruct~\citep{qwen2.5}, identical to the setup in \Cref{sec:sparsity_analysis}, and (2) code generation using the MBPP dataset~\citep{austin2021program} with Qwen2.5-Coder-7B-Instruct. For MBPP, we use 774 tasks for training and 190 for validation, with rewards based on test pass rates and syntax validity. For each task, we run 3 independent trials with the same GRPO implementation and base hyperparameters in \Cref{app:hyperparameters}, except that we use a lower learning rate ($1 \times 10^{-6}$) to ensure training stability in the distributed setting.

\subsection{Bandwidth Reduction}
\label{app:grail_bandwidth}

\Cref{fig:grail_training_curves} demonstrates that the high sparsity observed in \Cref{sec:sparsity_analysis} translates directly to PULSESync communication savings in practice. Upload sizes average $108$\,MB (SE: $1.1$\,MB), more than $100\times$ smaller than the $14$\,GB required for full 7B-model synchronization. At the mean, PULSESync achieves approximately $130\times$ bandwidth reduction at this learning rate, exceeding the $79\times$ observed at the benchmark learning rate ($3 \times 10^{-6}$) in the codec analysis of \Cref{app:compression}. The improvement is consistent with the higher sparsity induced by the lower learning rate, as predicted by our analysis in \Cref{subsec:precision}.

\subsection{Training Effectiveness}
\label{app:grail_training}

Despite PULSESync transmitting only sparse weight updates, training proceeds normally. Validation pass@1 improves steadily throughout training, reaching final improvements of $+50.1$ and $+49.4$ percentage points on MATH and MBPP respectively. Standard deviation across runs remains low ($\le 1.5$ percentage points). The main-text \Cref{fig:grail_training_curves} plots the per-window validation accuracy and upload sizes over the duration of training; each window is approximately 6 minutes during which up to 8 gradient steps may occur, with the exact count varying due to the system's asynchronous architecture.

\subsection{Lossless Reconstruction}
\label{app:grail_lossless}

All weight transfers pass a checksum verification, confirming bitwise-identical reconstruction at inference nodes. This validates the core premise of PULSESync: the sparsity induced by BF16 precision enables \emph{lossless} compression without approximation error or error feedback mechanisms. The verification protocol embeds a collision-resistant checksum of the post-patch weights into each patch header; receivers recompute the checksum after decoding and compare against the embedded value, rejecting any patch whose reconstruction does not match exactly.

\section{Experimental Details}
\label{app:experimental_details}

\subsection{Hardware Configuration}
\label{app:hardware}

We use different hardware configurations for the sparsity analysis (\Cref{sec:sparsity_analysis}), the PULSELoCo comparison (\Cref{sec:experiments}), and the grail deployment study (\Cref{app:grail}).

\paragraph{Sparsity analysis (\Cref{sec:sparsity_analysis}).}
For the controlled sparsity experiments, we use the same GPU classes as the PULSELoCo comparison: an NVIDIA B300 SXM5 GPU for Qwen2.5-7B-Instruct and an NVIDIA A100 SXM4-80GB GPU for the smaller Qwen, Llama, and Gemma models, each paired with a second GPU of the same class for inference and evaluation. This setup ensures reproducible measurements of weight update sparsity across different model sizes and training configurations.

\paragraph{PULSELoCo comparison (\Cref{sec:experiments}).}
The PULSELoCo, DiLoCo, and DDP comparison runs on two GPU classes. Each cell is one (algorithm, model, seed) configuration with four GPUs and intra-node NVLink for inter-rank communication. The Qwen2.5-7B-Instruct sweep uses NVIDIA B300 SXM5 nodes; the smaller-model sweeps use NVIDIA A100 SXM4-80GB nodes.

\paragraph{grail deployment study (\Cref{app:grail}).}
For the grail deployment study, the trainer process runs on a single NVIDIA B200 GPU. The inference nodes are fully decentralized and anonymous, participating voluntarily in the network without disclosing their hardware specifications. The network bandwidth between the trainer and inference nodes is approximately 400 Mb/s. Storage uses S3-compatible object storage for checkpoint distribution.

\subsection{Dataset Details}
\label{app:dataset_details}

\paragraph{MATH dataset.}
For the sparsity analysis (\Cref{sec:sparsity_analysis}), we train on mathematical reasoning tasks using the MATH dataset~\citep{hendrycks2021math}. The dataset contains 7,500 training examples spanning seven subjects: algebra, counting and probability, geometry, intermediate algebra, number theory, prealgebra, and precalculus. Problems range from competition mathematics (AMC, AIME) to olympiad-level difficulty. We extract a stratified 500-example validation split that remains fixed throughout all training runs; the remaining 7,000 problems form our training set. Stratification ensures proportional representation of both subjects and difficulty levels (1--5). The validation set is used exclusively for monitoring training progress; it is never used for gradient updates.

\paragraph{Controlled sparsity model suite.}
The controlled sparsity analysis uses five instruction-tuned checkpoints across three model families: Qwen2.5-0.5B/1.5B/7B-Instruct~\citep{qwen2.5}, Llama-3.2-3B-Instruct~\citep{llama3}, and Gemma-3-4B-it~\citep{gemma3}. This suite lets us test whether update sparsity persists across both architecture family and model scale. Checkpoint identifiers and licenses are listed in \Cref{tab:asset_licenses}.

\paragraph{Generalization to code tasks.}
In \Cref{app:grail}, the PULSESync deployment also evaluates code generation using the MBPP dataset~\citep{austin2021program}.

\paragraph{PULSELoCo comparison suite.}
The DDP, DiLoCo, and PULSELoCo comparison in \Cref{sec:experiments} uses four instruction-tuned models: Qwen2.5-1.5B-Instruct, Qwen2.5-3B-Instruct, Qwen2.5-7B-Instruct~\citep{qwen2.5}, and Llama-3.2-3B-Instruct~\citep{llama3}. We evaluate MATH with 3 independent seeds for each method and model. Rewards follow the same mathematical-reasoning formulation used in the preceding experiments.

\subsection{Bandwidth Accounting}
\label{app:bandwidth_accounting}

PULSELoCo communication is reported as a per-worker payload per outer round. In one outer round, each worker uploads one sparse pseudo-gradient and receives one sparse aggregate before applying the outer optimizer. We report one upload-sized payload. Including the receive side would double both PULSELoCo and DiLoCo, so the reduction ratios would not change. The baseline is DiLoCo's logical full pseudo-gradient payload, $N \times 4$ bytes per worker per outer round for a model with $N$ parameters. These are payload accounting numbers, not end-to-end wire measurements; real network measurements appear only for the grail PULSESync deployment in \Cref{app:grail}.

\paragraph{What is counted.} The sparsity tables in \Cref{app:pulseloco_sparsity} count only selected FP32 pseudo-gradient values, which isolates the compute-visibility gate. This section counts bytes. Byte-level accounting includes both selected FP32 values and the index metadata needed to locate them, so it is more conservative than value sparsity alone.

\paragraph{Sparse stream format.} The main text refers to PULSELoCo's 7B result as an encoded sparse FP32 pseudo-gradient payload. Concretely, PULSELoCo stores selected FP32 values together with sorted parameter indices. The indices are delta-encoded and varint-packed. We first report this packed sparse stream without a general-purpose codec, then separately measure standard byte-stream codecs such as zstd.

\paragraph{Raw sparse payload accounting.} The Qwen2.5-7B-Instruct run at $H{=}8$ has mean communication sparsity $0.948$ in \Cref{tab:pulseloco_sparse_payloads}. For conservative byte accounting, we round this down to $0.940$. This gives approximately $\text{nnz}=4.59 \times 10^8$ transmitted entries out of $N=7.62 \times 10^9$ parameters. The FP32 values require $\text{nnz}\times 4=1.84$\,GB. The index stream is small because sorted index gaps average $N/\text{nnz}\approx 16.6$, so most gaps fit in one varint byte. Bounding the extra varint bytes by $(N-\text{nnz})/127$ gives about $515$\,MB of indices before small container metadata. The resulting raw sparse payload is about $2.36$\,GB; the measured final-round delta-varint payload is $2.39$\,GB, a $12.8\times$ reduction over the dense FP32 baseline of $N\times 4=30.46$\,GB.

\paragraph{Byte-stream compression.} The raw sparse payload applies sparse indexing but no general-purpose codec to the FP32 value stream. We therefore also measure byte-stream compression on the packed sparse stream for the Qwen2.5-7B and Qwen2.5-3B PULSELoCo runs, using the same $N \times 4$ baseline as \Cref{tab:bw_operating_points}. On Qwen2.5-7B at $H{=}8$, the final per-worker payload is $2.39$\,GB with delta-varint indices and raw FP32 values ($12.8\times$), $1.77$\,GB with zstd-1 or zstd-3 ($17.2\times$), and $1.74$\,GB with byte-shuffle plus zstd-3 ($17.5\times$). On Qwen2.5-3B at $H{=}8$, the corresponding payloads are $0.68$\,GB, $0.52$\,GB, and $0.50$\,GB, giving $18.0$--$24.6\times$ reduction. The main text uses the measured 7B zstd-1 payload as the encoded sparse payload in \Cref{fig:hero}; \Cref{fig:pulseloco_compression_ratio} shows the corresponding compression-ratio curves over training, including the byte-shuffle and zstd-3 variants.

\begin{figure}[htbp]
    \centering
    \begin{subfigure}[t]{0.49\linewidth}
        \centering
        \includegraphics[width=\linewidth]{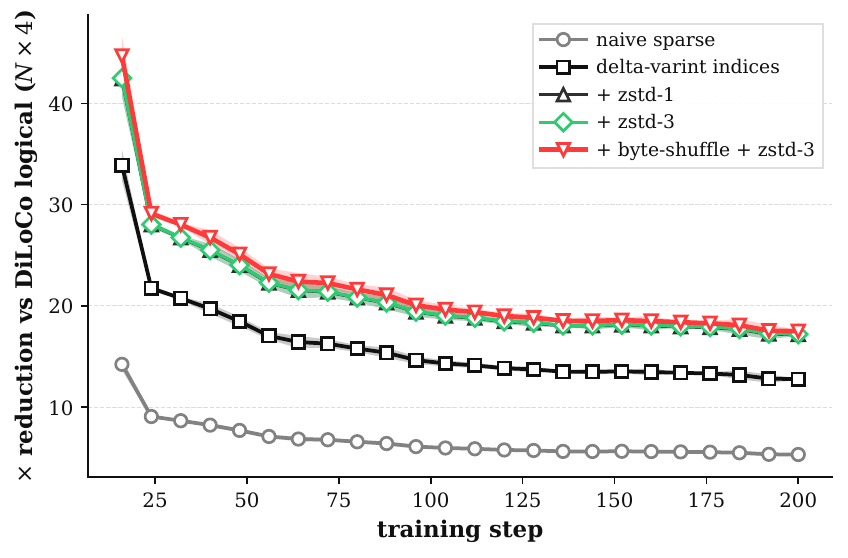}
        \caption{Qwen2.5-7B-Instruct}
    \end{subfigure}
    \hfill
    \begin{subfigure}[t]{0.49\linewidth}
        \centering
        \includegraphics[width=\linewidth]{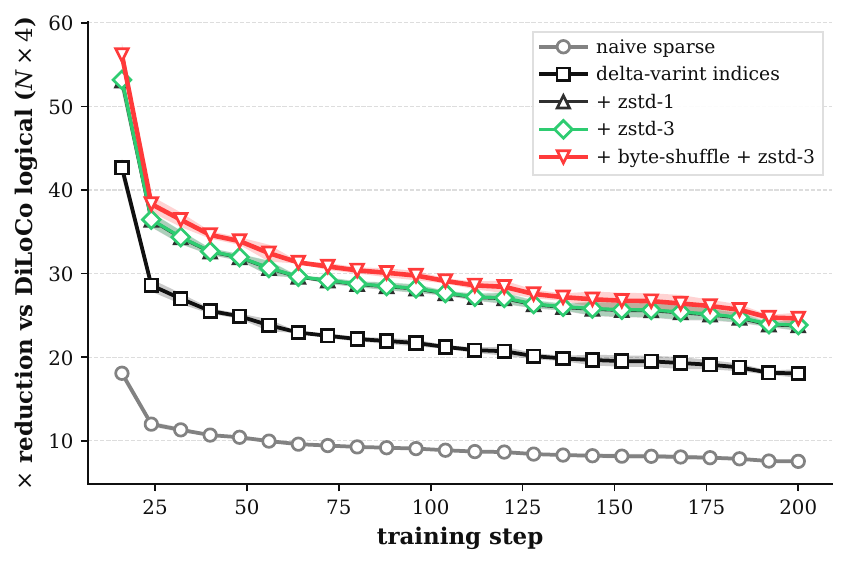}
        \caption{Qwen2.5-3B-Instruct}
    \end{subfigure}
    \caption{\textbf{Compression ratios for PULSELoCo pseudo-gradient payloads.} Ratios are relative to DiLoCo's full FP32 pseudo-gradient payload under the same $N \times 4$ per-worker accounting used in the main text. Curves average 3 seeds and show the effect of index packing and byte-stream codecs on the sparse pseudo-gradient stream.}
    \label{fig:pulseloco_compression_ratio}
\end{figure}

\paragraph{Operating points.} \Cref{tab:bw_operating_points} reports the conservative raw sparse payload for each measured setting. Sparsity generally rises as $H$ falls and as model size shrinks, so the byte-level reduction improves outside the largest 7B, $H{=}8$ setting. The table uses raw sparse payloads; the hero figure instead uses the measured 7B encoded payload with zstd-1.

\paragraph{DDP comparison.} \Cref{tab:bw_operating_points} compares PULSELoCo to DiLoCo's full FP32 pseudo-gradient payload at the same outer-round cadence. A per-step DDP baseline synchronizes once per optimizer step, so over one PULSELoCo outer round it performs $H$ dense synchronizations. Under the same payload accounting, the reduction relative to dense DDP is therefore $H$ times the table value: more than $100\times$ for the $H{=}8$ Qwen settings and $70\times$ for Llama-3.2-3B at $H{=}4$. Exact wire bytes depend on the collective implementation, but the factor of $H$ comes from synchronization frequency and is independent of the sparse codec.

\begin{table}[htbp]
    \centering
    \caption{Bandwidth reduction for PULSELoCo at each measured operating point, using delta-encoded indices and raw FP32 values. \emph{Sparsity} is the conservative value used for byte accounting; \emph{Reduction} is relative to the dense FP32 baseline $N\times 4$. Rows with measured raw payloads use the final-round measurement; other rows use the conservative byte estimate.}
    \label{tab:bw_operating_points}
    \vspace{0.5em}
    \small
    \begin{tabular}{@{}lrrrrr@{}}
    \toprule
    Model & $H$ & Sparsity & Nonzeros/rank & PULSELoCo payload & Reduction \\
    \midrule
    Qwen2.5-7B-Instruct   & 8 & $0.940$ & $4.59 \times 10^8$ & $2.39$\,GB & $12.8\times$ \\
    Qwen2.5-3B-Instruct   & 8 & $0.958$ & $1.30 \times 10^8$ & $0.68$\,GB & $18.0\times$ \\
    Qwen2.5-3B-Instruct   & 4 & $0.971$ & $0.90 \times 10^8$ & $\leq 0.47$\,GB & $\geq 26.1\times$ \\
    Qwen2.5-1.5B-Instruct & 8 & $0.958$ & $0.65 \times 10^8$ & $\leq 0.34$\,GB & $\geq 18.3\times$ \\
    Llama-3.2-3B-Instruct & 4 & $0.954$ & $1.42 \times 10^8$ & $\leq 0.73$\,GB & $\geq 17.5\times$ \\
    \bottomrule
    \end{tabular}
\end{table}

\subsection{Training Hyperparameters}
\label{app:hyperparameters}

We adopt GRPO configurations inspired by DAPO~\citep{yu2025dapo}, as summarized in \Cref{tab:hyperparameters}. To ensure our controlled sparsity analysis captures the \emph{intrinsic} behavior of the optimization process, we set weight decay and the KL penalty $\beta$ to zero during primary measurements. Results are averaged over 4 random seeds.

\paragraph{Optimizer $\beta_2$ settings.}
The controlled sparsity analysis uses $(\beta_1,\beta_2)=(0.9,0.999)$. The grail deployment study and the PULSELoCo experiments use $(0.9,0.95)$, matching the post-training setting used by modern LLM training pipelines. As discussed in \Cref{app:bounds}, $\beta_2=0.999$ gives a looser Adam update bound than $\beta_2=0.95$, so the sparsity characterization is conservative with respect to the deployment and PULSELoCo settings.

\paragraph{Distributed worker learning rate.}
For the grail deployment study and PULSELoCo experiments, we use a lower learning rate of $1 \times 10^{-6}$ to keep workers stable throughout training.

\paragraph{Distributed local-update windows.}
For the DDP, DiLoCo, and PULSELoCo comparison, all methods use $R = 4$ workers. DiLoCo and PULSELoCo use $H = 8$ local Adam steps for the Qwen models and $H = 4$ for Llama-3.2-3B-Instruct.

These values are tied to the shared-inference protocol used for local-update methods. Rollout workers serve the latest global checkpoint and are refreshed only after each outer round, rather than being assigned to individual trainers. During the $H$ local steps inside a round, each trainer has private local weights while rollout workers continue using the previous global checkpoint. Larger $H$ therefore reduces trainer-to-trainer communication, but it also increases the off-policy gap between the rollouts and the local trainer weights. This RL stability constraint is why our local-update windows are smaller than the much longer intervals often used in pre-training DiLoCo settings~\citep{douillard2023diloco}. We use the largest stable windows we found: $H = 8$ for the Qwen models, matching the RL staleness ceiling reported by \citet{scalerl}, and $H = 4$ for Llama-3.2-3B-Instruct, which became unstable for $H > 4$. \Cref{app:pulseloco_round_atomicity} describes the round-level synchronization convention, and \Cref{app:pulseloco_h_sensitivity} reports the payload sensitivity to $H$.

\paragraph{Training duration.}
We train for 400 steps, which is sufficient to observe both early-training dynamics and stable-phase behavior while remaining computationally tractable across multiple model scales. We verify convergence by examining pass@1 accuracy curves on validation sets (\Cref{app:training_curves}); performance plateaus by step 400 in all cases.

\paragraph{Asymmetric clipping.}
Following DAPO, we use asymmetric clipping bounds ($\epsilon_{\text{low}} = 0.2$, $\epsilon_{\text{high}} = 0.28$) to encourage exploration. The higher upper bound relaxes the clipping constraint for positive advantages, mitigating entropy collapse.

\begin{table}[htbp]
    \centering
    \caption{Training hyperparameters for GRPO experiments. Default values are for the controlled sparsity analysis (\Cref{sec:sparsity_analysis}); the grail deployment study (\Cref{app:grail}) and PULSELoCo experiments (\Cref{sec:experiments}) use learning rate $1 \times 10^{-6}$ and $\beta_2 = 0.95$.}
    \label{tab:hyperparameters}
    \vspace{0.5em}
    \small
    \begin{tabular}{@{}ll@{}}
    \toprule
    \textbf{Parameter} & \textbf{Value} \\
    \midrule
    Training steps & 400 \\
    Random seeds & 4 \\
    Optimizer & AdamW \\
    Learning rate ($\eta$) & $3 \times 10^{-6}$ \\
    $(\beta_1, \beta_2)$ & $(0.9, 0.999)$ \\
    Weight decay & $0.0$ \\
    LR schedule & Constant \\
    Gradient clipping & $1.0$ \\
    \midrule
    GRPO clipping $(\epsilon_{\text{low}}, \epsilon_{\text{high}})$ & $(0.2, 0.28)$ \\
    KL penalty $\beta$ & $0.0$ \\
    Prompts per batch & 32 \\
    Rollouts per prompt ($G$) & 16 \\
    Max generation length & 2048 \\
    Precision & BF16 \\
    \bottomrule
    \end{tabular}
\end{table}

\subsection{Reward Formulations}
\label{app:reward_details}

We use verifiable rewards for mathematical reasoning in the controlled experiments and for code generation in the grail deployment study.

\paragraph{Mathematical Reasoning (MATH).}
For math tasks, we use a composite reward with four components: correctness (70\% weight), answer format (15\% weight), thinking presence (10\% weight), and no-trailing penalty (5\% weight):
\begin{equation}
    R_{\text{math}} = 0.7 \cdot C_{\text{correct}} + 0.15 \cdot F_{\text{format}} + 0.1 \cdot T_{\text{thinking}} + 0.05 \cdot P_{\text{no-trailing}}
\end{equation}
where $C_{\text{correct}} \in [0,1]$ is verified by string matching the final answer after normalization.

\paragraph{Code Generation (MBPP).}
For code generation tasks, we use a composite reward based on test pass rates (70\% weight), syntax validity (10\% weight), solution format (10\% weight), and thinking presence (10\% weight):
\begin{equation}
    R_{\text{code}} = 0.7 \cdot C_{\text{pass}} + 0.1 \cdot S_{\text{valid}} + 0.1 \cdot F_{\text{format}} + 0.1 \cdot T_{\text{thinking}}
\end{equation}
where $C_{\text{pass}}$ is the fraction of unit tests passed by the generated code.

\subsection{Software Environment}
\label{app:environment}

The sparsity analysis (\Cref{sec:sparsity_analysis}) uses TRL~\citep{vonwerra2022trl} for GRPO training. The grail deployment study (\Cref{app:grail}) uses the grail training stack, which integrates PULSESync for sparse checkpoint synchronization.

\begin{lstlisting}[basicstyle=\small\ttfamily,frame=single,backgroundcolor=\color{gray!5}]
Python: 3.10.12
PyTorch: 2.2.0
CUDA: 12.1
transformers: 4.38.0
zstandard: 0.22.0
\end{lstlisting}

\subsection{Asset Licenses}
\label{app:licenses}

\Cref{tab:asset_licenses} lists the external datasets and base model checkpoints used in this paper, with their licenses.

\begin{table}[htbp]
    \centering
    \caption{External assets used in the paper, with source identifier and license. All assets are publicly hosted; we use them under their respective terms for non-commercial research benchmarking.}
    \label{tab:asset_licenses}
    \vspace{0.5em}
    \small
    \begin{tabular}{@{}lll@{}}
    \toprule
    \textbf{Asset} & \textbf{Source} & \textbf{License} \\
    \midrule
    MATH~\citep{hendrycks2021math}                & \texttt{EleutherAI/hendrycks\_math}        & MIT \\
    MBPP~\citep{austin2021program}                & \texttt{google-research-datasets/mbpp}     & CC-BY-4.0 \\
    \midrule
    Qwen2.5-1.5B-Instruct~\citep{qwen2.5}         & \texttt{Qwen/Qwen2.5-1.5B-Instruct}        & Apache-2.0 \\
    Qwen2.5-3B-Instruct~\citep{qwen2.5}           & \texttt{Qwen/Qwen2.5-3B-Instruct}          & Qwen Research License \\
    Qwen2.5-7B-Instruct~\citep{qwen2.5}           & \texttt{Qwen/Qwen2.5-7B-Instruct}          & Apache-2.0 \\
    Qwen2.5-Coder-7B-Instruct~\citep{qwen2.5}     & \texttt{Qwen/Qwen2.5-Coder-7B-Instruct}    & Apache-2.0 \\
    Llama-3.2-3B-Instruct~\citep{llama3}          & \texttt{meta-llama/Llama-3.2-3B-Instruct}  & Llama 3.2 Community License \\
    Gemma-3-4B                                    & \texttt{google/gemma-3-4b}                 & Gemma Terms of Use \\
    \bottomrule
    \end{tabular}
\end{table}

\section{Extended Results}
\label{app:extended_results}

\subsection{Gradient vs.\ Parameter Change Sparsity}
\label{app:gradient_sparsity}

To understand the mechanistic source of parameter sparsity, we separately analyze gradient sparsity before the optimizer processes them.

\begin{figure}[t]
    \centering
    \includegraphics[width=\linewidth]{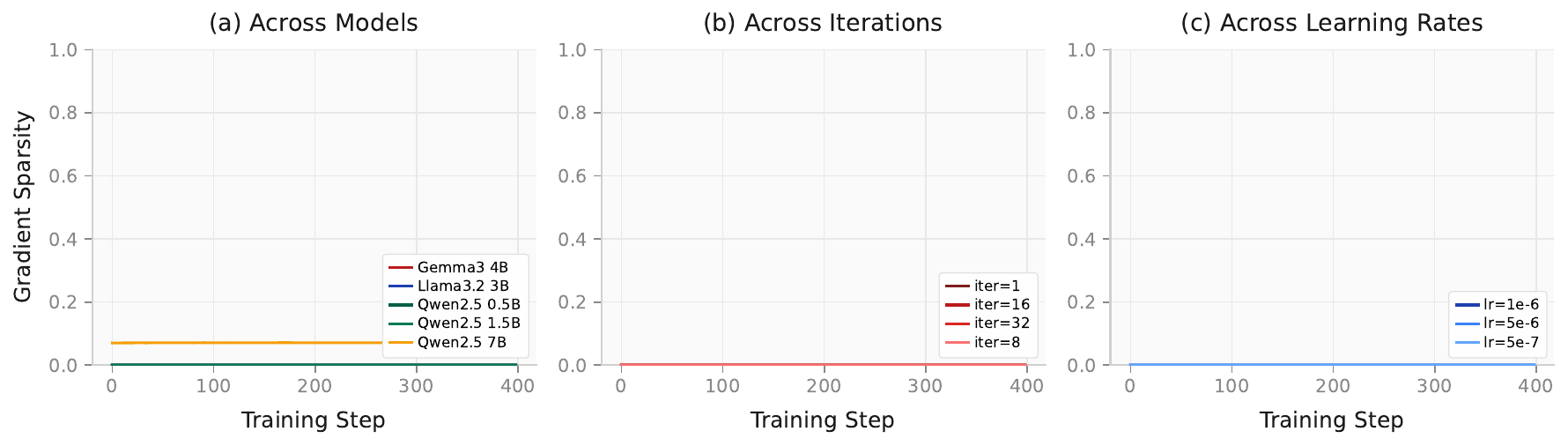}
    \caption{\textbf{Gradient sparsity throughout training for standard GRPO} across (a) model architectures and sizes, (b) iteration counts, and (c) learning rates. Sparsity is measured as the fraction of exactly-zero gradient values. Shaded regions indicate $\pm 1$ standard deviation across 4 seeds. Gradient sparsity remains near zero ($<$1\%) throughout training regardless of model, iteration count, or learning rate, demonstrating that standard reinforcement learning produces dense gradients unsuitable for efficient communication in distributed training.}
    \label{fig:gradient_sparsity_over_time}
\end{figure}

Gradients are nearly fully dense (${\sim}$99\% non-zero), yet parameter updates are highly sparse after BF16 casting (${\sim}$97\% unchanged). \Cref{fig:gradient_sparsity_over_time} visualizes this behavior across a wide range of training configurations, confirming that dense gradients are a universal property of standard GRPO. The BF16 absorption mechanism (\Cref{subsec:precision}) explains the transformation from dense gradients to sparse weight updates: dense gradients produce updates that fall below the representable threshold for most parameters.

This has practical implications for system design. Gradient compression techniques~\citep{lin2018deep,alistarh2017qsgd} would achieve far lower compression ratios than parameter change compression, since gradients remain dense throughout training.

\subsection{Training Curves Across Model Scales}
\label{app:training_curves}

To validate that our 400-step training duration captures the meaningful learning dynamics, we present pass@1 accuracy curves across all model families and sizes used in our sparsity analysis.

\begin{figure}[t]
    \centering
    \includegraphics[width=\linewidth]{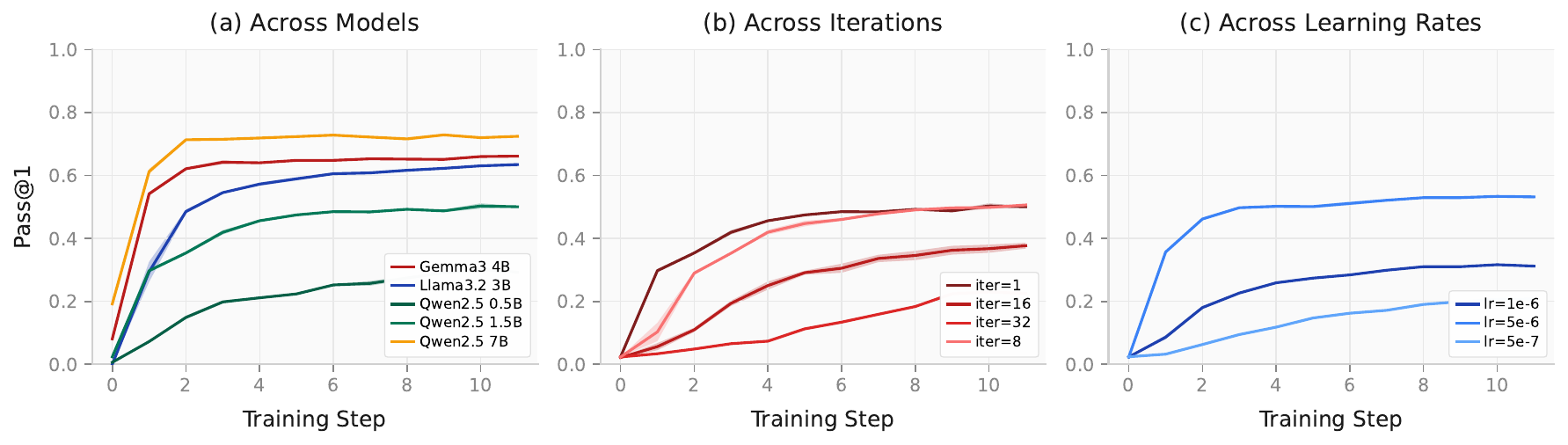}
    \caption{\textbf{Training curves across model scales.} Pass@1 validation accuracy throughout training for all models used in our sparsity analysis. All models show rapid initial improvement followed by convergence within 400 steps, validating our choice of training duration. Shaded regions indicate $\pm 1$ standard error across 4 seeds.}
    \label{fig:training_curves_all}
\end{figure}

\Cref{fig:training_curves_all} shows that all models exhibit similar learning dynamics: rapid initial improvement in the first 100--200 steps, followed by gradual convergence. By step 400, performance has largely plateaued across all model scales and families, confirming that our experimental duration is sufficient to capture stable-phase sparsity behavior. This consistency across architectures (Qwen, Llama, Gemma) and scales (0.5B--7B) provides confidence that our sparsity observations reflect the converged training regime rather than transient early-training artifacts.

\subsection{Factors Affecting Weight Update Sparsity}
\label{app:sparsity_factors}

\Cref{fig:lr_sparsity} reports the learning-rate sweep referenced in \Cref{subsec:precision}. Raising the learning rate shifts Adam updates upward relative to the BF16 absorption threshold, so more weights change. In practice, learning rates above ${\sim}5 \times 10^{-6}$ destabilize RL training, and the stable range remains in the high-sparsity regime.

The rollout-staleness sweep in \Cref{fig:staleness_sparsity_main} varies the rollout synchronization interval $S$, the number of optimizer steps between rollout regenerations. A cycle of length $S$ induces off-policy delays $\tau \in \{0, \ldots, S-1\}$, with $S=1$ corresponding to fully on-policy training. For per-step updates ($k=1$), sparsity remains above 98.5\% even at $S=32$. For larger $k$, sparsity decreases as more parameters accumulate changes that survive the BF16 cast, but remains above 97.5\% across all conditions tested.

\begin{figure}[t]
    \centering
    \includegraphics[width=0.78\linewidth]{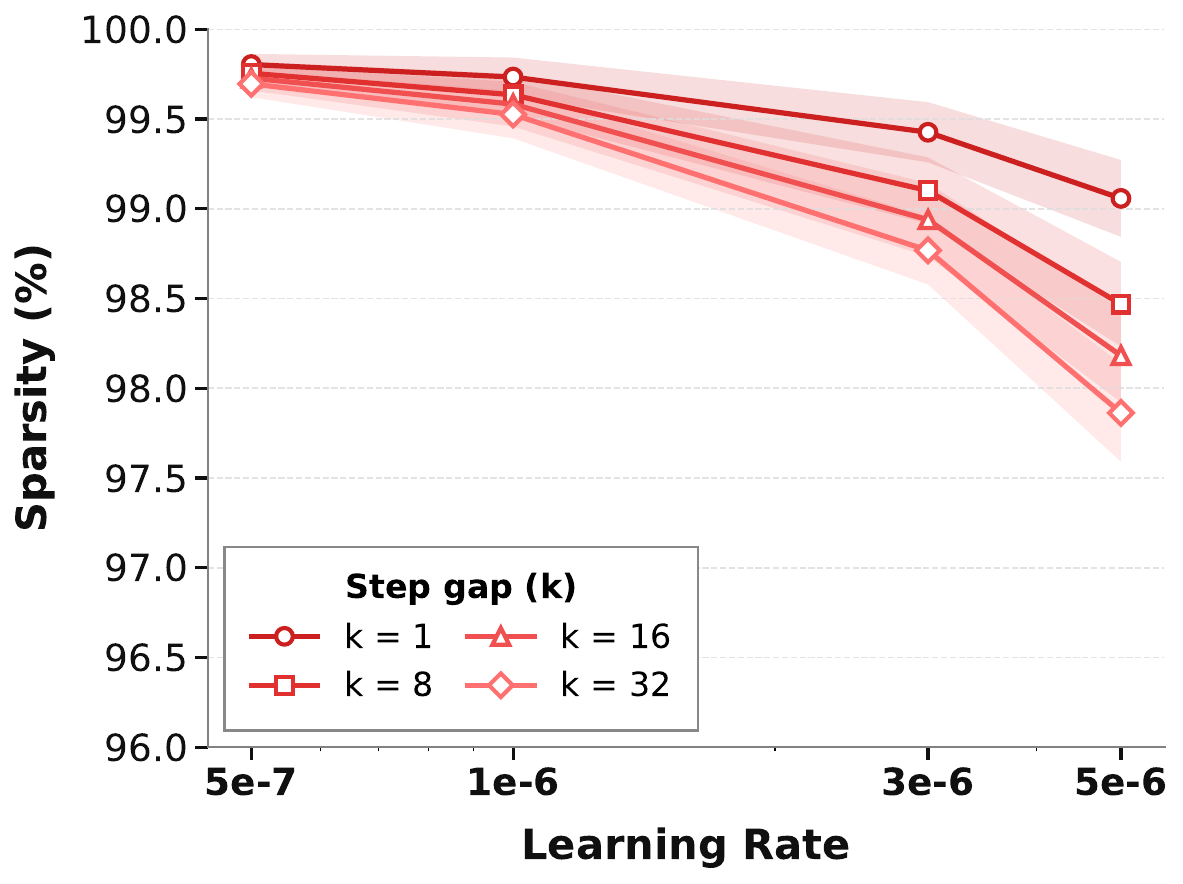}
    \caption{\textbf{Learning-rate effect on weight update sparsity.} Sparsity is measured after BF16 casting. Each line shows $k$-step sparsity as a function of learning rate. Higher learning rates reduce sparsity by increasing update magnitudes above the BF16 absorption threshold. Shaded regions indicate $\pm 1$ standard deviation across training steps.}
    \label{fig:lr_sparsity}
\end{figure}

\subsection{Sparsity Dynamics Throughout Training}
\label{app:sparsity_over_time}

\Cref{sec:sparsity_analysis} reports time-averaged sparsity statistics. Here we examine how sparsity evolves over individual training steps, revealing a characteristic transient that directly confirms the learning rate mechanism established in \Cref{subsec:precision}.

\begin{figure}[t]
    \centering
    \includegraphics[width=\linewidth]{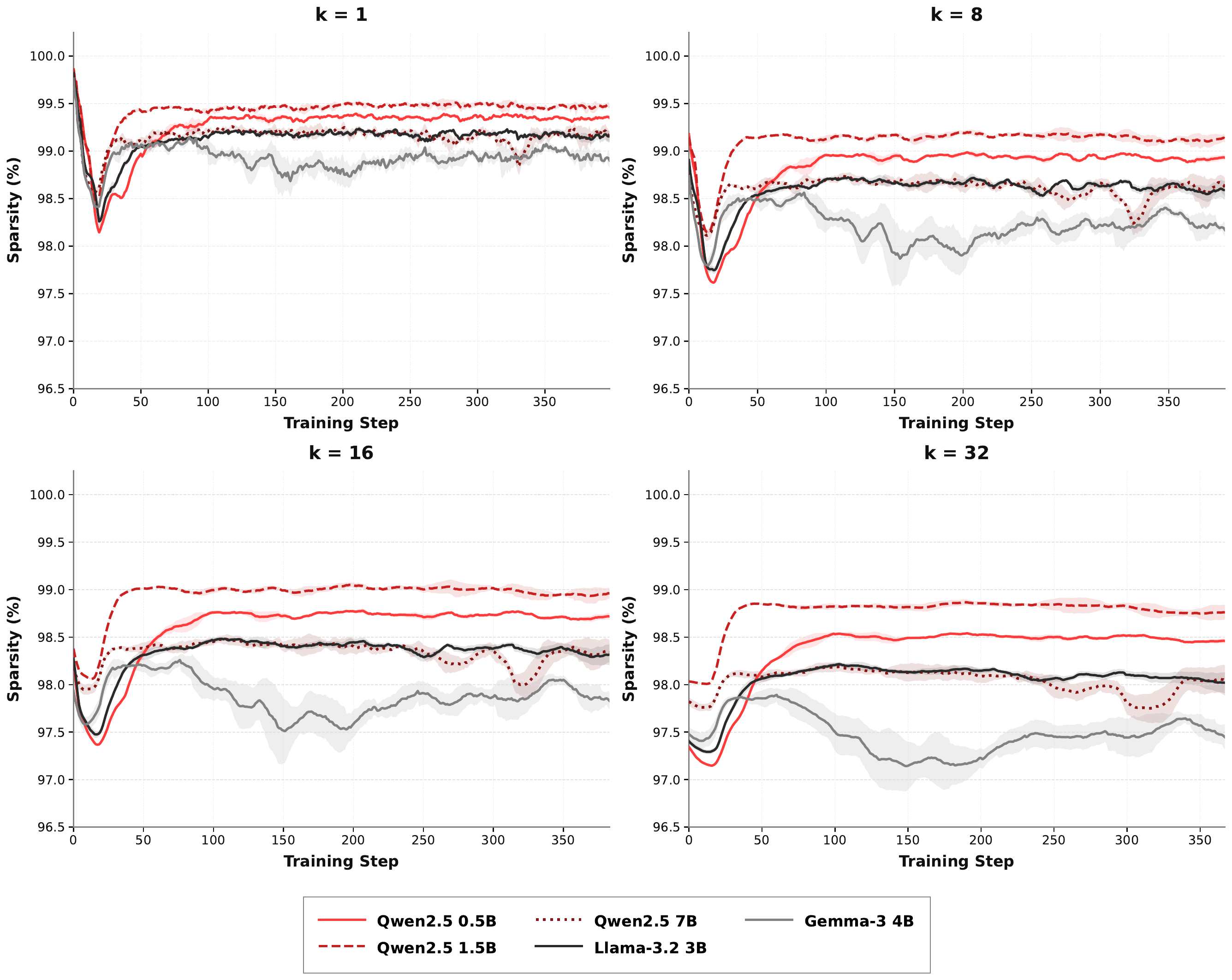}
    \caption{\textbf{Sparsity dynamics throughout training for $k$-step comparisons.} Each panel shows sparsity as a function of training step for a different comparison interval $k \in \{1, 8, 16, 32\}$. All models exhibit a characteristic dip during the learning rate warmup period (steps 0--20), followed by rapid recovery and stable sparsity for the remainder of training. Shaded regions indicate $\pm 1$ standard deviation across 4 seeds.}
    \label{fig:sparsity_over_time}
\end{figure}

\paragraph{Learning rate warmup produces a predictable transient.}
Our training configuration uses a linear warmup that ramps $\eta$ from 0 to $3 \times 10^{-6}$ over the first 20 steps (\Cref{app:hyperparameters}). \Cref{fig:sparsity_over_time} shows a pronounced sparsity dip precisely during this window. The mechanism follows directly from the BF16 absorption analysis (\Cref{app:bf16_details}): at step 0, $\eta \approx 0$, so all updates are absorbed and sparsity is near 100\%. As $\eta$ increases, update magnitudes grow proportionally ($|\Delta w| \approx \eta \cdot |\rho_t|$), pushing more parameters above the absorption threshold $|w|/256$. Sparsity reaches its minimum around step 20, precisely when $\eta$ attains its full value. This correspondence provides direct empirical confirmation that learning rate is the primary control variable for sparsity, as predicted in \Cref{subsec:precision}.

\paragraph{Post-warmup recovery and stabilization.}
After the warmup completes, sparsity recovers within approximately 20--30 steps and remains stable for the remainder of training. For $k=1$, steady-state sparsity settles at ${\sim}$99--99.5\% across all models. This recovery reflects Adam's moment estimates reaching equilibrium: during the warmup transient, bias-corrected moments overshoot because $m_t$ responds to rising gradients faster than $v_t$ (since $\beta_1 < \beta_2$), temporarily elevating the ratio $|\hat{m}_t|/\sqrt{\hat{v}_t}$ above its steady-state value of ${\sim}$1. Once $\eta$ stabilizes and the moments equilibrate, the ratio settles and sparsity locks into its characteristic level.

\paragraph{Multi-step comparisons amplify the transient.}
The warmup-induced dip is deeper and wider for larger $k$. At $k=1$, the minimum sparsity is ${\sim}$98\%; at $k=32$, it drops to ${\sim}$97\%. This is expected: a $k$-step comparison at step $t$ measures cumulative changes that survive the BF16 cast over the window $[\bar{\theta}^{\mathrm{BF16}}_t, \bar{\theta}^{\mathrm{BF16}}_{t+k}]$. When this window overlaps with the warmup period, it aggregates updates from steps with varying (and transiently elevated) learning rates, accumulating more changed parameters. After the warmup window clears (roughly by step $20 + k$), multi-step sparsity also stabilizes, remaining above 98\% for $k \leq 8$ and above 97\% for $k = 32$, consistent with the time-averaged results in \Cref{fig:sparsity_main}.

\paragraph{Implications.}
Even during the warmup transient, sparsity never drops below ${\sim}$97\% for any model or $k$ value, meaning sparse synchronization is viable from the very first training step. The warmup dip reduces compression ratios only marginally (from ${\sim}$99\% to ${\sim}$97--98\%), still yielding substantial bandwidth savings over full checkpoint transfer. Practitioners need not treat the warmup period as a special case; \protoname's compression benefits apply throughout the entire training run.

\subsection{PULSELoCo Sparse-Payload Sensitivity to Local Steps}
\label{app:pulseloco_h_sensitivity}

The main PULSELoCo comparison uses the largest stable local-update window per model: $H = 8$ local Adam steps for the Qwen models and $H = 4$ for Llama-3.2-3B-Instruct. Here we sweep $H \in \{4, 8, 16\}$ on Qwen2.5-3B at fixed $R = 4$ to show how the sparse payload changes along the local-step axis.

\begin{figure}[t]
\centering
\includegraphics[width=0.82\linewidth]{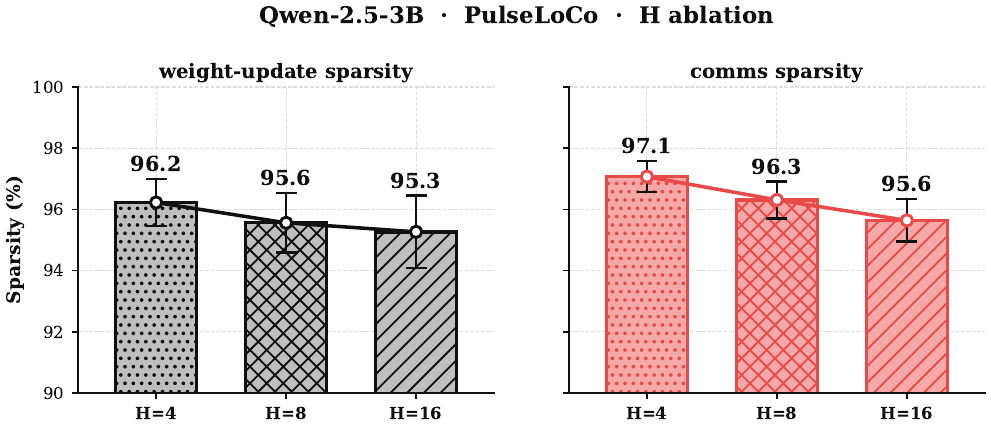}
\caption{\textbf{PULSELoCo sparsity sensitivity to local-step count $H$.} We sweep $H \in \{4,8,16\}$ on Qwen2.5-3B at $R = 4$. Left: BF16 weight-update sparsity between global checkpoints. Right: pseudo-gradient communication sparsity after error feedback. Error bars indicate $\pm 1$ standard deviation across logged outer rounds and seeds.}
\label{fig:h_sensitivity}
\end{figure}

\Cref{fig:h_sensitivity} isolates the sparsity side of the local-step tradeoff. Larger $H$ accumulates more local change before synchronization, reducing BF16 weight-update sparsity from $96.2\%$ at $H=4$ to $95.3\%$ at $H=16$ and pseudo-gradient communication sparsity from $97.1\%$ to $95.6\%$. The sparse payload remains far below dense synchronization throughout this range, while the main experiments keep $H$ at the largest stable setting for each model to avoid the rollout-staleness instability discussed in \Cref{app:hyperparameters}.

\section{Additional Method Details}
\label{app:method_details}

\subsection{GRPO Formulation}
\label{app:grpo_details}

GRPO eliminates the need for a separate value network by estimating advantages from group-relative rewards. Following DAPO~\citep{yu2025dapo}, we use asymmetric clipping bounds $\epsilon_{\text{low}}$ and $\epsilon_{\text{high}}$ to encourage exploration. The objective function is:
\begin{equation}
\label{eq:grpo_app}
\begin{split}
    \cJ_{\text{GRPO}}(\theta) &= \E_{x \sim \cD,\, \{y_i\}_{i=1}^G \sim \pi_{\theta_{\text{old}}}(\cdot|x)} \Bigg[ \frac{1}{G}\sum_{i=1}^G \frac{1}{|y_i|}\sum_{t=1}^{|y_i|} \Big\{ \min\Big[ r_{i,t}(\theta)\, \hat{A}_i, \\
    &\quad \text{clip}(r_{i,t}(\theta), 1{-}\epsilon_{\text{low}}, 1{+}\epsilon_{\text{high}})\, \hat{A}_i \Big] - \beta \, D_{\text{KL}}\big[\pi_\theta \| \pi_{\text{ref}}\big] \Big\} \Bigg]
\end{split}
\end{equation}
where $\{y_i\}_{i=1}^G$ are $G$ sampled responses for a given prompt $x$, and the importance weight ratio is:
\begin{equation}
    r_{i,t}(\theta) = \frac{\pi_\theta(y_{i,t} | x, y_{i,<t})}{\pi_{\theta_{\text{old}}}(y_{i,t} | x, y_{i,<t})}
\end{equation}
The advantage $\hat{A}_i$ for the $i$-th response is computed relative to the group statistics:
\begin{equation}
    \hat{A}_i = \frac{r(x, y_i) - \mu_G}{\sigma_G}, \quad \text{where} \quad \mu_G = \frac{1}{G}\sum_{j=1}^G r(x, y_j), \quad \sigma_G = \sqrt{\frac{1}{G}\sum_{j=1}^G (r(x, y_j) - \mu_G)^2}
\end{equation}
The clipping mechanism prevents overly large policy updates, with the asymmetric bounds ($\epsilon_{\text{high}} > \epsilon_{\text{low}}$) relaxing the upper limit to mitigate entropy collapse. The KL penalty (controlled by $\beta$) regularizes deviations from the reference policy $\pi_{\text{ref}}$.

\subsection{Index Encoding}
\label{app:indices}

We use delta encoding for indices to improve compression:
\begin{enumerate}[leftmargin=*,itemsep=2pt]
    \item Sort indices in ascending order
    \item Store first index as-is (4 bytes)
    \item Store subsequent indices as differences from previous
    \item Downcast index types (e.g., uint8 for row deltas, uint16 for column deltas; see \Cref{app:component_contribution})
\end{enumerate}
This typically reduces index storage by 40--60\% before zstd compression.

\subsection{Memory Management}
\label{app:memory}

PULSESync requires maintaining the previous checkpoint to compute the sparse delta. The memory overhead is minimal:
\begin{itemize}[leftmargin=*,itemsep=2pt]
    \item \textbf{Training node}: Maintains the current weights on the GPU and the previous weights in pinned CPU memory. This results in a total memory overhead of approximately $1.1\times$ the model size compared to standard training.
    \item \textbf{Inference node}: Loads the base weights once and applies incoming deltas in-place. No additional weight copies are required after the initial load.
\end{itemize}

\subsection{Compression Ablation Study}
\label{app:compression_ablation}

This section provides comprehensive ablation studies supporting the codec selection summarized in \Cref{app:compression}. We analyze: (1) component contributions to compression ratio, (2) sparse representation format choices, (3) full algorithm comparison with Pareto analysis, and (4) per-model variations.

\paragraph{Methodology.} We measured compression performance on sparse delta checkpoints from 14 experiments across 3 model families (Qwen2.5, Gemma3, LLaMA3.2), with 20 checkpoint files per experiment (5 evenly-spaced steps $\times$ 4 seeds). We tested 6 sparse representations $\times$ 5 compression algorithms = 30 combinations, yielding 8,100 total measurements. Timing measurements used 1 warmup + 3 measurement iterations on an AMD EPYC 7763 processor, with verification that decompress(compress(x)) == x for all.

\paragraph{Compression ratio definition.} Throughout this section, \emph{sparse ratio} refers to compression of the sparse representation itself (compressed bytes / COO baseline bytes), while \emph{full ratio} refers to compression versus the dense BF16 model (dense bytes / compressed bytes).

\subsubsection{Component Contribution Analysis}
\label{app:component_contribution}

The compression pipeline applies several transformations before entropy coding. \Cref{tab:component_contribution} shows the incremental contribution of each component using zstd-1 across all models ($n=270$).

\begin{table}[t]
    \centering
    \caption{\textbf{Component contribution to compression ratio.} Each row adds one transformation. Sparse ratio is relative to COO baseline; $\Delta$ shows incremental improvement. All measurements use zstd-1 ($n=270$).}
    \label{tab:component_contribution}
    \small
    \begin{tabular}{@{}lcccc@{}}
        \toprule
        \textbf{Configuration} & \textbf{Sparse Ratio} & \textbf{$\Delta$ Ratio} & \textbf{Encode (MB/s)} \\
        \midrule
        Raw COO (baseline) & $2.71\times \pm 0.25$ & -- & 397 \\
        + Index sorting & $2.71\times \pm 0.25$ & $+0.0\%$ & 397 \\
        + Delta encoding & $3.07\times \pm 0.34$ & $+13.3\%$ & 398 \\
        + Type downscaling & $3.33\times \pm 0.29$ & $+8.5\%$ & 534 \\
        \bottomrule
    \end{tabular}
\end{table}

\paragraph{Index sorting.} Sorting indices in ascending order has no direct size impact but enables delta encoding.

\paragraph{Delta encoding.} Instead of storing absolute indices, we store the first index and subsequent differences. Since changed parameters tend to cluster, differences are small and compress well. This contributes $+13.3\%$ improvement.

\paragraph{Type downscaling.} For COO format, we store row deltas as uint8 and column deltas as uint16, exploiting the fact that consecutive changes rarely span more than 255 rows or 65,536 columns. This contributes an additional $+8.5\%$ improvement and also increases encode throughput (smaller data to process).

\paragraph{Total improvement.} The full pipeline (delta encoding + downscaling) improves sparse compression ratio by $+22.9\%$ over the raw baseline ($2.71\times \to 3.33\times$).

\subsubsection{Sparse Representation Format Comparison}
\label{app:sparse_format}

We compared two sparse representation strategies: (1) \textbf{2D COO}: store per-tensor (row, column) indices; (2) \textbf{1D Flat}: flatten all tensors, store global indices. \Cref{tab:format_comparison} shows results with fair comparison (both using int32 indices).

\begin{table}[t]
    \centering
    \caption{\textbf{Sparse representation format comparison} using int32 indices and zstd-1 ($n=270$).}
    \label{tab:format_comparison}
    \small
    \begin{tabular}{@{}lccc@{}}
        \toprule
        \textbf{Format} & \textbf{Sparse Ratio} & \textbf{Encode (MB/s)} \\
        \midrule
        2D COO (delta\_coo\_int32) & $3.07\times \pm 0.34$ & 398 \\
        1D Flat (delta\_flat\_int32) & $3.19\times \pm 0.29$ & 312 \\
        \bottomrule
    \end{tabular}
\end{table}

\paragraph{Finding.} 1D Flat achieves $+3.9\%$ better compression than 2D COO when using identical index types, because global indices enable better delta encoding across tensor boundaries. However, 2D COO enables type downscaling (uint8 row deltas, uint16 column deltas) which is harder for flat indices. Our default configuration uses COO with downscaling ($3.33\times$), which outperforms flat with int32 ($3.19\times$).

\subsubsection{Full Algorithm Comparison}
\label{app:algorithm_comparison}

\Cref{tab:full_algorithm_comparison} compares compression algorithms using our default representation (delta\_coo\_downscaled). We mark Pareto-optimal configurations.

\begin{table}[t]
    \centering
    \caption{\textbf{Compression algorithm comparison} using our default representation ($n=270$). Sparse ratio is vs COO baseline; full ratio is vs dense BF16 model.}
    \label{tab:full_algorithm_comparison}
    \small
    \begin{tabular}{@{}lccccc@{}}
        \toprule
        \textbf{Algorithm} & \textbf{Sparse Ratio} & \textbf{Full Ratio} & \textbf{Encode (MB/s)} & \textbf{Decode (MB/s)} & \textbf{Pareto} \\
        \midrule
        snappy & $2.41\times \pm 0.15$ & $56\times$ & $1041 \pm 357$ & $1289 \pm 485$ & $\star$ \\
        lz4 & $2.40\times \pm 0.13$ & $56\times$ & $830 \pm 236$ & $1484 \pm 524$ & $\star$ \\
        \textbf{zstd-1} & $\mathbf{3.33\times \pm 0.29}$ & $\mathbf{79\times}$ & $\mathbf{534 \pm 56}$ & $\mathbf{851 \pm 108}$ & $\star$ \\
        zstd-3 & $3.40\times \pm 0.27$ & $80\times$ & $197 \pm 21$ & $670 \pm 69$ & $\star$ \\
        gzip-6 & $3.32\times \pm 0.26$ & $78\times$ & $14 \pm 2$ & $192 \pm 11$ & \\
        \bottomrule
    \end{tabular}
\end{table}

\paragraph{Key observations.}
\begin{itemize}[leftmargin=*, itemsep=2pt]
    \item \textbf{gzip-6 is never Pareto-optimal}: zstd-1 achieves the same ratio ($3.33\times$ vs $3.32\times$) but encodes $38\times$ faster (534 vs 14 MB/s).
    \item \textbf{snappy/lz4 for speed}: At high bandwidth, snappy (1041 MB/s) or lz4 (830 MB/s) minimize total transfer time despite lower ratios.
    \item \textbf{zstd dominates mid-range}: zstd-1 provides the best tradeoff for typical cloud bandwidths (15 Mbit/s--1 Gbit/s).
\end{itemize}

\subsubsection{Per-Model Breakdown}
\label{app:per_model_compression}

\Cref{tab:per_model_compression} shows compression varies across model families.

\begin{table}[t]
    \centering
    \caption{\textbf{Per-model compression} with zstd-1 default configuration.}
    \label{tab:per_model_compression}
    \small
    \begin{tabular}{@{}lcccc@{}}
        \toprule
        \textbf{Model Family} & \textbf{Sparsity} & \textbf{Sparse Ratio} & \textbf{Full Ratio} & \textbf{$n$} \\
        \midrule
        Qwen2.5 (0.5B--7B) & $99.0\% \pm 0.7\%$ & $3.31\times \pm 0.31$ & $76\times$ & 210 \\
        LLaMA3.2 (3B) & $99.3\% \pm 0.1\%$ & $3.36\times \pm 0.18$ & $100\times$ & 20 \\
        Gemma-3-4B & $99.2\% \pm 0.2\%$ & $3.42\times \pm 0.21$ & $80\times$ & 40 \\
        \bottomrule
    \end{tabular}
\end{table}

\paragraph{Observations.} LLaMA3.2 achieves the highest full compression ratio ($100\times$) due to its high sparsity ($99.3\%$) and favorable weight distribution. The measured range of $76$--$100\times$ across model families is consistent with the theoretical expectation: higher sparsity yields higher compression. The variation reflects differences in both sparsity levels and weight value distributions across architectures.

\subsubsection{Bandwidth-Dependent Algorithm Selection}
\label{app:bandwidth_selection}

The optimal algorithm depends on bandwidth. Total transfer time is:
\begin{equation}
    T_{\text{total}} = T_{\text{encode}} + \frac{S_{\text{payload}}}{R \cdot B} + T_{\text{decode}}
\end{equation}
where $S_{\text{payload}}$ is uncompressed sparse payload size, $R$ is compression ratio, and $B$ is bandwidth. At high $B$, encoding time dominates; at low $B$, transfer time dominates. \Cref{fig:compression_selection} visualizes this selection process across different bandwidth tiers.

\begin{figure}[htbp]
    \centering
    \includegraphics[width=0.85\linewidth]{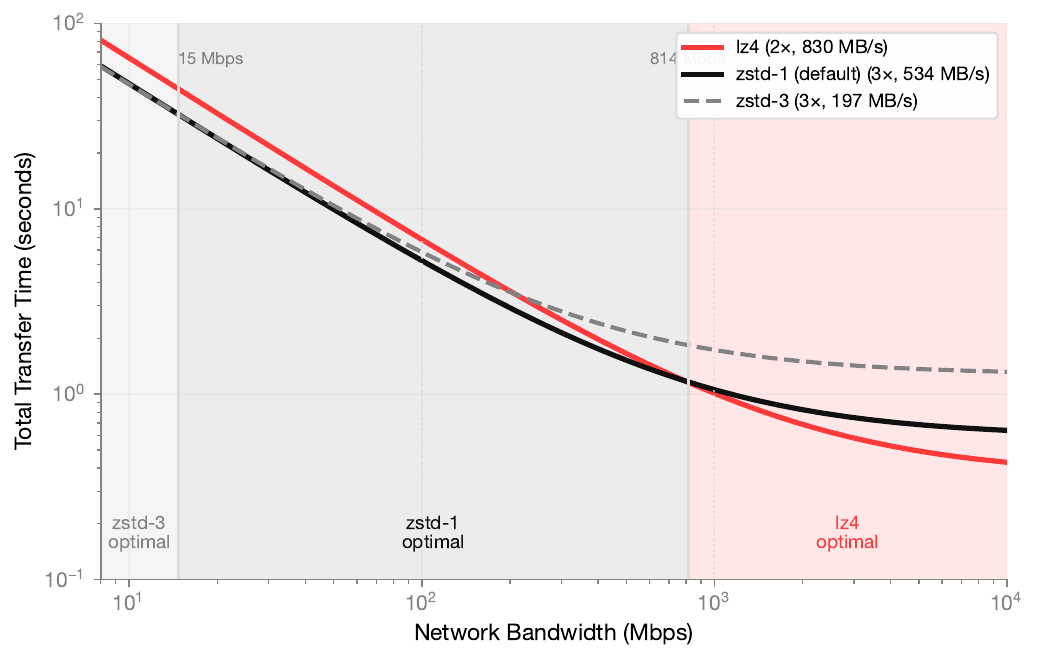}
    \caption{\textbf{Bandwidth-aware algorithm selection.} Total transfer time (encode + network + decode) for a 7B model. Shaded regions indicate the optimal algorithm per bandwidth tier. Fast algorithms like lz4 are preferred at high bandwidth, while high-ratio algorithms like zstd-3 are better for constrained links.}
    \label{fig:compression_selection}
\end{figure}

\paragraph{Crossover formula.} The crossover bandwidth where two algorithms $A$ and $B$ have equal total transfer time can be derived analytically. Setting $T_A = T_B$ and solving for bandwidth:
\begin{equation}
    B_{\text{crossover}} = \frac{S_{\text{payload}} \cdot (R_B^{-1} - R_A^{-1})}{(T_{\text{enc},A} + T_{\text{dec},A}) - (T_{\text{enc},B} + T_{\text{dec},B})}
\end{equation}

\paragraph{Crossover points.} From our empirical benchmarks (194\,MB payload, default configuration):
\begin{itemize}[leftmargin=*, itemsep=2pt]
    \item \textbf{zstd-3 $\to$ zstd-1}: ${\sim}$15\,Mb/s (below this, zstd-3's marginally better ratio wins)
    \item \textbf{zstd-1 $\to$ lz4}: ${\sim}$800\,Mb/s (above this, lz4's faster encode wins)
\end{itemize}
These crossovers scale with payload size; larger payloads shift crossovers to higher bandwidths.

\paragraph{Why zstd-1 is the default.} Most deployments operate in the 15--800\,Mb/s range (commodity internet, cross-datacenter links). In this regime, zstd-1 minimizes end-to-end latency while achieving $79\times$ full compression. Users with different bandwidth profiles can override via configuration.

\subsection{Algorithms}
\label{app:algorithms}

\begin{algorithm}[t]
\caption{Sparse Delta Encoding (detailed)}
\label{alg:encode}
\begin{algorithmic}[1]
\Require Current weights $W_t$, previous weights $W_{t-1}$
\Ensure Compressed patch $P$, hash $h$
\State $\cI \gets \emptyset$; $\cV \gets \emptyset$; $\cS \gets \emptyset$
\For{each parameter tensor $p \in \text{params}(W_t)$}
    \State $\cM \gets \{i : W_t[p]_i \neq W_{t-1}[p]_i\}$ \Comment{Find changed positions (bitwise)}
    \If{$|\cM| > 0$}
        \State $\cI[p] \gets \cM$; $\cV[p] \gets W_t[p][\cM]$; $\cS[p] \gets \text{shape}(W_t[p])$
    \EndIf
\EndFor
\State $(\cI, \cV) \gets \textsc{DeltaEncode}(\cI, \cV)$ \Comment{Sort indices, store differences (\Cref{app:indices})}
\State $\cI \gets \textsc{Downcast}(\cI)$ \Comment{Narrow index types (\Cref{app:component_contribution})}
\State $P \gets \textsc{Compress}((\cI, \cV, \cS))$; $h \gets \textsc{SHA256}(W_t)$
\State \Return $P$, $h$
\end{algorithmic}
\end{algorithm}

\begin{algorithm}[t]
\caption{Sparse Delta Application (detailed)}
\label{alg:apply}
\begin{algorithmic}[1]
\Require Base weights $W_{\text{base}}$, compressed patch $P$
\Ensure Reconstructed weights $W_{\text{recon}}$
\State $(\cI, \cV, \cS) \gets \textsc{Decompress}(P)$
\State $\cI \gets \textsc{Upcast}(\cI)$ \Comment{Restore original index types}
\State $(\cI, \cV) \gets \textsc{DeltaDecode}(\cI, \cV)$ \Comment{Recover absolute indices}
\State $W_{\text{recon}} \gets \textsc{Copy}(W_{\text{base}})$
\For{each parameter $p$ with changes}
    \State $W_{\text{recon}}[p][\cI[p]] \gets \cV[p]$ \Comment{Direct value assignment (no FP arithmetic)}
\EndFor
\State \Return $W_{\text{recon}}$
\end{algorithmic}
\end{algorithm}

\subsection{Lossless Reconstruction Guarantee}
\label{app:lossless_proof}

PULSESync guarantees bit-exact reconstruction because it stores \emph{actual weight values} rather than arithmetic differences.

\begin{proposition}[Lossless Reconstruction]
\label{thm:lossless}
For any patch $P = (\cI, \cV)$ derived from consecutive checkpoints $W_{t-1}$ and $W_t$, applying $P$ to $W_{t-1}$ reconstructs $W_t$ exactly:
\begin{equation}
    \mathrm{Decode}(W_{t-1}, P) \equiv W_t \quad \text{(bitwise)}
\end{equation}
This property extends to chains of patches: applying $P_1, P_2, \ldots, P_n$ sequentially to anchor $W_0$ reconstructs $W_n$ exactly.
\end{proposition}

The proof is immediate from the algorithm construction (\Cref{alg:sparse_patch}). Reconstruction performs direct memory assignment $W[\cI] \gets \cV$ with no floating-point arithmetic, as illustrated in \Cref{fig:patching}. For indices $i \in \cI$, we copy the exact bit pattern from $\cV$; for indices $i \notin \cI$, the value is unchanged and already correct. No rounding, truncation, or approximation occurs at any step.

\paragraph{Contrast with additive delta schemes.}
Traditional delta compression stores $\delta_t = W_t - W_{t-1}$ and reconstructs via $W_t = W_{t-1} + \delta_t$. This addition is a floating-point operation subject to rounding. Over long chains, small errors accumulate:
\begin{equation}
    W_{\text{recon}} = W_0 + \sum_{i=1}^n \delta_i \neq W_n \quad \text{(in general)}
\end{equation}
PULSESync avoids this entirely by storing values, not differences. Each patch application overwrites positions with their correct final values, independent of chain length.

\paragraph{Practical verification.}
We verify losslessness empirically via a collision-resistant checksum. Each patch includes a checksum of the expected reconstructed weights; inference nodes recompute and compare it after applying each patch. In all experiments, 100\% of reconstructions passed verification, confirming bit-identical weights across the network.

\begin{figure}[t]
    \centering
    \includegraphics[width=0.85\linewidth]{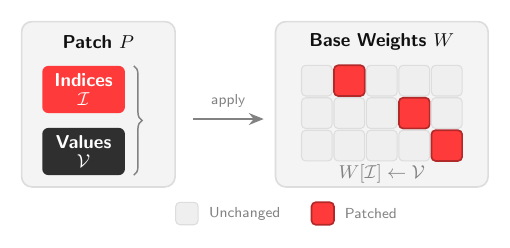}
    \caption{\textbf{Sparse value patching.} A patch $P = (\cI, \cV)$ consists of changed indices $\cI$ and their new values $\cV$. To reconstruct $W_t$ from $W_{t-1}$, we overwrite: $W_t[\cI] \gets \cV$. This direct assignment requires no floating-point arithmetic, guaranteeing bit-exact reconstruction.}
    \label{fig:patching}
\end{figure}

\section{Comparison with Related Methods}
\label{app:comparison}

PULSESync differs from standard gradient-compression methods in three ways. First, it operates on weight snapshots viewed at BF16 precision, not on gradients before trainer-side synchronization. Second, it is lossless for BF16 inference workers: the receiver reconstructs the exact BF16 tensor that the trainer would use for the next forward pass. Third, it stores new values rather than arithmetic differences, avoiding drift from repeated floating-point additions. PULSELoCo, by contrast, targets DiLoCo pseudo-gradient synchronization and uses explicit FP32 error feedback; its quality comparison to DiLoCo is reported in \Cref{sec:experiments}.

\section{Synchronization Protocol Details}
\label{app:distributed_sync}

This section provides implementation details for the synchronization protocols described in \Cref{sec:method}. We cover: (1) the PULSESync publication and recovery protocol, (2) how to choose the anchor interval, (3) integrity verification mechanisms, (4) failure recovery strategies, (5) end-to-end latency analysis, (6) storage format specification, (7) retention policies, and (8) PULSELoCo round atomicity.

\subsection{Distributed Synchronization Protocol}
\label{app:sync_protocol}

The PULSESync protocol operates asynchronously between training and inference nodes. Training nodes \emph{publish} checkpoints to shared storage, while inference nodes independently \emph{pull} updates. This decoupled design allows training and inference to scale independently.

\Cref{alg:sync_protocol} formalizes the protocol. The key parameters are: $W_t$ (weights at step $t$), $k$ (anchor interval), and $h_t$ (integrity hash). The protocol distinguishes between a \emph{fast path} (single delta application) and a \emph{slow path} (anchor download plus delta chain). Delta checkpoints and full anchors have separate ready markers: a delta-ready marker advances the steady-state stream, while an anchor-ready marker advertises a full checkpoint for slow-path recovery.

\begin{algorithm}[t]
\caption{Distributed Synchronization Protocol}
\label{alg:sync_protocol}
\begin{algorithmic}[1]
\Statex \textbf{Training Node (Publisher):}
\Procedure{PublishCheckpoint}{$W_t, W_{t-1}, t, k$}
    \State $h_t \gets \textsc{SHA256}(W_t)$ \Comment{Compute integrity hash}
    \If{$t \mod k = 0$} \Comment{Anchor window}
        \State \textsc{StartUploadFull}($W_t$, $t$, $h_t$) \Comment{Background full checkpoint}
    \EndIf
    \State $P \gets \textsc{Encode}(W_t, W_{t-1})$ \Comment{Sparse patch}
    \State \textsc{UploadDelta}($P$, $t$, $t-1$, $h_t$)
    \State \textsc{SetDeltaReadyMarker}($t$) \Comment{Signal fast-path availability}
    \If{$t \mod k = 0$}
        \State \textsc{SetAnchorReadyWhenComplete}($t$) \Comment{Signal slow-path anchor}
    \EndIf
\EndProcedure
\Statex
\Statex \textbf{Inference Node (Consumer):}
\Procedure{Synchronize}{$W_{\text{local}}, t_{\text{local}}$}
    \State $t_{\text{latest}} \gets \textsc{GetLatestDeltaReady}()$
    \If{$t_{\text{latest}} = t_{\text{local}}$}
        \State \Return $W_{\text{local}}$ \Comment{Already synchronized}
    \EndIf
    \If{$t_{\text{latest}} = t_{\text{local}} + 1$} \Comment{Fast path}
        \State $P, h \gets \textsc{DownloadDelta}(t_{\text{latest}})$
        \State $W_{\text{new}} \gets \textsc{Decode}(W_{\text{local}}, P)$
        \State \textbf{assert} $\textsc{SHA256}(W_{\text{new}}) = h$ \Comment{Verify integrity}
    \Else \Comment{Slow path: cold start or missed steps}
        \State $t_{\text{anchor}} \gets \textsc{GetLatestAnchorReady}(t_{\text{latest}})$
        \State $W_{\text{new}} \gets \textsc{DownloadFull}(t_{\text{anchor}})$
        \For{$t' \gets t_{\text{anchor}} + 1$ \textbf{to} $t_{\text{latest}}$}
            \State $P, h \gets \textsc{DownloadDelta}(t')$
            \State $W_{\text{new}} \gets \textsc{Decode}(W_{\text{new}}, P)$
            \State \textbf{assert} $\textsc{SHA256}(W_{\text{new}}) = h$ \Comment{Verify integrity}
        \EndFor
    \EndIf
    \State \Return $W_{\text{new}}$
\EndProcedure
\end{algorithmic}
\end{algorithm}

\begin{figure}[t]
    \centering
    \includegraphics[width=0.75\linewidth]{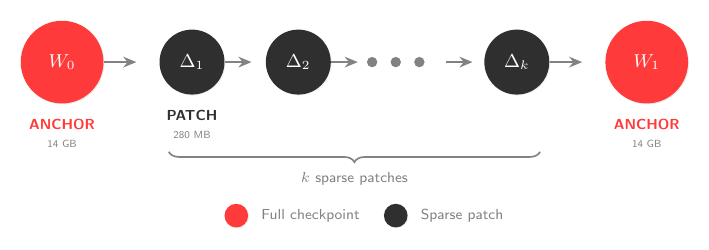}
    \caption{\textbf{Checkpoint chain structure.} Full checkpoints (anchors) are published every $k$ steps; between anchors, only sparse patches are transmitted. This structure enables the fast path (single patch application) for steady-state nodes while providing recovery points for late joiners via the slow path (anchor download plus patch chain). See \Cref{alg:sync_protocol} for the formal protocol.}
    \label{fig:checkpoint_chain}
\end{figure}

\paragraph{Ready markers.} The protocol uses explicit ready markers to ensure atomicity. A delta checkpoint is available only after its sparse patch and manifest have been uploaded. A full anchor is available only after its full checkpoint and manifest have been uploaded. This prevents inference nodes from reading partially uploaded objects.

\paragraph{Concurrent uploads.} At anchor windows, both FULL and DELTA objects are produced. The DELTA upload stays on the steady-state critical path, while the FULL upload runs asynchronously in the background and receives an anchor-ready marker only after completion. Until that marker appears, slow-path receivers use the previous ready anchor; steady-state receivers continue along the delta stream.

\subsection{PULSELoCo Round Atomicity}
\label{app:pulseloco_round_atomicity}

PULSELoCo does not use PULSESync's anchor-and-replay protocol because it is a trainer-to-trainer collective, not a checkpoint-distribution path. Each outer round is keyed by the shared base checkpoint $\theta_t$: workers send sparse pseudo-gradients for that base, the relay returns one sparse aggregate, and trainers apply the same aggregate before starting the next local-update window. This keeps the outer optimizer state aligned with DiLoCo.

This convention also defines rollout synchronization in the local-update experiments. Since trainers hold different private weights within an outer round, rollout workers serve the last shared global checkpoint and are refreshed only after the next global checkpoint is formed. This ties rollout generation to the same checkpoints across DiLoCo and PULSELoCo; the resulting $H$-dependent staleness tradeoff is described in \Cref{app:hyperparameters}.

\subsection{Anchor Interval Selection}
\label{app:anchor_interval}

The anchor interval $k$ determines how often full checkpoints are published. The choice involves three trade-offs:

\begin{itemize}[leftmargin=*, itemsep=2pt]
    \item \textbf{Cold-start latency}: New nodes must download one anchor plus up to $k-1$ deltas. For a 7B model, this is $14\,\text{GB} + (k-1) \times 108\,\text{MB}$.
    \item \textbf{Storage}: Over $n$ steps, storage is approximately $\lceil n/k \rceil \times 14\,\text{GB} + n \times 108\,\text{MB}$.
    \item \textbf{Trainer upload bandwidth}: Full checkpoints are ${\sim}130\times$ larger than deltas, so lower $k$ places significant upload burden on the trainer.
\end{itemize}

\paragraph{Practical guidance.} In bandwidth-constrained PULSESync deployments, higher $k$ is generally preferable. Steady-state inference nodes use the fast path (single delta) regardless of $k$, so the anchor interval only affects cold starts and trainer uploads. We use $k = 50$ in our experiments, balancing reasonable cold-start times (${\sim}5$ minutes at 400\,Mbit/s) with minimal trainer overhead.

\subsection{Integrity Verification}
\label{app:integrity}

Checkpoints may be corrupted during transmission or by malicious storage providers. PULSESync employs multi-level integrity verification.

\paragraph{File-level integrity.} Each checkpoint includes a signed manifest containing SHA256 hashes for all files. The manifest is signed with the trainer's cryptographic key, preventing tampering by storage providers.

\paragraph{Weight-level integrity.} Each delta includes a SHA256 hash of the \emph{resulting} weights after application:
\begin{equation}
    h_t = \textsc{SHA256}\left(\textsc{Concat}\left(\{W_t[p] : p \in \text{params}\}\right)\right)
\end{equation}
This enables end-to-end verification: after applying a delta chain, the consumer verifies that the reconstructed weights match the expected hash. Hash mismatches trigger automatic fallback to the slow path (re-download from anchor).

\paragraph{Deterministic hashing.} To ensure hash reproducibility across hardware, we use a deterministic serialization order and canonical byte representations. The hash is computed over raw BF16 bit patterns, ensuring bitwise consistency.

\subsection{Failure Recovery}
\label{app:failure_recovery}

\paragraph{Delta upload failure.} If a delta upload fails, the system falls back to uploading a full checkpoint. This ensures the chain remains valid even under network instability.

\paragraph{Hash verification failure.} If an inference node detects a hash mismatch, it discards the corrupted state and re-synchronizes from the nearest anchor. This self-healing behavior ensures eventual consistency.

\paragraph{Network partitions.} Inference nodes operate independently and can tolerate arbitrary network partitions. Upon reconnection, they synchronize to the latest checkpoint using the slow path if necessary.

\subsection{End-to-End Latency Analysis}
\label{app:latency}

We measure end-to-end synchronization latency on commodity hardware with 400\,Mb/s network bandwidth. \Cref{tab:latency_breakdown} breaks down the latency for three scenarios.

\begin{table}[t]
    \centering
    \caption{End-to-end latency breakdown for 7B model synchronization ($400\,Mb/s$ network). The slow path assumes recovery requiring 9 delta applications.}
    \label{tab:latency_breakdown}
    \vspace{0.5em}
    \small
    \begin{tabular}{@{}lrrr@{}}
        \toprule
        \textbf{Operation} & \textbf{Fast Path} & \textbf{Slow Path} & \textbf{Cold Start} \\
        \midrule
        \textit{Download} & & & \\
        \quad Full checkpoint ($14\,GB$) & -- & $280\,s$ & $280\,s$ \\
        \quad Delta(s) (${\sim}108\,MB$ each) & $2.2\,s$ & $19.8\,s$ & -- \\
        \midrule
        \textit{Processing} & & & \\
        \quad Decompression (zstd) & $0.6\,s$ & $5.4\,s$ & -- \\
        \quad Delta application & $0.3\,s$ & $2.7\,s$ & -- \\
        \quad Hash verification & $0.8\,s$ & $7.2\,s$ & $0.8\,s$ \\
        \midrule
        \textbf{Total} & \textbf{3.9\,s} & \textbf{315.1\,s} & \textbf{280.8\,s} \\
        \bottomrule
    \end{tabular}
\end{table}

\paragraph{Fast path dominance.} In steady-state operation, inference nodes use the fast path exclusively, achieving synchronization in ${\sim}4\,s$. This represents over $100\times$ speedup compared to downloading the full $14\,GB$ checkpoint.

\paragraph{Parallelization.} Delta downloads and applications can be pipelined: while applying delta $i$, download delta $i+1$ in parallel. This reduces slow path latency by ${\sim}30\%$ in our implementation.

\subsection{Retention Policy}
\label{app:retention}

Without cleanup, storage grows linearly. PULSESync implements an automatic retention policy.

\paragraph{Delta retention.} Keep the most recent 100 delta checkpoints. Older deltas are deleted, but their anchors are preserved if any retained delta references them.

\paragraph{Anchor retention.} Keep the most recent 10 full checkpoints, plus any anchors referenced by retained deltas.

\paragraph{Storage bounds.} With default settings, maximum storage for a 7B model is:
\begin{equation}
    S_{\max} = 10 \cdot 14\,\text{GB} + 100 \cdot 108\,\text{MB} \approx 151\,\text{GB}
\end{equation}

\end{document}